\documentclass{article}

\PassOptionsToPackage{numbers, compress}{natbib}
\usepackage[final]{neurips_2025}

\usepackage{enumitem}

\usepackage[utf8]{inputenc} 
\usepackage[T1]{fontenc}    
\usepackage{hyperref}       
\usepackage{url}            
\usepackage{booktabs}       
\usepackage{amsfonts}       
\usepackage{nicefrac}       
\usepackage{microtype}      
\usepackage{xcolor}         

\usepackage{graphicx} 
\usepackage[utf8]{inputenc}
\usepackage{hyperref}

\usepackage{bm}
\usepackage{xspace}
\usepackage{amssymb}
\usepackage{mathtools}
\usepackage{amsthm}
\usepackage{amsmath}
\usepackage[svgnames]{xcolor}

\usepackage{cleveref}
\usepackage{listings}
\usepackage{booktabs}
\usepackage{mathtools} 
\usepackage{tikz}
\usepackage{anyfontsize} 
\usetikzlibrary{arrows, positioning}
\usepackage{titletoc}
\usepackage{tabularx}   
\usepackage{booktabs}   
\usepackage{algorithm}
\usepackage{algcompatible}
\algnewcommand\algorithmicreturn{\textbf{Return} }
\algnewcommand\RETURN{\State \algorithmicreturn}%

\usepackage{titletoc}

\theoremstyle{plain}
\newtheorem{theorem}{Theorem}[section]
\newtheorem{proposition}[theorem]{Proposition}

\newtheorem{corollary}[theorem]{Corollary}
\theoremstyle{definition}

\theoremstyle{remark}
\newtheorem{remark}[theorem]{Remark}

\definecolor{tabblue}{RGB}{31, 119, 180}
\definecolor{taborange}{RGB}{255, 127, 14}

\newcommand{\newmethodlin}{\texttt{HiLoFi}\xspace}
\newcommand{\newmethodlinlow}{\texttt{LoLoFi}\xspace}
\newcommand{\newmethodvb}{\texttt{OnVBLL}\xspace}
\newcommand{\methodlaplace}{\texttt{LLL}\xspace}
\newcommand{\methodvbll}{\texttt{VBLL}\xspace}
\newcommand{\methodlrkf}{\texttt{LRKF}\xspace}
\newcommand{\methodlofi}{\texttt{LoFi}\xspace}
\newcommand{\methodgp}{\texttt{GP}\xspace}
\newcommand{\methodshampoo}{\texttt{muon}\xspace}

\newcommand{\eat}[1]{} 

\newcommand{\E}{\mathbb{E}}
\newcommand{\cond}{\,\vert\,}

\DeclareMathOperator*{\argmax}{arg\,max}
\DeclareMathOperator*{\argmin}{arg\,min}
\DeclareMathOperator{\Tr}{Tr}

\renewcommand{\d}{\mathrm{d}}
\newcommand{\cov}{{\rm Cov}}
\newcommand{\var}{{\rm Var}}

\newcommand{\plast}{{\vell}}  
\newcommand{\phidden}{{\bm h}} 
\newcommand{\nn}{{f}}  

\newcommand{\dimall}{{D_\vtheta}}
\newcommand{\dimsub}{d}
\newcommand{\dimlastsub}{{d_{\plast}}}
\newcommand{\dimhiddensub}{{d_{\phidden}}}
\newcommand{\dimlast}{{D_{\plast}}}
\newcommand{\dimhidden}{{D_{\phidden}}}
\newcommand{\dimobs}{{D_\vy}}
\newcommand{\dimin}{{D_\vx}}
\newcommand{\dimcontext}{{D_c}}
\newcommand{\dimaction}{{D_a}}

\newcommand{\varinnov}{\vS}

\newcommand{\covlast}[1]{\vSigma_{\plast,#1}}
\newcommand{\covhidden}[1]{\vSigma_{\phidden,#1}}
\newcommand{\lrhidden}{\vC} 
\newcommand{\meanhidden}[1]{\phidden_{#1|#1}}
\newcommand{\meanlast}[1]{\plast_{#1|#1}}
\newcommand{\gradhidden}{\tilde{\vH}}
\newcommand{\gradlast}{\tilde{\vL}}
\newcommand{\gradall}{\tilde{\vJ}}

\newcommand{\reals}{\mathbb{R}}
\newcommand{\naturals}{\mathbb{N}}
\newcommand{\mat}[1]{{\mathcal{M}_{#1}(\mathbb{R})}}
\newcommand{\matSemiPos}[1]{{\mathcal{M}^{+}_{#1}(\mathbb{R})}}
\newcommand{\matPos}[1]{{\mathcal{M}^{++}_{#1}(\mathbb{R})}}

\newcommand{\myvec}[1]{\mathbf{#1}}
\newcommand{\myvecsym}[1]{\boldsymbol{#1}} 

\newcommand{\vzero}{\myvecsym{0}}

\newcommand{\vepsilon}{\myvecsym{\epsilon}}

\newcommand{\vell}{\myvecsym{\ell}}

\newcommand{\vmu}{\myvecsym{\mu}}

\newcommand{\vtheta}{\myvecsym{\theta}}

\newcommand{\vSigma}{\myvecsym{\Sigma}}

\newcommand{\va}{\myvec{a}}
\newcommand{\vb}{\myvec{b}}

\newcommand{\vc}{\myvec{c}}

\newcommand{\ve}{\myvec{e}}

\newcommand{\vh}{\myvec{h}}

\newcommand{\vm}{\myvec{m}}

\newcommand{\vs}{\myvec{s}}

\newcommand{\vu}{\myvec{u}}
\newcommand{\vv}{\myvec{v}}

\newcommand{\vx}{\myvec{x}}

\newcommand{\vy}{\myvec{y}}

\def\vtheta{{\bm{\theta}}}
\def\va{{\bm{a}}}
\def\vb{{\bm{b}}}
\def\vc{{\bm{c}}}

\def\ve{{\bm{e}}}

\def\vh{{\bm{h}}}

\def\vm{{\bm{m}}}

\def\vs{{\bm{s}}}

\def\vu{{\bm{u}}}
\def\vv{{\bm{v}}}

\def\vx{{\bm{x}}}
\def\vy{{\bm{y}}}

\newcommand{\vA}{\myvec{A}}
\newcommand{\vB}{\myvec{B}}
\newcommand{\vC}{\myvec{C}}

\newcommand{\vE}{\myvec{E}}

\newcommand{\vH}{\myvec{H}}
\newcommand{\vI}{\myvec{I}}
\newcommand{\vJ}{\myvec{J}}
\newcommand{\vK}{\myvec{K}}
\newcommand{\vL}{\myvec{L}}
\newcommand{\vM}{\myvec{M}}

\newcommand{\vN}{\myvec{N}}

\newcommand{\vQ}{\myvec{Q}}
\newcommand{\vR}{\myvec{R}}
\newcommand{\vS}{\myvec{S}}

\newcommand{\vU}{\myvec{U}}
\newcommand{\vV}{\myvec{V}}
\newcommand{\vW}{\myvec{W}}

\newcommand{\data}{\mathcal{D}}

\title{
Martingale Posterior Neural Networks for Fast Sequential Decision Making
}

\author{%
Gerardo Duran-Martin$^{1}$ \quad Leandro Sánchez-Betancourt$^{1,2}$ \\
\textbf{Álvaro Cartea}$^{1,2}$ \quad \textbf{Kevin Murphy}$^{3}$\\
$^1$ Oxford-Man Institute of Quantitative Finance \\ $^2$ Mathematical Institute, University of Oxford \\ $^3$Google Deepmind\\
\\
\texttt{gerardo.duran-martin@eng.ox.ac.uk}\\
\texttt{\{leandro.sanchezbetancourt,alvaro.cartea\}@maths.ox.ac.uk}\\
\texttt{kpmurphy@google.com}
}

\begin{document}

\maketitle

\begin{abstract}
We introduce scalable algorithms for online learning of neural network parameters and Bayesian sequential decision making.
Unlike classical Bayesian neural networks,
which induce predictive uncertainty through a posterior over model parameters,
our methods adopt a predictive-first perspective based on martingale posteriors.
In particular, we work directly with the one-step-ahead posterior predictive, which we
parameterize with a neural network and update sequentially with incoming observations.
This decouples Bayesian decision-making from parameter-space inference:
we sample from the posterior predictive for decision making,
and update the parameters of the posterior predictive via fast, frequentist Kalman-filter-like
recursions. 
Our algorithms operate in a fully online, replay-free setting, providing principled uncertainty quantification without costly posterior sampling.
Empirically, they achieve competitive performance–speed trade-offs in non-stationary contextual bandits and Bayesian optimization,
offering 10–100 times faster inference than classical Thompson sampling while maintaining comparable or superior decision performance.
\end{abstract}


\section{Introduction}
\label{sec:introduction}

In various sequential decision-making problems, such as Bayesian optimization and contextual bandits, uncertainty quantification and efficient belief updates are essential for balancing exploration and exploitation.
A prominent Bayesian approach is Thompson sampling (TS) \cite{russo2020tutorialthompsonsampling},
which samples from the posterior over model parameters and evaluates the corresponding predictive function
(e.g., reward or surrogate) for decision-making.

Various works combine Bayesian uncertainty with the expressiveness of neural networks 
for sequential decision making
\cite{brunzema2024bayesian,riquelme2018banditbayes}.
These methods, often called Bayesian neural networks (BNNs),
typically approximate the posterior over model parameters (e.g., via variational inference) and then use this posterior for downstream tasks.
However, BNNs often rely on misspecified priors and likelihoods \cite{knoblauch2022optimization}, which often degrade predictive performance \cite{wenzel2020good} and result in slower inference than non-Bayesian alternatives \cite{lakshminarayanan2017simple}.

These limitations are especially problematic in online settings where updates and decisions must be made quickly,
such as in recommender systems, adaptive control, financial forecasting, and large-scale Bayesian optimization
\cite{shen2015portfolio,liu2022monolith,mcdonald2023impatient,cartea2023bandits,waldon2024dare}.

To address these problems,
we present a class of algorithms for Bayesian sequential decision making
inspired by the predictive-first philosophy of the martingale posterior framework
\cite{fong2023martingaleposterior,holmes2023statistical}.
Our approach operates in a fully online setting:
the one-step-ahead posterior predictive is parameterized by a neural network,
updated directly from real observations via frequentist Kalman filter-style recursions,
and sampled once per step for decision making.
This avoids the need for costly posterior inference,
while still supporting uncertainty and sampling-based sequential decision making.

The parameters defining the predictive density are updated with a single-pass frequentist Kalman filter-style methods,
which resemble online natural gradient steps \cite{ollivier2018expfamekf} and require no replay buffer
(copies of past observations, which one typically uses for multiple inner-iterations).
We explore three structured covariance strategies:
\methodlrkf applies low-rank updates to all layers;
\newmethodlin applies a full-rank update to the last layer and a low-rank update to the hidden layers;
\newmethodlinlow applies low-rank updates to both.
\newmethodlin and \newmethodlinlow are inspired by last-layer Bayesian methods \cite{harrison2024variational},
and offer different tradeoffs between computational efficiency and expressiveness.

Decision making is performed by sampling directly from the posterior predictive,
resulting in 10 to 100 times faster inference than that of classical TS,
even when compared with structured posteriors that use diagonal plus low-rank covariances \cite{chang23lofi}.
Our algorithms scale to million-parameter networks, maintain constant-time updates, and operate in fully online, streaming-compatible settings.
In addition, our methodology supports any acquisition strategy, such as classical TS or expected improvement.

In summary, our probabilistic method for training neural networks enjoys the following properties:
(i) efficient closed-form updates (no Monte Carlo sampling required);
(ii) sample efficiency through low-rank Kalman filter-like updates which do not require the specification of a posterior density;
(iii) closed-form uncertainty-aware predictions via the posterior predictive.

We evaluate the performance of our methods across a range of sequential decision-making tasks.
In non-stationary neural contextual bandits, our methods achieve the highest performance at the lowest computational cost compared to that of baselines.
In stationary settings such as Bayesian optimization, our approach matches the performance of replay-buffer methods
while achieving Pareto-efficient tradeoffs between runtime and performance.
Our code is available at \url{https://github.com/gerdm/martingale-posterior-neural-networks}.

A complete summary of the notation used throughout the paper is in Appendix~\ref{sec:notation}.

\section{Problem statement}
We consider a sequential-decision-making agent
that at time $t$ observes a context-action pair $\vx_t = (\vc_t, \va_t)$
with
context $\vc_t \in \reals^{\dimcontext}$ (possibly empty)
and
action $\va_t \in {\cal A}\subseteq \reals^\dimaction$.
The environment returns a reward (or observation) according to
\begin{equation}
    \vy_t \sim p_{\rm env}(\cdot \cond \vx_t), 
    \qquad \vy_t \in \reals^{\dimobs},
\end{equation}
which depends only on the context-action pair $\vx_t$.
We use the notation $y_t$ when the outcome is scalar ($\dimobs=1$) 
and $\vy_t$ when it is vector-valued ($\dimobs > 1$).

This formulation covers, e.g.,
Bayesian optimization ($\vc_t$ empty and $\va_t$ the query location), multi-armed bandits ($\vc_t$ empty and $\va_t$ the chosen arm),
and
contextual bandits ($\vc_t$ non-empty and $\va_t$ the chosen arm).
It also serves as a proxy representation for multi-armed bandits by letting the components of $\vy_t$ encode per-arm signals or auxiliary outcomes.

We write
$\data_t = (\vx_t, \vy_t)$ for a datapoint and $\data_{1:t} = \{\data_1, \ldots, \data_t\}$ for the dataset.
Here, we assume that $p_{\rm env}$ is unknown,
but we approximate it through the observation model
\begin{equation}
    \vy_t = f(\vtheta_t, \vx_t) + \ve_t,
\end{equation}
with $\ve_t$ a zero-mean random vector with covariance matrix $\vR_t \in \matPos{\dimobs}$,
$f$ a neural network,
and $\vtheta_t \in \reals^\dimall$ unknown time-varying model parameters.

The agent maintains a belief state
$\vb_t$ that summarizes its knowledge of the environment after observing $\data_{1:t}$.
This belief defines a parametric approximation to the one-step-ahead posterior predictive
\begin{equation}\label{eq:param-posterior-predictive}
    p_{\vb_t}(\vy_{t+1} \cond \vx_{t+1}) \approx p(\vy_{t+1} \cond \vx_{t+1}, \data_{1:t}).
\end{equation}
Our objective is to maintain $\vb_t$ via efficient, single-pass frequentist updates
and to use the predictive
$p_{\vb_t}(\vy_{t+1}\mid \vx_{t+1})$ for probabilistic, sample-based sequential decision making.
In this sense, we combine frequentist recursive parameter estimation with Bayesian decision making via the posterior predictive.

\section{Background}

\subsection{Tools for sequential decision making}
\label{sec:TS}

A central challenge in sequential decision making is the exploration-exploitation tradeoff,
where an agent balances whether to use its current knowledge of the environment (exploitation)
or to acquire new information to improve future  decisions (exploration) \cite{berger2014exploration}.

For example, when solving Bayesian neural contextual bandit problems \cite{riquelme2018banditbayes}
with the  classical Thompson sampling (TS) algorithm \cite{russo2020tutorialthompsonsampling},
an agent samples a parameter vector from the posterior distribution (or an approximation \cite{Phan2019}),
evaluates the reward function under this sampled parameter for each action,
and selects the action with the highest sampled reward.

Here, we propose a novel approach that avoids sampling from (high-dimensional) parameter posteriors
often found in Bayesian neural networks.
Instead, we work directly with the posterior predictive distribution defined by the current belief state
$\vb_t$.
For each candidate action, we sample a possible outcome from \eqref{eq:param-posterior-predictive}
and then choose the action with the highest sampled reward.
The belief $\vb_t$ is subsequently updated via frequentist recursions (Section \ref{sec:EKF-summary}).


Table~\ref{tab:TS-vs-predictive} highlights the key differences between our predictive sampling approach and classical TS.
The essential distinction is whether uncertainty is represented in parameter space (TS) or directly in outcome space
(our approach).
\begin{table}[htb]
    \centering
    \scriptsize
    \begin{tabular}{p{0.10\linewidth} p{0.40\linewidth} p{0.40\linewidth}}
        \toprule
        {Step} & {Predictive sampling (Our approach)} & {Classical TS} \\
        \midrule
        Sample & \textbf{Posterior predictive at each action} & \textbf{Posterior over model parameters} \\
        Evaluate & \textbf{N/A} & \textbf{Sampled parameters on reward function} \\
        Select action & Argmax over sampled rewards & Argmax over sampled rewards \\
        Get reward & $y_{t+1} \sim p_{\rm env}(\cdot \mid \vx_{t+1})$ & $y_{t+1} \sim p_{\rm env}(\cdot \mid \vx_{t+1})$ \\
        Update belief & \textbf{Frequentist update} (e.g., Algorithm \ref{algo:low-rank-full-rank-update}) & \textbf{Bayesian posterior update over parameters} \\
        \bottomrule
    \end{tabular}
    \caption{
    Comparison between Bayesian predictive sampling (our approach) and classical Thompson sampling.
     Differences highlighted in bold.
    }
    \label{tab:TS-vs-predictive}
\end{table}
Algorithm~\ref{algo:predictive-sequential-decision-making} shows our predictive sampling procedure
for contextual bandits.
At each decision-making step, the agent samples rewards from the predictive distribution for each action,
selects the action with the highest sampled reward, observes the reward from the environment, and updates its belief state.
\begin{algorithm}[htb]
\small
\begin{algorithmic}[1]
    \REQUIRE $\vb_t$ \texttt{// belief state obtained at time $t$}
    \REQUIRE $\vc_{t+1}$ \texttt{// context}
    \REQUIRE $p_\vb(\cdot \cond \vx)$ \texttt{// posterior predictive density parameterized by $\vb$}
    \FORALL{$\va \in {\cal A}$}
        \STATE $\hat{\vx} \gets (\vc_{t+1},\,\va)$
        \STATE $\hat{\vy}_{t,\va} \sim p_{\vb_t}(\cdot \mid \hat{\vx})$
    \ENDFOR
    \STATE $\va_{t+1} \gets \argmax_{a \in {\cal A}} \hat{\vy}_{t+1,\va}$
    \STATE $\vx_{t+1} \gets (\vc_{t+1},\,\va_{t+1})$
    \STATE $y_{t+1} \sim p_{\mathrm{env}}(\cdot \mid \vx_{t+1})$ \texttt{// observe outcome}
    \STATE $\vb_{t+1} \gets \texttt{Update}(\vb_t,\,y_{t+1},\,\vx_{t+1})$ \texttt{// e.g, Algorithm  \ref{algo:low-rank-full-rank-update}}
    \RETURN $(\vb_{t+1}, y_{t+1})$ \texttt{// belief and reward}
\end{algorithmic}
\caption{
Predictive sampling for sequential decision making in contextual bandits.
\footnotesize
}
\label{algo:predictive-sequential-decision-making}
\end{algorithm}

\subsection{Extended Kalman filtering for online learning and sequential decision making}
\label{sec:EKF-summary}

To maintain and update the belief state $\vb_t$, we adapt ideas from the frequentist perspective of the extended Kalman filter (EKF).
Here,
$\vb_t = (\vtheta_{t|t},\,\vSigma_t)$,
$\vtheta_{t|t} = \vA_t\,\vy_{1:t}$,
$\vA_t = \argmin_{\vA}\E[\|\vtheta_t - \vA\,\vy_{1:t}\|_2^2]$ is
the best linear unbiased predictor (BLUP) of $\vtheta_t$,
and
$\vSigma_t = \var(\vtheta_t - \vtheta_{t|t})$ is
the error-variance-covariance (EVC) matrix of the estimate.
The initial belief state $\vb_0 = (\vtheta_{0|0},\,\vSigma_{0})$ is given.

Next, we partition the neural network parameters into last-layer weights and hidden-layer weights: $\vtheta_t = (\plast_t, \phidden_t)$,
with $\plast_t \in \reals^\dimlast$ (last layer)
and
$\phidden_t \in \reals^\dimhidden$ (hidden layers)
so that
$\vtheta_{t|t} = (\plast_{t|t}, \phidden_{t|t})$, and
$\nn(\plast_t, \phidden_t, \vx_t) = \plast_t^\intercal\,\phi(\vx_t, \phidden_t) \in \reals$,
where $\phi(\vx_t, \phidden_t)$ is the hidden-layer feature map of the neural network.

To obtain closed-form updates,
we linearize the neural network around the previous parameter estimate
$\vtheta_{t-1|t-1} = (\plast_{t-1|t-1}, \phidden_{t-1|t-1})$,
and make a random-walk assumption for the model parameters
$\vtheta_t = \vtheta_{t-1} + \vu_t$,
with $\var(\vu_t) = \vQ_t\in\matPos{\dimall}$.
This assumption prevents uncertainty collapse, adaptation in non-stationary environments, and numerical stability
(see e.g., \cite{mehra2003approaches} or Section 2.6.4 in \cite{duranmartin2025thesis}).

With the above assumptions, we obtain the linearized state-space model
\begin{align}
    \vtheta_t &= \vtheta_{t-1} + \vu_t, \label{eq:ekf-dynamics}\\
    \vy_t &=
    \nn(\plast_{t-1|t-1},\,\phidden_{t-1|t-1},\,\vx_t)
    + \gradlast_t(\plast_t - \plast_{t-1|t-1})
    + \gradhidden_t(\phidden_t - \phidden_{t-1|t-1})
    + \ve_t, \label{eq:ekf-measurement-model}
\end{align}
where
$\ve_t$ is a zero-mean random variable with $\var[\ve_t] = \vR_t \in \matPos{\dimobs}$,
$\gradlast_t = \nabla_{\plast}\nn(\plast_{t-1|t-1},\phidden_{t-1|t-1},\vx_t) \in \mat{1\times2}$,
and
$\gradhidden_t = \nabla_{\phidden}\nn(\plast_{t-1|t-1},\phidden_{t-1|t-1},\vx_t)$
are the Jacobians of the neural network w.r.t. the last and hidden layers respectively.

\paragraph{Updates.}
The state-space model defined by \eqref{eq:ekf-dynamics}-\eqref{eq:ekf-measurement-model}
allows for recursive estimation of $\vb_t$ via Kalman-like updates.
See Appendix \ref{sec:EKF} for details.



\paragraph{Posterior predictive model.}
At each timestep the one-step-ahead posterior predictive induced by the state-space model
\eqref{eq:ekf-dynamics}-\eqref{eq:ekf-measurement-model},
given the belief state $\vb_t$ and under the assumption $p(\ve_t) = {\cal N}(\ve_t \mid \vzero,\,\vR_t)$ is
\begin{equation}\label{eq:pp-ekf}
    p_{\vb_t}(\vy_{t+1} \mid \vx_{t+1})
    = \mathcal{N}\,\left(
        \vy_{t+1} \,\middle|\,
        \nn(\vtheta_{t|t}, \vx_{t+1}),\;
        \underbrace{\gradall_{t+1}\,\vSigma_{t}\,\gradall_{t+1}^\intercal}_{\text{epistemic}}
        + \underbrace{\vR_t}_{\text{aleatoric}}
    \right),
\end{equation}
where
$\gradall_t = \begin{bmatrix} \gradlast_t & \gradhidden_t \end{bmatrix}$.
Here, the covariance of the posterior predictive decomposes into epistemic uncertainty
(from parameter uncertainty under the linearization) and aleatoric uncertainty (from observation noise $\vR_t$).
The condition $\vR_t \in \matPos{\dimobs}$ ensures that~\eqref{eq:pp-ekf} defines a proper Gaussian density.

\subsection{Martingale posteriors}
Our notion of Bayesian uncertainty is inspired by the martingale posterior framework \cite{fong2023martingaleposterior}, 
which interprets Bayesian uncertainty through missing data \cite{holmes2023statistical}.
In the original setting, one imagines imputing an infinite sequence of future observations; 
statistics computed on these extended datasets form a martingale, 
and under mild conditions they converge.

However, in sequential decision making, however,
given the context-action pair $\vx_{t+1}$,
the only missing data,
is the next observation $\vy_{t+1}$ from the environment.
In our case, uncertainty is quantified to guide decisions:
$\vy_{t+1}$ (or a partial observation of it, in the bandit setting) is always revealed before the next step.
Thus, rather than simulating an infinite sequence of future steps, 
we focus entirely on the one-step-ahead posterior predictive $p_{\vb_t}(\vy_{t+1}\mid \vx_{t+1})$.

\section{Method}
\label{sec:our-method}
\label{sec:method}

The standard EKF strategy for online learning presented in Section \ref{sec:EKF-summary} does not scale to large neural networks
because the memory cost is $O(\dimall^2)$ and the compute cost is $O(\dimall^3)$.
A number of approaches have been proposed to tackle this issue
(see Section \ref{sec:related} for details).
Our method leverages the frequentist perspective of filtering to derive low-rank updates
and the Bayesian interpretation of filtering to make sequential decisions through the posterior predictive.

\subsection{Algorithm}
\label{sec:algorithm}

Consider the linearized model \eqref{eq:ekf-measurement-model},
with initial conditions
\begin{equation}
\begin{aligned}
    \mathbb{E}[\plast_0] &= \meanlast{0}, & \mathbb{E}[\phidden_0] &= \meanhidden{0}, &
    \var(\plast_0) &= \hat{\vSigma}_{\plast, 0}, & \var(\phidden_0) &=\lrhidden_{0}^\intercal\,\lrhidden_{0},
\end{aligned}
\end{equation}
for known
$\lrhidden_{0} \in \mat{\dimhiddensub\times\dimhidden}$,
$\meanhidden{0} \in \reals^\dimhidden$,
$\hat{\vSigma}_{\plast, 0} \in \matPos{\dimlast}$, and
$\meanlast{0} \in \reals^\dimlast$.
Next, assume that for $t\geq 1$, the last layer parameters $\plast_{1:t}$ and the hidden layer parameters $\phidden_{1:t}$
follow the dynamics 
\begin{equation}\label{eq:main-ssm}
    \plast_t = \plast_{t-1} + \vu_{\plast,t}, \qquad \phidden_t = \phidden_{t-1} + \vu_{\phidden,t},\\
\end{equation}
where $\vu_{\plast,1:t}$ and $\vu_{\phidden,1:t}$ are zero-mean independent noise variables
with covariance matrices
$\vQ_{\plast,t} = q_{\plast,t} \vI_\dimlast$ and
$\vQ_{\phidden,t} = q_{\phidden,t} \vI_\dimhidden$,
with $q_{\plast,t} \geq 0$
and
$q_{\phidden,t} \geq 0$.
Our  method
maintains a belief state for the state-space model
represented by $\vb_t = (\meanlast{t}, \meanhidden{t}, \hat{\vSigma}_{\plast, t}^{1/2}, \lrhidden_{t})$,
where
$\meanlast{t} \in \reals^\dimlast$ is the estimate for model parameters in the last layer,
$\meanhidden{t} \in \reals^\dimhidden$ is the estimate  for model parameters in the hidden layers,
$\hat{\vSigma}_{t-1}^{1/2} \in \matPos{\dimlast}$ is the Cholesky EVC factor for the covariance of the last layer,
and
$\lrhidden_{t} \in \mat{\dimhiddensub\times\dimhidden}$ is the low-rank EVC factor for the hidden layers.

\paragraph{Updates.}
Given $\vb_{t-1}$,
the updated $\vtheta_{t|t} = (\plast_{t|t},\,\phidden_{t|t})$
follows directly from the EKF equations
which takes the form
$\vtheta_{t|t} = \vtheta_{t-1|t-1} + \vK_t \vepsilon_t$,
where $\vepsilon_t = \vy_t - \hat{\vy}_t$
is the error or innovation term,
and $\vK_t$ is the gain matrix.
Because of the block-diagonal assumption, it can be shown that
$\vK_t$ is decoupled into last-layer-only and a hidden-layers-only submatrices.

Next, applying the EKF EVC update to $\vb_{t-1}$
yields the dense matrix $\vSigma_t \in \matSemiPos{\dimall}$
which couples the effects from the last and hidden layers.
To obtain a block-diagonal matrix with components $(\hat{\vSigma}_{\plast, t}^{1/2}, \lrhidden_{t})$,
we consider a two-stage approximation.
The first stage replaces $\vSigma_t$ with a surrogate covariance matrix $\tilde{\vSigma}_t$
that avoids the interaction between the hidden layer parameters and the synthetic dynamics $q_{\phidden,t}$;
the second stage approximates the surrogate matrix with the best block-diagonal approximation
whose first block (for last-layer parameters)
are of rank $\dimlast$
and
the second block (for hidden-layer parameters)
are rank-$\dimhiddensub$.
Finally,
to maintain numerical stability and low-memory updates,
we track the Cholesky factor for the last layer and a low-rank factor for the hidden layers, i.e.,
$\hat{\vSigma}_t^{1/2} = {\rm diag}(\lrhidden_{t},\,\hat{\vSigma}_{\plast, t}^{1/2})$.
In this sense, estimation of the factors for the last-layer and the hidden layer follows the sequence
\[
\vSigma_t \to \tilde{\vSigma}_t \to \hat{\vSigma}_t \to \hat{\vSigma}_t^{1/2} = {\rm diag}(\lrhidden_{t},\,\hat{\vSigma}_{\plast, t}^{1/2}).
\]


Algorithm \ref{algo:low-rank-full-rank-update} shows a single step of \newmethodlin.
The derivation is in Appendix \ref{sec:derivation-hilofi}.
\begin{algorithm}[htb]
\small
\begin{algorithmic}[1]
    \REQUIRE $\vR_t$ \texttt{// measurement variance}
    \REQUIRE $(q_{\plast,t},\,q_{\phidden, t})$  \texttt{// dynamics covariance for last layer and hidden layers}
    \REQUIRE $(\vy_t,\,\vx_t)$ \texttt{// observation and context-action pair}
    \REQUIRE $\vb_{t-1} = \left(\meanlast{t-1}, \meanhidden{t-1}, \hat{\vSigma}_{\plast, t-1}^{1/2} ,\lrhidden_{t-1}\right)$ \texttt{// previous belief}
    \Statex \texttt{// predict step}
    \STATE $\hat{\vy}_t \gets \nn(\plast_{t-1|t-1}, \phidden_{t-1|t-1}, \vx_t)$
    \STATE $\gradlast_t = \nabla_{\plast} \nn(\plast_{t-1|t-1},\,\phidden_{t-1|t-1},\,\vx_{t})$
    \STATE $\gradhidden_t = \nabla_{\phidden} \nn(\plast_{t-1|t-1},\,\phidden_{t-1|t-1},\,\vx_{t})$
    \Statex \texttt{// innovation (one-step-ahead error) and Cholesky innovation variance}
    \STATE $\vepsilon_t \gets \vy_t - \hat{\vy}_t$
    \STATE $
        \vS_t^{1/2}  = {\cal Q}_R\left(
        \hat{\vSigma}_{\plast, t-1}^{1/2}\,\gradlast_t^\intercal,\,
        \sqrt{q_{\plast,t}}\,\gradlast_t^\intercal,\,
        \lrhidden_{t-1}\,\gradhidden_t^\intercal,\,
        \sqrt{q_{\phidden, t}}\,\gradhidden_t^\intercal,\,
        \vR_t^{1/2}
    \right),
    $
    \Statex \texttt{// gain for hidden layers}
    \STATE $\vV_{\phidden,t} \gets \varinnov_t^{-1/2}\,\varinnov_{t}^{-\intercal/2}\gradhidden_t$
    \STATE $\vK_{\phidden, t}^\intercal \gets \vV_{\phidden,t}\,\lrhidden_{t-1}\,\lrhidden_{t-1}^\intercal + q_{\phidden,t}\,\vV_{\phidden, t}$
    \Statex \texttt{// gain for last layer}
    \STATE $\vV_{\plast,t} \gets \varinnov_t^{-1/2}\,\varinnov_{t}^{-\intercal/2}\gradlast_t$   
    \STATE $\vK_{\plast, t}^\intercal \gets \vV_{\plast, t}\,\covlast{t|t-1} + q_{\plast,t}\,\vV_{\plast, t}$
    \Statex \texttt{// mean update step}
    \STATE $\phidden_{t|t} \gets \phidden_{t-1|t-1} + \vK_{\phidden, t}\vepsilon_t$
    \STATE $\plast_{t|t} \gets \plast_{t-1|t-1} + \vK_{\plast, t}\vepsilon_t $
    \Statex \texttt{// EVC low-rank and Cholesky update step}
    \STATE $\lrhidden_{t} \gets {\cal P}_{\dimhiddensub, +q_{\phidden,t}}\left( \lrhidden_{t-1}\,\left(\vI - \vK_{\phidden, t}\,\gradhidden_t\right)^\intercal,\, \vR_t^{1/2}\vK_{\phidden, t}^\intercal\right)$
    \STATE $ \hat{\vSigma}_{\plast, t}^{1/2} \gets {\cal Q}_R\left( \hat{\vSigma}_{\plast, t-1}^{1/2},\left(\vI - \vK_{\plast, t}\,\gradlast_t\right)^\intercal,\, \vR_t^{1/2}\vK_{\plast, t}^\intercal\right)$
    \RETURN $\vb_{t} = \left(\meanlast{t}, \meanhidden{t}, \hat{\vSigma}_{\plast, t}^{1/2}, \lrhidden_{t}\right)$ \texttt{// updated belief}
\end{algorithmic}
\caption{
Single update step of \newmethodlin.
\footnotesize
The notation ${\cal Q}_R(\vB_1, \ldots, \vB_k)$
denotes the Cholesky factorization
of the matrix
$\sum_{k=1}^K \vB_k^\intercal\,\vB_k$,
where $\vB_k$ is an upper-triangular Cholesky factor;
see Appendix \ref{sec:sum-cholesky} for details.
The function ${\cal P}_d(\vA_1, \ldots, \vA_K)$
returns a $d\times D$ matrix which is the best rank-$d$ approximation 
of the matrix $\sum_{k=1}^K \vA_k^\intercal\,\vA_k$;
see Appendix \ref{sec:sum-SVD} for details.
}
\label{algo:low-rank-full-rank-update}
\end{algorithm}

The following proposition shows the upper bound of the the per-step error incurred by \newmethodlin.
\begin{proposition}[Covariance approximation error]\label{prop:hilofi-cov-per-step-error}
At time $t$, the per-step approximation error of \newmethodlin
under the linearized SSM \eqref{eq:ekf-measurement-model}
is bounded by
\begin{equation}
\|\vSigma_{t|t} - \hat{\vSigma}_{t|t}\|_{\rm F}
\leq q_{\phidden,t}\,\vE_{\phidden, {\rm surr}}
+ q_{\plast,t}\,\vE_{\plast, {\rm surr}}
+ \sqrt{\sum_{k=\dimhiddensub+1}^{\dimhidden}\lambda_k^2}
+ 2\|\vK_{\plast,t}\vR_t\vK_{\phidden,t}^\intercal\|_{\rm F}^2,
\end{equation}
where
$\vE_{\phidden, {\rm surr}} = (2\|\vK_{\phidden,t}\gradhidden_t\|_{\rm F} + \|\vK_{\phidden,t}\gradhidden_t\|_{\rm F}^2)$,
$
\vE_{\plast, {\rm surr}} =
\|(\vI - \vK_{\plast, t}\,\gradlast_t)\,(\vI - \vK_{\plast, t}\,\gradlast_t)^\intercal\|_{\rm F}
$, and
$\{\lambda_k\}_{k=\dimhiddensub+1}^\dimhidden$ are smallest $(\dimhidden - \dimhiddensub)$ eigenvalues of $\tilde{\vSigma}_t$.
\end{proposition}

\begin{remark}
    Proposition \ref{prop:hilofi-cov-per-step-error}
    shows that the per-step error in the \newmethodlin approximation
    is bounded by:
    (i) the neglected terms by each surrogate matrix,
    (ii) the low-rank approximation for the covariance of the hidden layers,
    and
    (iii) the cross-terms from the blocks in the off-diagonal
    Note, however, that the bound does not account for error due to linearization of the neural network.
    One could minimize the above lower bound by setting $q_{\phidden,t} = q_{\plast,t}=0$.
    However, this may result in lower per-step performance.
    See Appendix~\ref{sec:error-analysis-lrkf-mnist} for an example of the main sources of error.
    See Appendix \ref{proof:hilofi-cov-per-step-error} for a proof.
\end{remark}

\paragraph{Posterior predictive model.}
After every update, the posterior predictive model used for sequential decision making is
\begin{equation}\label{eq:pp-hilofi}
    p_{\vb_t}(\vy_{t+1} \cond \vx_{t+1})
    = {\cal N}\left(
    \vy_{t} \mid
    \nn(\plast_{t|t}, \phidden_{t|t}, \vx_{t+1}),\;\;
    \gradlast_{t+1}\,\covlast{t}\gradlast_{t+1}^\intercal + \gradhidden_{t+1}\,\covhidden{t}\gradhidden_{t+1}^\intercal
    + \vR_{t+1}
    \right).
\end{equation}


The computational complexity of \newmethodlin is
$O(\dimlast\,(\dimlast + \dimobs)^2 + \dimhidden\,(\dimhiddensub + \dimobs)^2)$.
A full breakdown of the computational complexity and its relationship to other methods is in Appendix \ref{sec:computational-complexity}.

\paragraph{Variants.}
Whenever the output layer
or the input to the last layer is high dimensional,
we perform a low-rank approximation to the last-layer to reduce computational
costs.
We call this variant \newmethodlinlow.
Its computational complexity is $O(\dimlast^3 + \dimhidden \dimhiddensub^2)$.
See Appendix \ref{sec:LoLoFi} for details.

A purely low-rank version of our method,
which we call \methodlrkf,
is discussed in Appendix \ref{sec:LRKF}.
This method uses a single low rank covariance matrix approximation for all parameters in the neural network.
An error analysis of \methodlrkf in the linear setting is presented in Appendix \ref{sec:error-analysis-lrkf}.
In the experiments,
we find that this approach performs worse than that of \newmethodlin, which is not surprising.
It also has lower performance than
 \newmethodlinlow
(at least on  the contextual bandit problem in \cref{experiment:mnist-bandits}),
presumably because modeling the last layer separately allows us to use a different subspace for the low rank approximation of the output layer weights.



\section{Experiments}
\label{sec:experiments}

In this section, we evaluate our method.
First, we consider a one-dimensional example 
``in-between'' uncertainty, to
illustrate the behavior of the method.
We then compare its performance to that of other methods 
across the following sequential decision-making tasks: 

\begin{enumerate}[label=(\Roman*)]
    
    \item \textbf{MNIST for classification as a bandit problem}:
    This is an online high-dimensional classification problem, which we study as a contextual bandit problem. \label{enum:mnist}

     \item \textbf{Recommender systems}: 
    This is a challenging real-world problem
    with  non-stationary data.
    Here, tackling exploration-exploitation tradeoff is key.  \label{enum:Bandits}
    
    \item \textbf{Bayesian optimization}: 
    The goal of this task is to find the maximum of an unknown blackbox function 
     $g: [0,1]^\dimin \to \mathbb{R}$.
     This is a static inference problem
    where sample efficiency 
    and uncertainty quantification are crucial.
    \label{enum:BO}

\end{enumerate}

For each of the above tasks, we assess the various methods on 
predictive performance and wall-clock time.
All experiments were run on a TPU v4-8.
An additional experiment demonstrating the scalability of \methodlrkf
to multi-million-parameter neural networks for online classification
is provided in Appendix~\ref{sec:cifar-10-vgg}.

We consider the following variants of our framework.
    \textbf{High-rank low-rank filter}
    (\newmethodlin):
    Our main method
    that uses full rank for last layer
    and low rank for hidden layers.
    \textbf{Low-rank low-rank filter}
    (\newmethodlinlow):
    A version of our method that uses low-rank matrices for last and hidden layers.    
    We only consider \newmethodlinlow in task \ref{enum:Bandits}, due to the relatively large dimension of the last layer.
    \textbf{Low-rank square-root Kalman Filter}
    (\methodlrkf): 
    This is a special case of our method, described in Appendix \ref{sec:LRKF},
    that uses a single low-rank covariance
    for all the parameters.
    This illustrates the gains from modeling the final layer separately.

As existing baselines, we consider the following.
  \textbf{Diagonal plus low-rank precision} (\methodlofi) \cite{chang23lofi}: 
    a fully online method that models the precision matrix using a low-rank plus diagonal matrix, and uses the same linearization scheme as us.
    \textbf{Gaussian processes} (\methodgp) \cite{xu2024standard}: 
    a standard approach to modeling predictive uncertainty.
    We only include GPs in the low-dimensional tasks in \ref{enum:BO} because of scalability limitations.
    \textbf{Variational Bayesian last layer} (\methodvbll) \cite{brunzema2024bayesian}: 
    a partially Bayesian method that requires access to the full dataset (or a replay buffer), 
    and performs multiple optimization steps per update.
    See Appendix \ref{sec: new method vb} for details.
    \textbf{Online variational Bayesian last-layer} (\newmethodvb): 
    a straightforward modification of the  {\methodvbll} method with a first-in-first-out (FIFO) buffer and no regularization. 
    %
   %
    %
    \textbf{Last-layer Laplace approximation with FIFO buffer} (\methodlaplace) \cite{daxberger2021laplace}: 
    a method that uses a Laplace approximation on the last layer, 
    with a  FIFO replay buffer and multiple inner updates per time step.
    Unless otherwise stated, we train the parameters of the neural network with the AdamW
    optimization algorithm \cite{loshchilov2019adamw}.

\subsection{In-between uncertainty}
\label{sec:in-between}

\paragraph{Problem description.}
In this experiment, we test the ability of \newmethodlin to capture the \textit{in-between uncertainty} \cite{foong2019between}
after a single pass of the data.
This concept refers to how well a model captures uncertainty in regions of input space that lie between,
or away from, observed data. Intuitively, we expect uncertainty to be higher in such unexplored areas.
This behavior is important for Bayesian decision-making problems for effictive balancing of exploration and exploitation.
We consider the one-dimensional dataset introduced in \cite{van2021feature}.

\paragraph{Model.}
We use an MLP with 4 hidden layers
and 128 units per layer, and we employ an ELU activation function.

\paragraph{Results.} 
Figure \ref{fig:in-between-uncertainty-by-timestep-main}
shows the evolution of the posterior predictive mean and the posterior predictive variance
as a function of the number of seen observations.
We see that \newmethodlin behaves in an intuitively sensible way, showing more uncertainty away from the data, just like a Gaussian process.
See Appendix \ref{experiment:in-between-uncertainty} for the performance of other methods
(specifically \methodlrkf  and \methodvbll)
 on this problem, as well an ablation study that illustrates the effect of changing rank.

\begin{figure}[htb]
    \centering
    \includegraphics[width=0.7\linewidth]{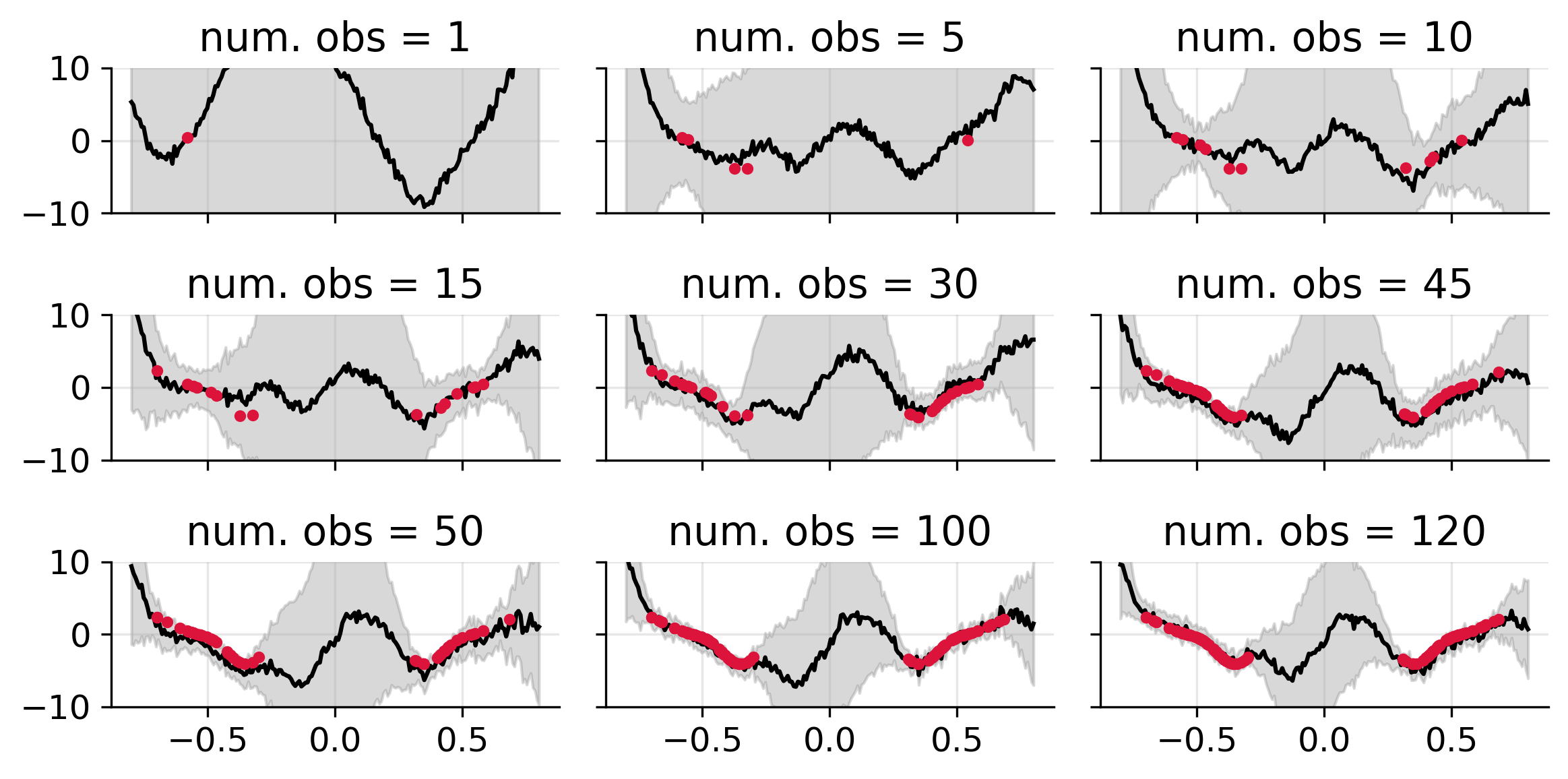}
    \vspace{-1em}
    \caption{
        In-between uncertainty  induced by \newmethodlin as a function of the processed observations.
    }
    \label{fig:in-between-uncertainty-by-timestep-main}
\end{figure}

\subsection{MNIST}
\label{experiment:mnist-bandits}
\paragraph{Problem description.}
We consider the MNIST
contextual bandit  task introduced by \cite{osband2019behaviour},
where the agent is presented with an image and must choose one of ten possible classes as an action,
and then gets a binary feedback, based on 
whether its prediction was correct or not; we refer to this as the incomplete information case.
An additional experiment on online classification is in Appendix \ref{sec:classification-as-regression}.

\paragraph{Model.} The predictive model (for the 10 label logits or the per-action rewards)
is represented using a modified LeNet5 CNN architecture \cite{lecun1998gradient} with ELU activation function.

\paragraph{Exploration Methods.}
We consider two exploration strategies:
Classical TS and our proposed predictive-sampling (PS)
as outlined in Table \ref{tab:TS-vs-predictive}.
Additional results comparison TS, PS, and $\epsilon$-greedy are in Appendix \ref{sec:further-results-mnist-bandit}.

\paragraph{Inference methods.}
We compare \methodlaplace, \methodlrkf, \methodlofi, and \newmethodlin.
For \methodlaplace, we use the online last-layer variant from \cite{daxberger2021laplace},
without hyperparameter re-optimization at each timestep.
Choice for hyperparameters are detailed in Appendix \ref{sec:further-results-mnist-bandit}.

\paragraph{Bandit results.}
Figure \ref{fig:mnist-reward-cumulative} shows the average cumulative reward over 10 runs of the contextual bandit version of MNIST.
We observe that the top three performing methods are \newmethodlin, \newmethodlinlow, and \methodlrkf.
This outcome is expected because
\newmethodlin employs a full-rank covariance matrix in the last layer, which is better than 
  a low-rank approximation of the covariance matrix associated with the last layer (\newmethodlinlow), or a low-rank 
 covariance to model all dependencies across layers
(\methodlrkf). 
For further results showing regret and comparison using $\varepsilon$-greedy, see Appendix \ref{sec:further-results-mnist-bandit}.

\begin{figure}[htb]
    \centering
    \includegraphics[width=0.8\linewidth]{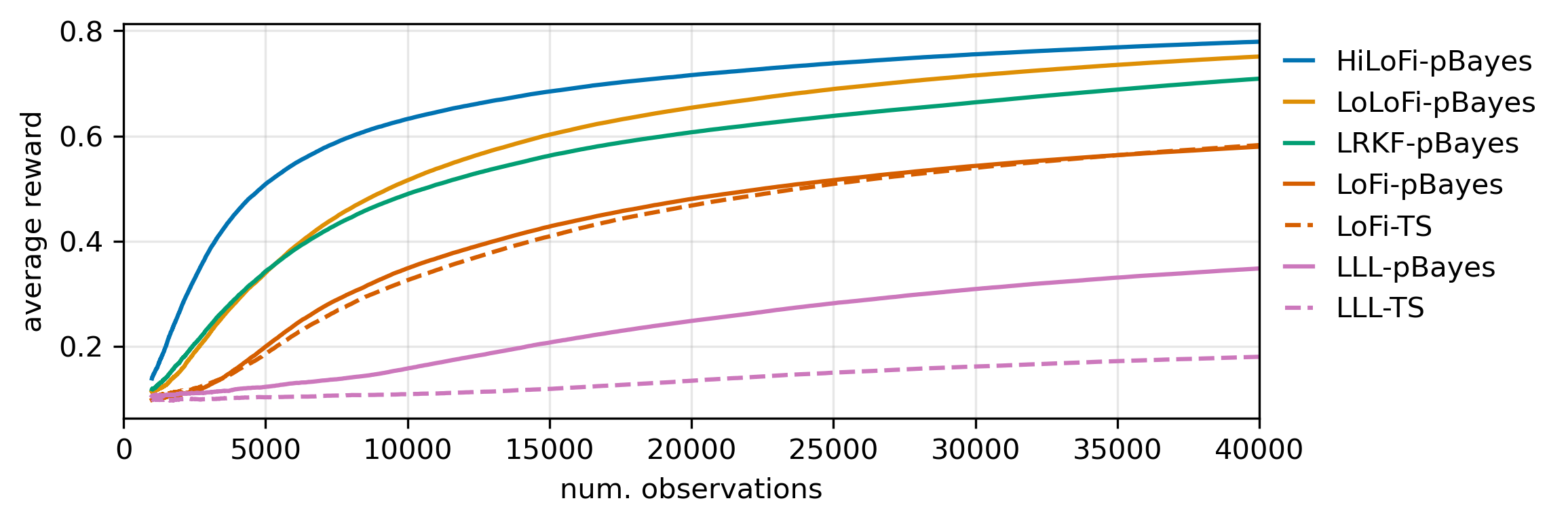}
    \vspace{-1.2em}
    \caption{
        Cumulative average reward for the bandit MNIST problem.
    }
    \label{fig:mnist-reward-cumulative}
\end{figure}
Next, Table~\ref{tab:ps-vs-ts} reports the average cumulative reward and total runtime (over 10 trials) for each method,
comparing posterior predictive sampling (pBayes) with Thompson sampling (TS).
\begin{table}[htb]
\tiny
\centering
\label{tab:ps-vs-ts}
\begin{tabular}{lrrrr}
\toprule
\textbf{Agent} & \textbf{Reward (pBayes)} & \textbf{Reward (TS)} & \textbf{Time (pBayes)} & \textbf{Time (TS)} \\
\midrule
\newmethodlin  & \textbf{31,167.7} & --      & 8.44  & --     \\
\newmethodlinlow  & 30,038.5          & --      & 3.67  & --     \\
\methodlrkf    & 28,354.9          & --      & 2.51  & --     \\
\methodlofi    & 23,184.5          & 23,304.3 & 4.76  & 38.90  \\
\methodlaplace   & 13,922.3          & 7,201.9  & 1.36  & 191.55 \\
\bottomrule
\end{tabular}
\caption{Performance comparison of Bayesian decision-making strategies: 
posterior predictive sampling (pBayes) vs. Thompson sampling (TS). 
Rewards are cumulative and averaged over 10 runs; time is in minutes.}
\end{table}
Overall, pBayes is consistently faster across all methods.
For the two agents where both approaches are applicable (\methodlofi and \methodlaplace),
the average cumulative reward is nearly identical for \methodlofi, while pBayes clearly outperforms TS for \methodlaplace.
However, TS incurs a much higher runtime due to the need to sample repeatedly from the high-dimensional posterior over model parameters,
whereas pBayes only requires sampling in the comparatively low-dimensional outcome space (one dimension per arm/class).
Finally, as expected, \newmethodlin achieves the best performance,
followed by \newmethodlinlow and then \methodlrkf.
This performance hierarchy comes at a cost in runtime: \newmethodlin is slower than \newmethodlinlow, which in turn is slower than \methodlrkf.

\subsection{Recommender systems}
\label{experiment:bandits}

\paragraph{Problem description.} We study the performance of the methods 
on the Kuairec dataset of \cite{gao2022kuairec},
which is derived from a real-world, non-stationary, video recommender system.
This dataset is used to
study  non-stationary contextual bandits
in
\cite{zhu2023nonstationarybandit}.

\paragraph{Dataset.} We group rows in the dataset (for a given user) in blocks of the
10 next videos that the user saw.
For the features, we consider the  $\log(1 + x)$ transform
of like count, share count, play count, comment count, complete play count, follow count, reply comment count, and download count.
For the target variable, we use the 
 $\log(1 + x)$ transform  of the watch ratio $[0, \infty)$.

\paragraph{Model.} The reward model is a neural network with an embedding layer (one per video/arm),
and a dense layer for all features.
We join the embedding layer and the  dense layer, and we
then consider a three-hidden-layer neural network (50 units per layer, ELU activations), and linear output unit.

\paragraph{Algorithm.} We perform sequential Bayesian inference using the relevant method,
and then use TS to choose the action at each step (see Section \ref{sec:TS}).
The choice of hyperparameters is  in  \ref{sec:further-results-recommend}.

\paragraph{Results.} Figure \ref{fig:bandits-reward-boxplot} shows the average daily reward (left panel) and the running time of each method (right panel).
We observe that \newmethodlin has the highest average daily reward with second lowest running time.
On average, \methodlrkf and \methodlofi have similar performance,
however, \methodlofi has a higher running time.
Next, \methodvbll has higher average daily reward than \newmethodvb, but it is slower.
For an in-depth analysis of the results see Appendix \ref{sec:further-results-recommend}. 

\begin{figure}[htb]
    \centering
    \includegraphics[width=0.9\linewidth]{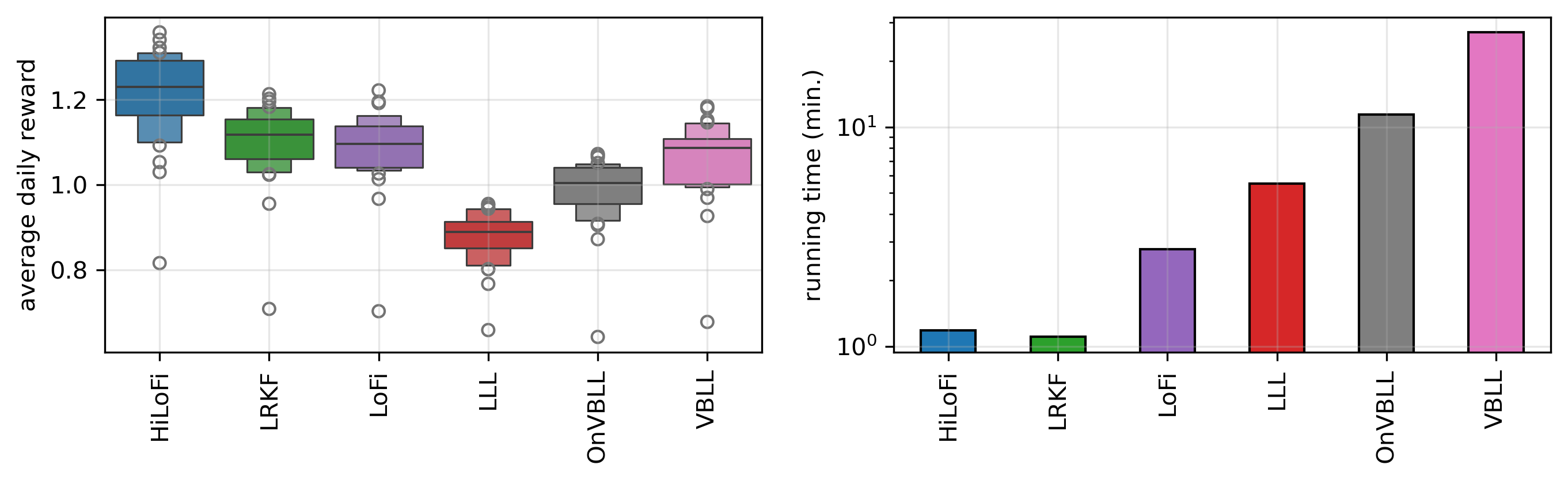}
    \vspace{-1.5em}
    \caption{
    Recommender system results.
     Left: average daily reward.
     Right: running time.
    }
    \label{fig:bandits-reward-boxplot}
\end{figure}

\subsection{Bayesian optimization}
\label{experiment:bayesopt}

\paragraph{Problem description.} Following \cite{brunzema2024bayesian}, we evaluate the competing methods on classical Bayesian optimization (BO) tasks, 
where the goal is to maximize an unknown function $g: [0,1]^\dimin \to \mathbb{R}$ 
with a surrogate model trained sequentially on past observations. 

\paragraph{Model.} We employ a three-layer MLP (180 units per layer, ELU activations) as the surrogate for all neural-based methods.
For the GP baseline,
we use a Matérn-5/2 kernel with lengthscale 0.1 and a 20-point buffer to reduce computational cost. 
While  \methodgp performs better when its hyperparameters are tuned during training  \cite{hvarfner2024bayesopt},
this incurs much higher computational costs, so we do not pursue it here.
The choice of hyperparameters is detailed in Appendix \ref{sec:further-results-BO}.

\paragraph{Datasets.}
We consider seven benchmark functions commonly used in BO \cite{brunzema2024bayesian}:
Ackley (2D, 5D, 10D, 50D), Branin (2D), Hartmann (6D), and DrawNN (50D and 200D). 
These span a range of dimensionalities and multi-modalities, and are standard in evaluating global optimal performance.
For all methods, except for DrawNN,
we sample a function from the posterior predictive and evaluate it on a fixed set of candidate points generated
using a Sobol sequence \cite{scyphers2024BO}
to find its maximum.
Instead, for Drawnn functions,
we use projected gradient descent,
implemented with the Jaxopt library \cite{jaxopt2021},
to optimize the sampled function directly over the continuous domain $[0,1]$.

\paragraph{Algorithm.}
We perform sequential Bayesian inference using the relevant method, and then choose the next query point at each step using TS
(see Section \ref{sec:TS}).
Figure \ref{fig:bayesopt-iterations-expected-improvement} in the Appendix
shows the results using the expected improvement algorithm \cite{jones1998expectedimprovment}.

\paragraph{Results.} Figure~\ref{fig:bayesopt-rank} (left) shows the final best value
on the vertical axis (higher is better),
and the amount of run time on the horizontal axis
(lower is better).
The number of steps for each dataset are in the x-axis in Figure \ref{fig:bayesopt-steps}
in Appendix \ref{sec:further-results-BO}.
Figure~\ref{fig:bayesopt-rank} (right) shows the tradeoff between performance (measured by rank)
and compute time.
Our method  is on the Pareto frontier for a range of tradeoffs between time and performance;
furthermore,  across tasks, the performance and compute time of \newmethodlin dominates
(better performance and less compute time) those from competing methods except from
\methodvbll (higher performance) and \methodlrkf (lower compute time).
Thus, \newmethodlin offers strong tradeoffs between compute time and performance.
For further analyses and results see Appendix \ref{sec:further-results-BO}.

\begin{figure}[htb]
    \centering
    \includegraphics[width=0.9\linewidth]{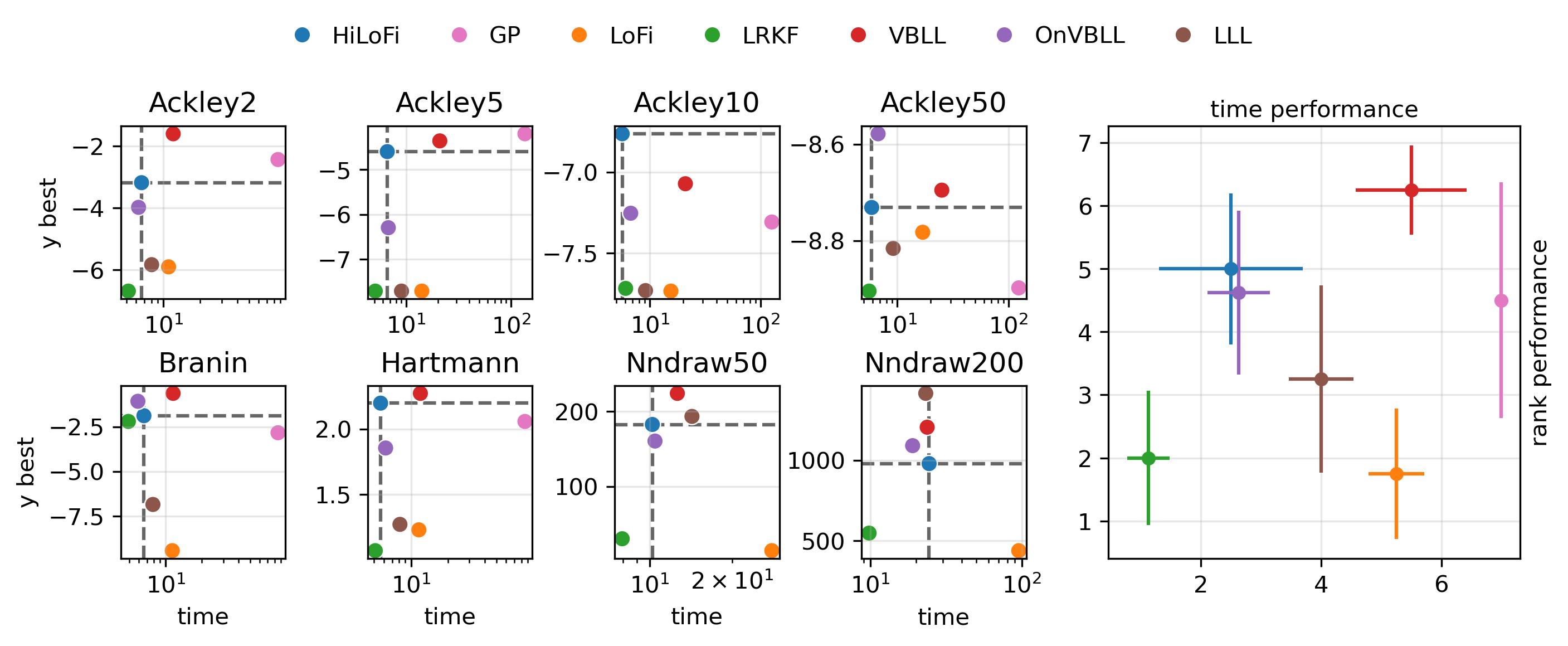}
    \vspace{-1em}
    \caption{
    Left panel:
        Performance across all BO benchmark functions.
        For time, lower is better.
        For performance, higher is better.
        The dashed lines correspond to the results
        for our method, so methods that are above and to the left are better.
        Right panel: Time versus performance tradeoff plots.
    }
    \label{fig:bayesopt-rank}
\end{figure}

\section{Related work}
\label{sec:related}

We briefly review deterministic approaches for approximate online parameter inference. 
Several methods maintain diagonal plus low-rank (DLR) structures for efficient covariance updates, 
including L-RVGA \cite{lambert2023lrvga}, SLANG \cite{mishkin2018slang}, and LoFi \cite{chang23lofi}. 
Low-rank Kalman filtering has also been explored in state-space models \cite{schmidt2023rankreducedkf}, 
though with stronger restrictions on the rank relative to measurement dimension. 
Beyond low-rank, FOO-VB \cite{zeno2021foovb} exploits Kronecker-structured approximations, 
while offline methods explore similar ideas for scaling BNNs 
\cite{lee2020bnnuncertainty,tomczak2020bnndrl,kristiadi2022posteriorpredictive}.
Another direction builds on the lottery ticket hypothesis \cite{li2018measuring,larsen2022degreesfreedomneedtrain}, 
restricting learning to sparse or low-dimensional subspaces. 
Examples include the subspace neural bandit \cite{duran2022efficient} and PULSE \cite{cartea2023detecting}, 
which pre-train projection matrices offline before online adaptation.

\section{Conclusion and future work}
\label{sec:conclusion}
We introduced a predictive-first approach for efficient online training of neural networks,
based on linearization and structured low-rank approximations of error covariances.
Our three variants achieve strong performance-runtime trade-offs,
with \newmethodlin particularly effective in high-dimensional, non-stationary settings.
Limitations include the lack of fixed-lag smoothing,
sensitivity to linearization and outliers  \cite{duran2024wolf,laplante2025robust},
and reliance on hyperparameters. 
Future work includes extending our approach to fully-online reinforcement learning \cite{elsayed2024streamingrl}.



We view our contributions as primarily methodological, with no particular ethical concerns.

\section*{Acknowledgements}
We thank the TPU Research Cloud program for providing the compute resources used in our experiments, and James Harrison for helpful comments.


\bibliographystyle{plainurl}
\bibliography{refs}

\newpage
\appendix
\startcontents[apx]

\section*{Appendix}
\printcontents[apx]{ }{1}[2]{} 

\newpage
\section{Notation}
\label{sec:notation}

\begin{table}[h!]
\small
\begin{tabularx}{\linewidth}{lX}
\toprule
\textbf{Symbol} & \textbf{Description} \\
\midrule
$\|\cdot\|_2$ & $\ell_2$ norm. \\
$\|\cdot\|_F$ & Frobenius norm. \\

\midrule
$\dimobs, \dimin \in \naturals$ & Dimension of observations / inputs. \\
$\dimhidden, \dimlast \in \naturals$ & Number of parameters for the hidden layer / last layer. \\
$\dimall = \dimhidden + \dimlast$ & Total number of neural network parameters. \\

\midrule
${\cal N}(\vx \cond \vm, \vS)$ & Multivariate Gaussian density evaluated at $\vx$ with mean $\vm$ and covariance matrix $\vS$. \\
$p_{\rm env}(\vy \cond \vx)$ & Unknown data-generating process for $\vy$ conditioned on $\vx$ \\
$p_{\vb}(\vy \cond \vx)$ & Posterior predictive model for $\vy$, conditioned on $\vx$, and parameterized by $\vb$.\\

\midrule
$\mat{m\times n}$ & Space of $m$-by-$n$ real-valued matrices. \\
$\matPos{m}$ & Space of $m$-dimensional positive definite (pd) matrices. \\
$\matSemiPos{m}$ & Space of $m$-dimensional positive semidefinite (psd) matrices. \\
$\vI_D$ & $D$-dimensional identity matrix. \\
$\vS^{1/2}$ & Upper-triangular Cholesky decomposition of the pd matrix $\vS$. \\

\midrule
$\vx_t = (\vc_t, \va_t)$ & Inputs with context $\vc_t \in \reals^{\dimcontext}$ (possibly empty) and action $\va_t \in {\cal A}\subseteq \reals^\dimaction$. \\
$y_t \in \reals$ & Scalar measurements, rewards, or observations obtained at timestep $t$. \\
$\vy_t \in \reals^\dimobs$ & Measurements, rewards, or observations obtained at timestep $t$. \\
$\data_t = (\vx_t, \vy_t)$ & Datapoint at time $t$. \\
$\vu_{1:t} = (\vu_1, \ldots, \vu_t)$ & Time-ordered collection of vectors $\vu_k \in \reals^{D_{\vu}}$. \\
$\data_{1:t} = (\data_1, \ldots, \data_{t})$ & Dataset at time $t$. \\

\midrule
$\plast \in \reals^\dimlast$ & Last layer parameters. \\
$\phidden \in \reals^\dimhidden$ & Hidden layer parameters. \\
$\vtheta = (\plast, \phidden) \in \reals^\dimall$ & Collection of all neural network parameters. \\
$\nn(\vtheta, \vx) = \nn(\plast, \phidden, \vx) \in \reals^\dimobs$ & Neural network output with parameters $\vtheta$ and inputs $\vx$. \\

\midrule
$\ve_t \in \reals^\dimobs$ & Zero-mean noise with covariance $\vR_t \in \matPos{\dimobs}$. \\
$\vu_t \in \reals^\dimall$ & Zero-mean dynamics noise with variance $\vQ_t \in \matSemiPos{\dimall}$. \\
$q_{\plast,t} \geq 0$ & Last-layer dynamics. \\
$q_{\phidden,t} \geq 0$ & Hidden-layer dynamics. \\
$\vu_{t,\plast} \in \reals^{\dimlast}$ & Zero-mean dynamics noise of last layer parameters with variance $q_{\plast,t}\,\vI_\dimlast$. \\
$\vu_{t,\phidden} \in \reals^\dimhidden$ & Zero-mean dynamics noise of hidden-layer parameters with variance $q_{\phidden,t}\,\vI_\dimhidden$. \\

\midrule
$\plast_{t|t} \in \reals^\dimlast$ & Frequentist estimate for $\plast_t$ given $\data_{1:t}$. \\
$\phidden_{t|t} \in \reals^\dimhidden$ & Frequentist estimate for $\phidden_t$ given $\data_{1:t}$. \\
$\vtheta_{t|t} = (\plast_{t|t}, \phidden_{t|t}) \in \reals^\dimall$ & Frequentist estimate for $\vtheta_t$ given $\data_{1:t}$. \\
$\vSigma_{t} = \var(\vtheta_t - \vtheta_{t|t})$ & Error variance-covariance matrix. \\
$\tilde{\vSigma}_t$ & Algorithm-dependent surrogate covariance matrix to $\vSigma_{t}$. \\
$\hat{\vSigma}_{t} = \argmin_{\vSigma:{\rm rank}(\vSigma) = d} \left\|\tilde{\vSigma}_{t|t} - \vSigma\right\|_{\rm F}^2$ & Best rank-$d$ approximation to the surrogate covariance matrix $\tilde{\vSigma}_{t}$. \\
$\vC_t \in \mat{d\times D}$ & Rectangular matrix $(d\times D)$ matrix such that $\vC_t^\intercal\,\vC_t = \hat{\vSigma}_t$\\

\midrule
$\gradlast_{t+1} = \nabla_{\plast} \nn(\plast_{t|t},\,\phidden_{t|t},\,\vx_{t+1})$ & Jacobian of neural network w.r.t parameters in last layer. \\
$\gradhidden_{t+1} = \nabla_{\phidden} \nn(\plast_{t|t},\,\phidden_{t|t},\,\vx_{t+1})$ & Jacobian of neural network w.r.t. parameter in hidden layers. \\

\midrule
${\cal Q}_R(\vB_1, \ldots, \vB_k)$ & Cholesky factorization of $\sum_{k=1}^K \vB_k^\intercal\,\vB_k$, where each $\vB_k$ is an upper-triangular Cholesky factor (see Appendix \ref{sec:sum-cholesky}). \\
${\cal P}_d(\vA_1, \ldots, \vA_K) \in \mat{d\times D}$ & Best rank-$d$ approximation of $\sum_{k=1}^K \vA_k^\intercal\,\vA_k$, with $\vA_k \in \mat{d_k\times D}$, $d_k \leq D$. \\
${\cal P}_{d,+a}(\vA_1, \ldots, \vA_K) \in \mat{d\times D}$ & Best rank-$d$ approximation of $\sum_{k=1}^K \vA_k^\intercal\,\vA_k + a\,\vI_D$, with $a > 0$ (see Appendix \ref{sec:sum-SVD}). \\
\bottomrule
\end{tabularx}
\end{table}

\section{Extended Kalman filtering for online learning}
\label{sec: KF online learning}
\label{sec:EKF}

Here, we explain in more detail the background of the EKF for online learning presented in Section \ref{sec:EKF-summary}. To make this section self-contained, we repeat parts of the material introduced in that section. 
Let us assume that $\vy_t = \plast_t^\intercal\,\phi(\vx_t, \phidden_t) + \ve_t$,
where
$\ve_t$ is a zero-mean random variable with observation covariance $\vR_t \in \matPos{\dimobs}$.
\eat{
for a zero-mean error with variance $\vR_t \in \matPos{\dimobs}$, and 
 group the model parameters of the neural network as follows
$$
\vtheta_t = 
\begin{bmatrix}
    \plast_t\\
    \phidden_t
\end{bmatrix} \in \reals^{\dimlast + \dimhidden}.
$$
}
Given a starting vector $\vtheta_0$,
we assume
that model parameters $\vtheta_t = (\plast_t, \phidden_t)\in\reals^\dimall$ and observations $\vy_t \in \reals^\dimobs$ evolve according to the 
state-space model
\begin{equation}
\label{eq:ssm-full}
\begin{aligned}
    \vtheta_t &= \vtheta_{t-1} + \vu_t,\\
    \vy_t &= \nn(\vtheta_t,\,\vx_{t}) 
    + \ve_t,
\end{aligned}
\end{equation}
where
$\vu_t \in \reals^\dimall$ is a zero-mean random vector with known dynamics covariance $\vQ_t \in \matSemiPos{\dimall}$,
and
$\ve_t \in \reals^\dimobs$ is a zero-mean random vector with known observation covariance $\vR_t \in \matPos{\dimobs}$.
As before, we assume $\cov(\vtheta_0, \ve_t) = \vzero$ for all $t\in \mathbb{N}$ and 
$\cov(\vu_t, \ve_k) = \vzero$ for all $t,k \in \mathbb{N}$.

We now consider a 
first-order approximation of $\nn$ around
$\vtheta_{t-1|t-1}$, that is,
\begin{equation}
\label{eq:lin-ssm-full}
\vy_t \approx 
\nn(\vtheta_{t-1|t-1},\,\vx_{t}) 
+ \gradall_t\,(\vtheta_t -\vtheta_{t-1|t-1})
+ \ve_t\,,
\end{equation}
where
$\vtheta_{t-1|t-1} = \begin{bmatrix} \plast_{t-1|t-1}^\intercal & \phidden_{t-1|t-1}^\intercal \end{bmatrix}^\intercal$ is given,
and
$\gradall_t = \begin{bmatrix} \gradlast_t & \gradhidden_t \end{bmatrix}$ with
$\gradlast_t = \nabla_{\plast} \nn(\plast_{t-1|t-1},\,\phidden_{t-1|t-1},\,\vx_{t})$ and 
$\gradhidden_t = \nabla_{\phidden} \nn(\plast_{t-1|t-1},\,\phidden_{t-1|t-1},\,\vx_{t})$.

Following a frequentist approach  to the Kalman filter \cite{humpherys2012freqkf} and
given the starting latent random vector $\vtheta_{0|0} := \vtheta_0$,
we seek to obtain sequential updates for the \textit{best} ($L^2$) linear
estimate of the expected value of $\vtheta_t$ given data $\vy_{1:t}$.
We formalize this in the following proposition.

\begin{proposition}
Let $k,t\in\mathbb{N}$ and let $\vy_t$, $\vtheta_t$ follow the SSM \eqref{eq:lin-ssm-full}.
The solution to the optimization problem
\begin{equation}\label{eq:problem-kf}
    \argmax_{\vA\in \mat{(\dimlast+\dimhidden)\times(k\,\dimobs)}} \mathbb{E}\left[ ||\vtheta_t - \vA\,\vy_{1:k}||_2^2\right],
\end{equation}
is the matrix $\vA_{t|k}^\star$ given by
\begin{equation}
    \vA^\star_{t|k}  = \cov(\vtheta_t, \vy_{1:k})\,\var(\vy_{1:k})^{-1}.
\end{equation}
The best linear unbiased predictor (BLUP) for model parameters $\vtheta_t$, given observations $\vy_{1:k}$ is
\begin{equation}\label{eq:vtheta t|k}
\begin{aligned}
    \vtheta_{t|k} &= \vA_{t|k}^\star\,\vy_{1:k},
\end{aligned}
\end{equation}
and the error variance covariance (EVC) matrix is defined as
\begin{equation}
\begin{aligned}
    \vSigma_{t|k} &:= \var(\vtheta_t) - \vA_{t|k}^\star\,\var(\vy_{1:k})\,\vA_{t|k}^{\star\intercal}.
\end{aligned}
\end{equation}
\end{proposition}

\begin{proposition}\label{prop:kalman-filter-steps}
Under the SSM \eqref{eq:lin-ssm-full}, the BLUP and the EVC
can be written in the form of Kalman filtering predict and update equations.
Here, the predict equations are
\begin{equation}
\begin{aligned}
    \vtheta_{t|t-1} &= \vtheta_{t-1|t-1},\\
    \vSigma_{t|t-1} &= \vSigma_{t-1|t-1} + \vQ_t,
\end{aligned}
\end{equation}
and the update equations are
\begin{equation}
\begin{aligned}
    \vtheta_{t|t} &= \vtheta_{t|t-1} + \vK_t\,\vepsilon_t,\\
    \vSigma_{t|t} &= (\vI - \vK_t\,\gradall_t)\,\vSigma_{t|t-1}(\vI - \vK_t\,\gradall_t)^\intercal + \vK_t\,\vR_t\vK_t^\intercal,
      \label{eqn:EKFupdate1}
\end{aligned}
\end{equation}
with
\begin{equation}
\begin{aligned}
    \vepsilon_t &= \vy_t - \nn(\vtheta_{t-1|t-1}, \vx_t),\\
    \vK_t &= \vSigma_{t|t-1}\,\gradall_t^\intercal\,\vS_t^{-1},\\
    \vS_t &= \gradall_t\vSigma_{t|t-1}\,\gradall_t^\intercal + \vR_t.
    \label{eqn:EKFupdate2}
\end{aligned}
\end{equation}
\end{proposition}
\begin{proof}
The proof follows from the linear form of \eqref{eq:lin-ssm-full}, Lemmas 2.1 to 2.3, and Theorem 4.2 in \cite{eubank2005kalman}.

\end{proof}

\begin{remark}
    The update equations in \eqref{eqn:EKFupdate1} are recursive and
    characterize estimate of the unknown latent state and the estimated
    error estimation through the first two moments.
    Furthermore, the update for the covariance in \eqref{eqn:EKFupdate1} is  in Joseph form \cite{zanetti2013joseph}
    and is known to be numerically stable.
    In our method, we leverage these two facts to target a low-rank EVC through a second-step optimization that
    has numerically-stable updates.
\end{remark}

\begin{corollary}[Kalman filter as a Bayesian posterior]
    Assume a Gaussian density for the initial parameters
    $p(\vtheta_0) = {\cal N}(\vtheta_0 \cond \vtheta_{0|0}, \vSigma_{0|0})$ and for the noise term
    $p(\ve_t) = {\cal N}(\ve_t \cond \vzero,\,\vR_t)$.
    The  BLUP $\vtheta_{t|t}$ and the EVC $\vSigma_{t|t}$
    are the parameters of the Gaussian that characterize the posterior of model parameters $\vtheta_t$ given $\vy_{1:t}$,
    more precisely,
    \begin{equation}\label{eq:kf-bayes}
        p(\vtheta_t \cond \vy_{1:t}) = {\cal N}(\vtheta_t \cond \vtheta_{t|t}, \vSigma_{t|t}).
    \end{equation}
    As a consequence, the posterior predictive for the next observation in the sequence is
    \begin{equation}
    \begin{aligned}
        p(\vy_{t} \cond \vy_{1:t-1})
        &= \int {\cal N}\left(\vy_t \cond \nn(\vtheta_{t-1|t-1},\,\vx_{t}) + \gradall_t(\,\vtheta_t -\vtheta_{t|t-1}), \vR_t\right) p(\vtheta_t \cond \data_{1:t-1}) \d\vtheta_t\\
        &= {\cal N}(\vy_{t} \cond \nn(\vtheta_{t-1|t-1}, \vx_t),\;\;\gradall_t\,\vSigma_{t|t-1}\gradall_t^\intercal + \ \vR_t).
    \end{aligned}
    \end{equation}
\end{corollary}
\section{Linear algebra results}
\subsection{Sum of Cholesky matrices}
Here, we generalize the result found in \cite{tracy2022qrkf} to any set of $K > 2$ matrices, which we use to derive the method \methodlrkf, \newmethodlin, and \newmethodlinlow.

\begin{proposition}[QR of sum of Cholesky matrices]
\label{sec:sum-cholesky}
Let $\vA_i$ for $i\in\{1,2,\dots, I\}$ be a collection of $(D\times D)$ positive definite matrices
and let $\vA_i^{1/2}$ be the upper-triangular Cholesky decomposition of $\vA_i$.
Let $\vM$ be the stacked $(D\,I\times D)$ matrix given by
\begin{equation}
     \vM =
    \begin{bmatrix}
        \vA^{1/2}_1 \\ \vA^{1/2}_2 \\ \vdots\\ \vA^{1/2}_I
    \end{bmatrix}.
\end{equation}
It follows that 
\begin{equation}
   \vR^\intercal\,\vR = \sum_{i=1}^I \vA_i\,. 
\end{equation}
where $\vR$ is R component in the QR decomposition of $\vM$.
\end{proposition}

\begin{proof}
We see that 
\begin{equation}
   \vM^\intercal\,\vM = \sum_{i=1}^I \vA_i\,. 
\end{equation}
Consider the QR decomposition of $\vM$
\begin{equation}
    \vM = \vQ\vR,
\end{equation}
with $\vQ$ an orthogonal matrix and $\vR$ and upper-triangular matrix.

Then,
\begin{equation}
    \sum_{i=1}^I\vA_i = 
    \vM^\intercal\,\vM
    = \vR^\intercal\,\vQ^\intercal\,\vQ\,\vR
    = \vR^\intercal\,\vR.
\end{equation}
As a consequence, we note that the \textit{square root} of the matrix $\vA_1 + \dots + \vA_I$
is the upper-triangular matrix in the QR decomposition of $\vM$.
\end{proof}

\paragraph{Remark}
In what follows, we denote the result of obtaining the $\vR$ matrix of the QR decomposition of the
row-stacked matrices $\vA^{1/2}_i$ for $i\in\{1,\dots,I\}$ by
\begin{equation}
    {\cal Q}_R(\vA^{1/2}_1, \dots, \vA^{1/2}_I).
\end{equation}

\subsection{Singular value decomposition given sum of low-rank-matrices}
\label{sec:sum-SVD}
\begin{proposition}[SVD of sum of low-rank matrices]\label{prop:sum-SVD}
Let $\vA_i$ for $i\in\{1,2,\dots, N\}$ be $\matPos{D}$ positive semi-definite matrices such that
$\vA_i = \vW_i^\intercal\,\vW_i$, where $\vW_i\in\reals^{d\times D}$ and $d \ll D$.
The best rank-$d$ approximation of the sum is
\begin{equation}
    \sum_{i=N}^I \vA_i \approx \vJ^\intercal\,\vJ\,,
\end{equation}
where 
$\vJ = \vS_{:d}\,\vV^\intercal \in \reals^{d\times D}$,
$\vS_{:d}$ are the top $d$ singular values and $\vV$ contains the right singular vectors
of the stacked $(d\,N\times D)$ matrix given by
\begin{equation}
    \vN = 
    \begin{bmatrix}
        \vW_1 \\ \vW_2 \\ \vdots \\ \vW_N
    \end{bmatrix}.
\end{equation}
\end{proposition}

\begin{proof}
Let $\vA_i$ for $i\in\{1,2,\dots,N\}$ be $\matPos{D}$ matrices with
$\vA_i = \vW_i^\intercal \vW_i$,
$\vW_i\in\reals^{d\times D}$, 
and 
$d \ll D$.
Form the  matrix $\vN$ given by
\begin{equation}
\vN = 
\begin{bmatrix}
  \vW_1 \\[6pt]
  \vW_2 \\[3pt]
  \vdots \\[3pt]
  \vW_N
\end{bmatrix}
= \vU\,\vS\,\vV^\intercal,
\end{equation}
where $\vU\,\vS\,\vV^\intercal$ is 
its full singular value decomposition, with
$\vU\in\reals^{dN\times dN}$, 
$\vS\in\reals^{dN\times D}$, and 
$\vV\in\reals^{D\times D}$.  Then
\begin{equation}
\label{eq:SVD-sum-of-matrices}
\begin{aligned}
  \sum_{i=1}^N \vA_i
  &= \sum_{i=1}^N \vW_i^\intercal \vW_i
  = \vN^\intercal \vN\\
  &= (\vU\,\vS\,\vV^\intercal)^\intercal \,(\vU\,\vS\,\vV^\intercal)\\
  &= \vV\,\vS^\intercal\,\vU^\intercal\,\vU\,\vS\,\vV^\intercal\\
  &= \vV\,\vS^2\,\vV^\intercal.
\end{aligned}
\end{equation}
Hence $\sum_i\vA_i$ has eigenvectors $\vV$ and eigenvalues given by the diagonal entries of $\vS^2$.
By the symmetric‐matrix form of the Eckart–Young theorem \cite{eckart1936lowrank,dax2014low},
its best rank-$d$ approximation in Frobenius norm is
\begin{equation}
\vV\,{\rm diag}(\sigma_1^2,\dots,\sigma_d^2,0,\dots)\,\vV^\intercal.
\end{equation}
Writing $\vS_{:d} = {\rm diag}(\sigma_1,\dots,\sigma_d)\in\reals^{d\times D}$, one checks
\begin{equation}
\vV\,{\rm diag}(\sigma_1^2,\dots,\sigma_d^2,0,\dots)\,\vV^\intercal
= \bigl(\vS_{:d}\,\vV^\intercal\bigr)^{T}\,\bigl(\vS_{:d}\,\vV^\intercal\bigr)
= \vJ^\intercal\,\vJ,
\end{equation}
where
\begin{equation}
\begin{aligned}
    \vJ &= \vS_{:d}\,\vV^\intercal \in \reals^{d\times D},\\
    \vS_{:d} &:= {\rm diag}\left(\sigma_1,\dots,\sigma_d\right),
\end{aligned}
\end{equation}
and $\{\sigma_k\}$ are the singular vectors of $\sum_{n=1}^N \vA_n$.
\end{proof}

\paragraph{Takeaway.}
Proposition \ref{prop:sum-SVD} shows that to form the best rank-$d$ approximation of
$\sum_{i=1}^N \vA_i$ one does not need to 
assemble or carry out the SVD of the full $D\times D$ sum.
Instead, one stacks the low-rank factors into
\begin{equation}
\vN =
\begin{bmatrix}
\vW_1 \\
\vdots \\
\vW_N
\end{bmatrix}
\in \mathbb{R}^{(dN)\times D},
\end{equation}
and compute a reduced SVD (or symmetric eigendecomposition) of the small $(dN)\times(dN)$ Gram matrix $\vN \vN^\top$.
This costs
\begin{equation}
\mathcal{O}((dN)^2 D + (dN)^3)
\approx
\mathcal{O}(d^2 N^2 D)
\text{ when } dN \ll D,
\end{equation}
versus the $\mathcal{O}(D^3)$ required for a full SVD on the $D\times D$ sum.
Moreover,  memory and computation scale with $d$ and $dN$ rather than with the large dimension $D$ because all operations involve only the $(dN)\times D$ matrix $\vN$ and the $d\times D$ factor $\vJ$.

In what follows, we denote the result of obtaining the best rank-$d$ approximation matrix of the SVD decomposition of the
row-stacked matrices $\vW^{1/2}_i$ for $i\in\{1,\dots,N\}$ by
\begin{equation}
    {\cal P}_d(\vW_1, \ldots, \vW_N).
\end{equation}

\begin{corollary}\label{cor:sum-svd-plus-id}
    If the sum of matrices is of the form
    \begin{equation}
          \sum_{n=1}^N \vA_n + a\,\vI_D
    \end{equation}
    with $a > 0$,
    then the best rank-$d$ approximation is of the form
    \begin{equation}
        \sum_{n=1}^N \vA_n + a\,\vI_D \approx \vJ_{+q}^\intercal\,\vJ_{+q}
    \end{equation}
    where
    \begin{equation}
    \begin{aligned}
        \vJ_{+q}
        &= \vS_{:d,+q}\,\vV^\intercal \in \reals^{d\times D},\\
        \vS_{:d, +q} &:= {\rm diag}\left(\sqrt{\sigma_1^2 + q},\dots,\sqrt{\sigma_d^2 + q}\right),
    \end{aligned}
    \end{equation}
    and $\{\sigma_k\}_{k=1}^d$ are the top-$d$ singular values of $\sum_{n=1}^N \vA_n$.
\end{corollary}

\begin{proof}
    From \eqref{eq:SVD-sum-of-matrices}
    in Proposition \ref{prop:sum-SVD} we obtain
\begin{equation}
  \sum_{i=1}^N \vA_i + a\,\vI = \vV\,\vS^2\,\vV^\intercal + a\vI = \vV\,(\vS^2 + a\vI_D)\,\vV^\intercal
  = \vV\,\vS_{+a}^2\,\vV^\intercal,
\end{equation}
where
\begin{equation}
    \vS_{+a}^2 = {\rm diag}(\sigma_1^2 + a, \ldots, \sigma_D^2 + a).
\end{equation}
Let
\begin{equation}
    \vS_{+a} = {\rm diag}\left(\sqrt{\sigma_1^2 + a}, \ldots, \sqrt{\sigma_D^2 + a}\right),
\end{equation}
then
\begin{equation}
\begin{aligned}
    \sum_{n=1}^N \vA_n
    &= \vV\,\vS_{+a}^2\,\vV^\intercal\\
    &= (\vV\,\vS_{+a}\vV^\intercal)^\intercal\,(\vV\,\vS_{+a}\vV^\intercal)\\
    &= \vJ_{+a}^\intercal\,\vJ_{+a}.
\end{aligned}
\end{equation}
\end{proof}

\paragraph{Takeaway.}
If the sum of matrices can be written as the sum of multiplied low-rank factors plus an additional identity matrix times a constant,
then, computation of the best rank $d$-matrix in low-rank factor can be obtained by performing SVD over the low-rank factors
and modifying the singular values to include the term $a$.

In what follows, we denote the result of obtaining the best rank-$d$ approximation matrix of the SVD decomposition of the
row-stacked matrices $\vW^{1/2}_i$ for $i\in\{1,\dots,N\}$ plus an identity times a real-valued number $a>0$ as
\begin{equation}
    {\cal P}_{d,+a}(\vW_1, \ldots, \vW_N).
\end{equation}

\section{\newmethodlin --- further details}
\label{sec:further-results}

\subsection{Derivation of \newmethodlin}
\label{sec:derivation-hilofi}

Consider the linearized model
\begin{equation}\label{eq:hilofi-measurement-model}
    \vy_t =  \nn(\plast_{t-1|t-1},\,\phidden_{t-1|t-1},\,\vx_{t})
    + \gradlast_t(\plast_t -\plast_{t-1|t-1})
    + \gradhidden_t(\phidden_t -\phidden_{t-1|t-1})
    + \ve_t,
\end{equation}
where $\plast_{t-1|t-1}$ and $\phidden_{t-1|t-1}$ are given.

We start the algorithm
by initialising the beliefs about the last layer parameters $\plast$
and the hidden layer parameters $\phidden$. We set 
\begin{equation}
\begin{aligned}
    \mathbb{E}[\plast_0] &= \meanlast{0}, & \mathbb{E}[\phidden_0] &= \meanhidden{0}, &
    \var(\plast_0) &= \hat{\vSigma}_{\plast, 0}, & \var(\phidden_0) &=\lrhidden_{0}^\intercal\,\lrhidden_{0},
\end{aligned}
\end{equation}
for known
$\lrhidden_{0} \in \mat{\dimhiddensub\times\dimhidden}$,
$\meanhidden{0} \in \reals^\dimhidden$,
$\hat{\vSigma}_{\plast, 0} \in \matPos{\dimlast}$, and
$\meanlast{0} \in \reals^\dimlast$.
Next, we assume that for $t\geq 1$, the latent last layer parameters $\plast_{1:t}$ and the hidden layer parameters $\phidden_{1:t}$
follow the dynamics 
\begin{equation}\label{eq:hilofi-ssm}
    \plast_t = \plast_{t-1} + \vu_{\plast,t}, \qquad \phidden_t = \phidden_{t-1} + \vu_{\phidden,t},\\
\end{equation}
where $\vu_{\plast,1:t}$ and $\vu_{\phidden,1:t}$ are zero-mean independent noise variables\footnote{
i.e., $\cov(\vu_{\plast, i}, \vu_{\phidden, j}) = \vzero$ for all $i,j \in \{1, \ldots, T\}$.
}
with dynamics covariance matrices
$\var(\vu_{\plast,t})=\vQ_{\plast,t} \in \matSemiPos{\dimlast}$ and
$\var(\vu_{\phidden,t})=\vQ_{\phidden,t} \in \matSemiPos{\dimhidden}$.
For simplicity
and to exploit efficient linear algebra techniques for belief updates,
we assume
$\vQ_{\plast,t} = q_{\plast,t} \vI_\dimlast$ and
$\vQ_{\phidden,t} = q_{\phidden,t} \vI_\dimhidden$.
A simple choice is to set $\vQ_{\phidden,t} = 0 \vI$ to avoid forgetting past data; see e.g.,
\cite{duran2025bone}.

\begin{proposition}[Predict step]\label{prop:new-predict-step}
With the SSM assumption \eqref{eq:hilofi-ssm},
the approximate covariance 
$\var(\plast_{t-1} - \plast_{t-1|t-1}) = \hat{\vSigma}_{\plast, t-1}$
and
the approximate covariance $\var(\phidden_{t-1} - \phidden_{t-1|t-1}) = \lrhidden_{t-1}^\intercal\,\lrhidden_{t-1}$,
the predict step becomes
\begin{equation}
\begin{aligned}
    \phidden_{t|t-1} &= \phidden_{t-1|t-1}, &
    \plast_{t|t-1} &= \plast_{t-1|t-1},\\
    \vSigma_{\plast, t|t-1} &= \hat{\vSigma}_{\plast,t-1} + q_{\plast,t}\,\vI_\dimlast,&
    \vSigma_{\phidden, t|t-1} &= 
    \lrhidden_{t-1}^\intercal\,\lrhidden_{t-1} + q_{\phidden,t}\,\vI_\dimhidden.
\end{aligned}
\end{equation}
\end{proposition}

The proof of Proposition \ref{prop:new-predict-step} is in Appendix \ref{proof:new-predict-step}.

\begin{proposition}[Variance of the innovation]\label{prop:new-innovation-variance}
The upper Cholesky decomposition
of innovation variance
takes the form
\begin{equation}
    \vS_t^{1/2}  = {\cal Q}_R\left(
        \hat{\vSigma}_{\plast, t-1}^{1/2}\,\gradlast_t^\intercal,\,
        \sqrt{q_{\plast,t}}\,\gradlast_t^\intercal,\,
        \lrhidden_{t-1}\,\gradhidden_t^\intercal,\,
        \sqrt{q_{\phidden, t}}\,\gradhidden_t^\intercal,\,
        \vR_t^{1/2}
    \right),
\end{equation}
where
$\gradlast_t$ and $\gradhidden_t$
are defined in Section \ref{sec:EKF-summary}.
\end{proposition}
The proof of Proposition \ref{prop:new-innovation-variance} is in Appendix \ref{proof:new-innovation-variance}.

Next, our update step for the BLUP resembles an EKF update, while the EVC update approximates a block-diagonal covariance:
the last layer block is full-rank, and the hidden layer blocks are low-rank.
To maintain numerical stability and low-memory updates, we track the Cholesky factor for the last layer and a low-rank factor for the hidden layers.
The Cholesky factor is taken from a surrogate matrix that avoids computing a full $\mat{\dimlast\times\dimlast}$ matrix by neglecting the  artificial
dynamics covariance $q_{\plast,t}$.
The low-rank component for hidden layers is approximated via a two-step process: (i) a fast \textit{surrogate} predicted covariance, and (ii) a rank $\dimhidden$ projection error.
Like the Cholesky factor, the surrogate covariance in (i) reduces the computational cost of dynamics noise,
which does not reflect the system’s true dynamics.
However, we retain some information from the dynamics noise to update all hidden layer parameters.
Step (ii) enables a fast update rule, crucial for overparameterized neural networks \cite{larsen2022degreesfreedomneedtrain}.
We detail this step in the following proposition.

\begin{proposition}[Kalman gain, update BLUP, and update EVC]\label{prop:new-gain-new-update}
Define the gain matrices
\begin{equation}
\vK_{\phidden, t}^\intercal = \vV_{\phidden,t}\lrhidden_{t-1}\lrhidden_{t-1}^\intercal + q_{\phidden,t}\vV_{\phidden, t},\quad
\vK_{\plast, t}^\intercal = \vV_{\plast, t}\covlast{t|t-1} + q_{\plast,t}\vV_{\plast,t},
\end{equation}
where
$ \vV_{\phidden,t} = \varinnov_t^{-1/2}\,\varinnov_t^{-\intercal/2}\,\gradhidden_t,\quad$
and 
$\vV_{\plast,t} = \varinnov_t^{-1/2}\,\varinnov_t^{-\intercal/2}\,\gradlast_t$.
The updated BLUP for the last layer parameters and hidden layer parameters are
\begin{equation}
\phidden_{t|t} = \phidden_{t-1|t-1} + \vK_{\phidden, t}\,\vepsilon_t,\quad
\plast_{t|t} = \plast_{t-1|t-1} + \vK_{\plast, t}\,\vepsilon_t.
\end{equation}
The approximate posterior covariance $\hat{\vSigma}_{t|t}$ is the best block-diagonal approximation (in Frobenius norm) of the surrogate matrix $\tilde{\vSigma}_{t|t} $ given by
\begin{align}
&\scriptsize
\nonumber
\left[
\begin{smallmatrix}
(\vI_{\dimlast} - \vK_{\plast, t}\gradlast_t)\hat{\vSigma}_{t-1}(\vI_{\dimlast} - \vK_{\plast, t}\gradlast_t)^\intercal - \vK_{\plast,t}\vR_t\vK_{\plast,t}^\intercal &
\vK_{\plast,t}\vR_t\vK_{\phidden,t}^\intercal
\\
\vK_{\phidden,t}\vR_t\vK_{\plast,t}^\intercal &
(\vI_{\dimhidden} - \vK_{\phidden,t}\gradhidden_t)\lrhidden_{t-1}^\intercal\lrhidden_{t-1}(\vI_{\dimhidden} - \vK_{\phidden,t}\gradhidden_t)^\intercal - \vK_{\phidden,t}\vR_t\vK_{\phidden,t}^\intercal + q_{\phidden,t}\vI_{\dimhidden}
\end{smallmatrix}
\right],
\end{align}
that has full rank for the last layer and rank $\dimhiddensub$ for the hidden layer.
This results in a Cholesky factor for the last-layer covariance given by
\begin{equation}
    \hat{\vSigma}_{\plast, t}^{1/2}
    = {\cal Q}_R\left(
        \hat{\vSigma}_{\plast, t-1}^{1/2}(\vI - \vK_{\plast,t}\gradlast_t)^\intercal,\,
        \vR_t^{1/2}\vK_{\plast,t}^\intercal
    \right),
\end{equation}
and a low-rank factor for the hidden layers given by
\begin{equation}
\lrhidden_{t} = {\cal P}_{\dimhiddensub,+q_{\phidden,t}}\left(\lrhidden_{t-1}(\vI - \vK_{\phidden,t}\gradhidden_t)^\intercal,\; \vR_t^{1/2}\vK_{\phidden,t}^\intercal\right),
\end{equation}
\end{proposition}
The proof is in Appendix~\ref{proof:new-gain-new-update}.

\subsection{Computational complexity}
\label{sec:computational-complexity}

We now compare the computational complexity of the algorithms we employ.
Recall that $\dimlast$, $\dimhidden$, and $\dimobs$  are the number of parameters in the last layer, hidden layers, and the size of the output layer, respectively. As before, $\dimall=\dimlast + \dimhidden$. The dimensions $\dimlastsub, \dimhiddensub$, with  
$\dimlastsub \ll \dimlast$ and $\dimhiddensub \ll \dimhidden$, are the low-dimensional subspaces for the last and hidden layers, respectively,
and $\dimsub=\dimlastsub + \dimhiddensub$.

The computational complexity of our algorithm when using full-rank in the last layer is
$O(\dimlast\,(\dimlast + \dimobs)^2 + \dimhidden\,(\dimhiddensub + \dimobs)^2)$.
For low-dimensional outputs, this is
$O(\dimlast^3 + \dimhidden \dimhiddensub^2)$.
For high-dimensional outputs, we may choose to use 
 a low-rank approximation for the last layer (which we call \newmethodlinlow) 
this reduces the cost to
$O(\dimlast\,(\dimlastsub + \dimobs)^2 + \dimhidden\,(\dimhiddensub + \dimobs)^2)$,
which is linear in $\dimall$
The primary computational bottlenecks in our method are the calculations for the Kalman gain of
 the hidden and last layers, with costs
$O(2\,\dimhidden\,\dimobs^2)$ and
$O(2\,\dimlast\,\dimobs^2)$ respectively.
This efficiency arises because
$\vS_t^{1/2}$ is upper triangular,
allowing the system
$\vS_t^{\intercal/2}\,\vS_t^{1/2}\,\vV = \gradall$ for $\vV$
to be solved in
$O(2\,\dimhidden\,\dimobs^2)$ as opposed to 
$O(\dimhidden\,\dimobs^3)$.
Additionally, approximating the covariance matrix via truncated SVD incurs a cost of
$O(\dimhidden\,(\dimhiddensub + \dimobs)^2 + (\dimhiddensub + \dimobs)^3 + \dimhiddensub\,(\dimhiddensub + \dimobs)\,\dimhidden)$,
as detailed in Figure 8.6.1 in \cite{golub2013matrix}.
Finally, the corresponding cost for the last layer is
$O(\dimlast\,(\dimlast + \dimobs)^2)$.

Among related methods,
the closest to ours in terms of computational costs are variational Bayes approaches
such as the \texttt{Slang} method \cite{mishkin2018slang}, the \texttt{L-RVGA} method \cite{lambert2023lrvga}, and the \methodlofi method \cite{chang23lofi}.
In particular, \methodlofi
uses the linearized Gaussian updates (similar to the ones in this paper),
but
approximates the precision matrix using
a diagonal-plus-low rank (DLR) form.
The appeal for DLR precision matrices is twofold.
First, they enable the application of the Woodbury identity, leading to a predict step of cost
$O(\dimall\,\dimsub + \dimsub^3)$
and an update step of
$O(\dimall\,(\dimsub + \dimobs)^2) = O(\dimlast\,(\dimsub + \dimobs)^2 + \dimhidden\,(\dimsub + \dimobs)^2)$.
Second,
incorporating positive diagonal terms ensures the matrix is positive definite,
so that a valid posterior Gaussian density is defined.

Although \newmethodlin and \methodlofi have the same asymptotic complexity,
\methodlofi incurs additional practical overhead due to three key operations absent in \newmethodlin:
(1) the inversion of a $(\dimsub + \dimobs)$ rectangular matrix,
(2) the inversion of a $\dimsub$ rectangular matrix, and
(3) the Cholesky decomposition of a $\dimsub$ rectangular matrix.
These operations increase the actual per-step computational time, which in our example,
scales at about one second per additional rank.
This behaviour is shown in the empirical comparison across \newmethodlin, \methodlrkf, and \methodlofi
in the online classification setting
(Figure \ref{fig:mnist-online-classification-rank-comparison}, Appendix \ref{sec:further-results-mnist}).

\section{Proofs}
\label{sec:ll-lrkf-proofs}

\subsection{Proof of Proposition \ref{prop:new-predict-step}}
\begin{proof}\label{proof:new-predict-step}
Following Proposition \ref{prop:kalman-filter-steps},
the predicted mean takes the form
\begin{equation}
    \vtheta_{t|t-1} = \vtheta_{t-1|t-1}
    = \begin{bmatrix}
    \plast_{t-1|t-1} \\
    \phidden_{t-1|t-1}
    \end{bmatrix}.
\end{equation}
Next, the predicted posterior covariance takes the form
\begin{equation}
\begin{aligned}
    \vSigma_{t|t-1} &= \vSigma_{t-1|t-1} + \vQ_t,
\end{aligned}
\end{equation}
where
\begin{equation}
\begin{aligned}
    &\vSigma_{t-1|t-1}\\
    &\quad = \var(\vtheta_{t-1} - \vtheta_{t-1|t-1})\\
    &\quad =
    \begin{bmatrix}
    \var(\plast_{t-1} - \plast_{t-1|t-1})
    & \cov(\plast_{t-1} - \plast_{t-1|t-1}, \phidden_{t-1} - \phidden_{t-1|t-1})\\
    \cov(\phidden_{t-1} - \phidden_{t-1|t-1}, \plast_{t-1} - \plast_{t-1|t-1})
    & \var(\phidden_{t-1} - \phidden_{t-1|t-1}).
    \end{bmatrix} \\
    &\quad = \begin{bmatrix}
        \covlast{t-1} \\
        & \cov(\plast_{t-1} - \plast_{t-1|t-1}, \phidden_{t-1} - \phidden_{t-1|t-1})\\
        \cov(\phidden_{t-1} - \phidden_{t-1|t-1}, \plast_{t-1} - \plast_{t-1|t-1})
        & \vW_{t-1}^\intercal\,\vW_{t-1}
    \end{bmatrix}.
\end{aligned}
\end{equation}
Next,

\begin{equation}
\begin{aligned}
    &\cov(\plast_{t-1} - \plast_{t-1|t-1}, \phidden_{t-1} - \phidden_{t-1|t-1})\\
    &\quad= \cov(\plast_{t-1}, \phidden_{t-1})\\
    &\quad= \cov\left(\plast_0 + \sum_{\tau=1}^{t-1} \vu_{\plast, \tau},\,\phidden_0 + \sum_{\tau=1}^{t-1} \vu_{\phidden, \tau}\right)\\
    &\quad= \cov(\plast_0, \phidden_0) + \sum_{\tau=1}^{t-1} \cov(\plast_0, \vu_{\phidden,\tau}) + \sum_{\tau=1}^{t-1}\cov(\vu_{\plast, \tau}, \phidden_0) + \sum_{\tau=1}^{t-1}\sum_{\tau'=1}^{-1} \cov(\vu_{\plast, \tau}, \vu_{\phidden, \tau'})\\
    &\quad= \vzero,
\end{aligned}
\end{equation}
which follows from the assumptions about the SSM. Then 
\begin{equation}
    \vSigma_{t-1|t-1} = {\rm diag}(\covlast{t-1|t-1},\,\lrhidden_{t-1}^\intercal\,\lrhidden_{t-1}).
\end{equation}
and
\begin{equation}
\begin{aligned}
    \vSigma_{t|t-1}
    &= \vSigma_{t-1|t-1} + \vQ_t\\
    &= {\rm diag}(
    \covlast{t-1|t-1} + q_\plast\,\vI_\dimlast,
    \,\lrhidden_{t-1}^\intercal\,\lrhidden_{t-1} + q_\phidden\,\vI_\dimhidden
    ).
\end{aligned}
\end{equation}
Given that 
$\covlast{t-1|t-1}$ is approximated at $t-1$ through $\hat{\vSigma}_{\plast, t-1}$ and
$\covhidden{t-1|t-1}$ is approximated at $t-1$ through $\vC_t^\intercal\,\vC_t$, we obtain
\begin{equation}
\begin{aligned}
    \vSigma_{t|t-1,\plast} &\approx \hat{\vSigma}_{\plast, t-1} + q_\plast\,\vI_\dimlast,\\
    \vSigma_{t|t-1,\phidden} &\approx \lrhidden_{t-1}^\intercal\,\lrhidden_{t-1} + q_\phidden\,\vI_\dimlast.
\end{aligned}
\end{equation}

\end{proof}

\subsection{Proof of Proposition \ref{prop:new-innovation-variance}}
\label{proof:new-innovation-variance}
\begin{proof}
By definition of $\vy_t$ and Proposition \ref{prop:kalman-filter-steps},
the innovation takes the form
\begin{equation}
\begin{aligned}
    \vepsilon_t
    &= \vy_t - \nn(\plast_{t-1|t-1}, \phidden_{t-1|t-1}, \vx_t)\\
    &= \gradall_t(\,\vtheta_t -\vtheta_{t|t-1}) + \ve_t\\
    &= \gradlast_t\,(\plast_t - \plast_{t|t-1}) + \gradhidden_t\,(\phidden_t - \phidden_{t|t-1}) + \ve_t.
\end{aligned}
\end{equation}
Next, the variance of the innovation is
\begin{equation}
\begin{aligned}
    \vS_t &=  \var(\vepsilon_t)\\
    &= \var\left(\gradlast_t\,(\plast_t - \plast_{t|t-1}) + \gradhidden_t\,(\phidden_t - \phidden_{t|t-1}) + \ve_t\right)\\
    &=
     \gradlast_t\,\var(\plast_t - \plast_{t|t-1})\gradlast_t^\intercal 
    + \gradhidden_t\,\var(\phidden_t - \phidden_{t|t-1})\gradhidden_t^\intercal 
    + \var(\ve_t)\\
    &=
     \gradlast_t\,\vSigma_{\plast, t|t-1}\,\gradlast_t^\intercal 
    + \gradhidden_t\,\vSigma_{\phidden, t|t-1}\,\gradhidden_t^\intercal 
    + \vR_{t}\\
    &= 
     \gradlast_t\,\hat{\vSigma}_{\plast,t-1}\,\gradlast_t^\intercal 
    + q_{\plast,t}\,\gradlast_t\,\gradlast_t^\intercal 
    + \gradhidden_t\,\lrhidden_t^\intercal\,\lrhidden_t\gradhidden_t^\intercal
    + q_{\phidden,t}\gradhidden_t\,\gradhidden_t^\intercal
    + \vR_t\\
    &=\begin{bmatrix}
        \gradlast\,\hat{\vSigma}_{\plast,t-1}^{\intercal/2} &
        \sqrt{q_{\plast,t}}\,\gradlast_t &
        \gradhidden\,\lrhidden_{t-1}^\intercal &
        \sqrt{q_{\phidden,t}}\,\gradhidden_t &
        \vR^{\intercal/2}\,
    \end{bmatrix}
    \begin{bmatrix}
        \gradlast\,\hat{\vSigma}_{\plast,t-1}^{1/2} \\
        \sqrt{q_{\plast,t}}\,\gradlast_t^\intercal \\
        \lrhidden_{t-1}\,\gradhidden^\intercal \\ 
        \sqrt{q_{\plast,t}}\,\gradlast_t^\intercal \\
        \vR_t^{1/2}
    \end{bmatrix}\\
    &= \vS_{t}^{\intercal/2}\,\vS_t^{1/2},
\end{aligned}
\end{equation}
with
\begin{equation}
    \vS_t^{1/2}  = {\cal Q}_R\left(
        \covlast{t|t}^{1/2}\,\gradlast_t^\intercal,\,
        \sqrt{q_{\plast,t}}\,\gradlast_t^\intercal,\,
        \lrhidden_{t-1}\,\gradhidden_t^\intercal,\,
        \sqrt{q_{\phidden, t}}\,\gradhidden_t^\intercal,\,
        \vR_t^{1/2}
    \right).
\end{equation}
\end{proof}

\subsection{Proof of Proposition \ref{prop:new-gain-new-update}}
\label{proof:new-gain-new-update}

\begin{proof}
First, by Proposition \ref{prop:kalman-filter-steps}, the Kalman gain matrix is given by
\begin{equation}
    \vK_t = \vSigma_{t|t-1}\,\gradall_t^\intercal\,\vS_t^{-1}.
\end{equation}
Next, following
Proposition \ref{prop:new-predict-step} and
Proposition \ref{prop:new-innovation-variance},
we rewrite the Kalman gain as
\begin{equation}
\begin{aligned}
    \vK_t
    &= 
    {\rm diag}(\covlast{t|t-1}\, \lrhidden_{t-1}^\intercal\,\lrhidden_{t-1})
    \begin{bmatrix}
    \gradlast_t \\
    \gradhidden_t
    \end{bmatrix}
    \vS_t^{-1}\\
    &= 
    {\rm diag}(\covlast{t|t-1}\, \lrhidden_{t-1}^\intercal\,\lrhidden_{t-1})
    \begin{bmatrix}
    \gradlast_t\,\vS_t^{-1} \\
    \gradhidden_t\,\vS_t^{-1}
    \end{bmatrix}\\
    &= \begin{bmatrix}
        \covlast{t|t-1}\,\gradlast_t\,\vS_t^{-1}\\
        \lrhidden_{t-1}^\intercal\,\lrhidden_{t-1}\,\gradlast_t\,\vS_t^{-1}\\
    \end{bmatrix}\\
    &= \begin{bmatrix}
        \vK_{\plast,t}\\
        \vK_{\phidden_t}
    \end{bmatrix},
\end{aligned}
\end{equation}
where
\begin{equation}
\begin{aligned}
    \vK_{\plast, t} &:= \Big(\vL_t^{-1}\,\vL_{t}^{-\intercal}\gradlast_t\,\covlast{t|t-1}\Big)^\intercal,\\
    \vK_{\phidden, t} &= \Big(\vL_t^{-1}\,\vL_{t}^{-\intercal}\gradhidden_t\,\vW_{t-1}\,\vW_{t-1}^\intercal\Big)^\intercal.
\end{aligned}
\end{equation}

From  the above and Proposition \ref{prop:new-predict-step}
the updated covariance matrix takes the form
\begin{equation}
    \vSigma_{t|t}\\
    = (\vI - \vK_t\,\gradall_t)\,\vSigma_{t|t-1}\,(\vI - \vK_t\gradall_t)^\intercal + \vK_t\,\vR_t\,\vK_t^\intercal,
\end{equation}
where
\begin{equation}
\begin{aligned}
    &\vI - \vK_t\,\gradall_t\\
    &\quad = \begin{bmatrix}
        \vI_\dimlast & \vzero \\
        \vzero & \vI_\dimhidden
    \end{bmatrix}
    - \begin{bmatrix}
        \vK_{\plast,t}\\
        \vK_{\phidden,t}\\
    \end{bmatrix}
    \begin{bmatrix}
        \gradlast_t & \gradhidden_t
    \end{bmatrix}\\
    &\quad = \begin{bmatrix}
        \vI_\dimlast & \vzero \\
        \vzero & \vI_\dimhidden
    \end{bmatrix} - 
    \begin{bmatrix}
        \vK_{\plast, t}\,\gradlast_t & \vK_{\plast, t}\,\gradhidden_t\\
        \vK_{\phidden, t}\,\gradlast_t & \vK_{\phidden,t}\,\gradhidden_t
    \end{bmatrix}\\
    &\quad =
    \begin{bmatrix}
        \vI_{\dimlast} - \vK_{\plast, t}\,\gradlast_t & \vK_{\plast, t}\,\gradhidden_t\\
        \vK_{\phidden, t}\,\gradlast_t & \vI_{\dimhidden} -  \vK_{\phidden,t}\,\gradhidden_t
    \end{bmatrix},
\end{aligned}
\end{equation}
the predicted error covariance $\vSigma_{t|t-1}$ is 
\begin{equation}
    \vSigma_{t|t-1} = 
    \begin{bmatrix}
        \covlast{t|t-1} & \vzero \\
        \vzero & \covhidden{t|t-1}
    \end{bmatrix},
\end{equation}
and
\begin{equation}
\begin{aligned}
    \vK_t\,\vR_t\,\vK_t^\intercal
    &= \begin{bmatrix}
        \vK_{\plast,t} \\
        \vK_{\phidden,t}
    \end{bmatrix}\,
    \vR_t\,
    \begin{bmatrix}
        \vK_{\plast,t}^\intercal & \vK_{\phidden,t}^\intercal
    \end{bmatrix}\\
    &=
    \begin{bmatrix}
        \vK_{\plast,t}\,\vR_t\,\vK_{\plast,t}^\intercal & \vK_{\plast,t}\,\vR_t\,\vK_{\phidden,t}^\intercal\\
        \vK_{\phidden,t}\,\vR_t\,\vK_{\plast,t}^\intercal & \vK_{\phidden,t}\,\vR_t\,\vK_{\phidden,t}^\intercal\\
    \end{bmatrix}.
\end{aligned}
\end{equation}

Thus, 
\begin{equation}\label{eq:part}
\begin{aligned}
    &\vSigma_{t|t}\\
    &\quad  =
    {\tiny
    \begin{bmatrix}
        (\vI_{\dimlast} - \vK_{\plast, t}\,\gradlast_t)\covlast{t|t-1}
        (\vI_{\dimlast} - \vK_{\plast, t}\,\gradlast_t)^\intercal
        + \vK_{\plast,t}\,\vR_t\,\vK_{\plast,t}^\intercal
        &
        \vK_{\plast,t}\,\vR_t\,\vK_{\phidden,t}^\intercal\\
        \vK_{\phidden,t}\,\vR_t\,\vK_{\plast,t}^\intercal &
        (\vI_{\dimhidden} -  \vK_{\phidden,t}\,\gradhidden_t)\covhidden{t|t-1}
        (\vI_{\dimhidden} -  \vK_{\phidden,t}\,\gradhidden_t)^\intercal + \vK_{\phidden,t}\,\vR_t\,\vK_{\phidden,t}^\intercal
    \end{bmatrix}
    }.
\end{aligned}
\end{equation}

Next, we consider the following surrogate covariance matrix which modifies the second block-diagonal entry 
\begin{equation}
\begin{aligned}
&\tilde{\vSigma}_{t|t} =\\
&
\left[
\begin{smallmatrix}
\tiny
(\vI_{\dimlast} - \vK_{\plast, t}\gradlast_t)\covlast{t-1|t-1}(\vI_{\dimlast} - \vK_{\plast, t}\gradlast_t)^\intercal + \vK_{\plast,t}\vR_t\vK_{\plast,t}^\intercal
&
\tiny
\vK_{\plast,t}\vR_t\vK_{\phidden,t}^\intercal
\\
\tiny
\vK_{\phidden,t}\vR_t\vK_{\plast,t}^\intercal
&
\tiny
(\vI_{\dimhidden} - \vK_{\phidden,t}\gradhidden_t)\lrhidden_{t-1}^\intercal\lrhidden_{t-1}(\vI_{\dimhidden} - \vK_{\phidden,t}\gradhidden_t)^\intercal + \vK_{\phidden,t}\vR_t\vK_{\phidden,t}^\intercal + q_{\phidden,t}\vI_{\dimhidden}
\end{smallmatrix}
\right].
\end{aligned}
\end{equation}

It then follows that
\begin{equation}
\begin{aligned}
    \hat{\vSigma}_{\plast,t}, \hat{\vSigma}_{\phidden,t}
    &=
    \argmin_{\vSigma_{\plast}, \vSigma_{\phidden}:{\rm rank}(\vSigma_{\phidden}) = d}
    \left\|\tilde{\vSigma}_{t|t} - \begin{bmatrix} \vSigma_{\plast} &\vzero \\ \vzero & \vSigma_{\phidden} \end{bmatrix}\right\|_{\rm F}^2\\
    &= \argmin_{\vSigma_{\plast}, \vSigma_{\phidden}:{\rm rank}(\vSigma_{\phidden}) = d}
    \|\tilde{\vSigma}_{\plast,t} - \vSigma_{\plast}\|_{\rm F}^2 + 
    \|\tilde{\vSigma}_{\phidden,t} - \vSigma_{\phidden}\|_{\rm F}^2
    + 2\,\|\vK_{\plast,t}\,\vR_t\,\vK_{\phidden,t}^\intercal\|_{\rm F}^2,
\end{aligned}
\end{equation}
with
\begin{equation}
\begin{aligned}
    \tilde{\vSigma}_{\plast,t} &= (\vI_{\dimlast} - \vK_{\plast, t}\,\gradlast_t)\covlast{t-1|t-1}
        (\vI_{\dimlast} - \vK_{\plast, t}\,\gradlast_t)^\intercal
        + \vK_{\plast,t}\,\vR_t\,\vK_{\plast,t}^\intercal,\\
    \tilde{\vSigma}_{\phidden,t} &= 
        (\vI_{\dimhidden} -  \vK_{\phidden,t}\,\gradhidden_t)\lrhidden_{t-1}^\intercal\,\lrhidden_{t-1}
        (\vI_{\dimhidden} -  \vK_{\phidden,t}\,\gradhidden_t)^\intercal + \vK_{\phidden,t}\,\vR_t\,\vK_{\phidden,t}^\intercal
        + q_{\phidden,t}\vI_{\dimhidden}.
\end{aligned}
\end{equation}
Then, $\hat{\vSigma}_{\plast,t} = \tilde{\vSigma}_{\plast,t}$ and
$\hat{\vSigma}_{\phidden,t}$ is given by the first $\dimhiddensub$ principal components of $\tilde{\vSigma}_{\phidden,t}$.

Next, we find
the update rule for the Cholesky and low-rank factors.
First,
the EVC for the last layer takes the form
\begin{equation}
\begin{aligned}
    \hat{\vSigma}_{\plast,t}
    &= \left(\vI - \vK_{\plast, t}\,\gradlast_t\right)\,\covlast{t|t-1}\,\left(\vI - \vK_{\plast, t}\,\gradlast_t\right)^\intercal + \vK_{\plast, t}\,\vR_t\vK_{\plast, t}^\intercal\\
    &= \left(\vI - \vK_{\plast, t}\,\gradlast_t\right)\,\covlast{t|t-1}\,\left(\vI - \vK_{\plast, t}\,\gradlast_t\right)^\intercal
    + \vK_{\plast, t}\,\vR_t^{\intercal/2}\,\vR_t^{1/2}\vK_{\plast, t}^\intercal\\
    &=
    \begin{bmatrix}
        \left(\vI - \vK_{\plast, t}\,\gradlast_t\right)\,\covlast{t-1 |t-1}^{\intercal/2}
        & \vK_{\plast, t}\,\vR_t^{\intercal/2}
    \end{bmatrix}
    \begin{bmatrix}
        \covlast{t|t-1}^{1/2}\,\left(\vI - \vK_{\plast, t}\,\gradlast_t\right)^\intercal \\
        \vR_t^{1/2}\vK_{\plast, t}^\intercal
    \end{bmatrix},
\end{aligned}
\end{equation}
so that
\begin{equation}
    \hat{\vSigma}_{\plast,t}^{1/2} = {\cal Q}_R\left(
        \covlast{t-1|t-1}^{1/2}\,\left(\vI - \vK_{\plast, t}\,\gradlast_t\right)^\intercal,\,
        \vR_t^{1/2}\vK_{\plast, t}^\intercal\right).
\end{equation}

Then, the EVC covariance for the hidden layers (lower right block-diagonal) is
\begin{equation}
\begin{aligned}
    \tilde{\vSigma}_{\phidden,t}
    &= \left(\vI - \vK_{\phidden, t}\,\gradhidden_t\right)\,\lrhidden_{t-1}^\intercal\,\lrhidden_{t-1}\,\left(\vI - \vK_{\phidden, t}\,\gradhidden_t\right)^\intercal
    + \vK_{\phidden, t}\,\vR_t^{\intercal/2}\,\vR_t^{1/2}\vK_{\phidden, t}^\intercal
    + q_{\phidden,t}\,\vI_{\dimhidden}\\
    &=
    \begin{bmatrix}
        \left(\vI - \vK_{\phidden, t}\,\gradhidden_t\right)\,\lrhidden_{t-1}^\intercal & \vK_{\phidden, t}\,\vR_t^{\intercal/2}
    \end{bmatrix}
    \begin{bmatrix}
        \lrhidden_{t-1}\,\left(\vI - \vK_{\phidden, t}\,\gradhidden_t\right)^\intercal \\
        \vR_t^{1/2}\vK_{\phidden, t}^\intercal
    \end{bmatrix}
    + q_{\phidden,t}\,\vI_{\dimhidden}\\
    &= \tilde{\vC}_t^\intercal\,\tilde{\vC}_t + q_{\phidden, t}\,\vI_{\dimhidden}.
\end{aligned}
\end{equation}
Finally, by Corollary \ref{cor:sum-svd-plus-id},
the best rank-$d$ low-rank factor is

\begin{equation}
    \lrhidden_t = {\cal P}_{\dimhiddensub, +q_{\phidden,t}}\left(
        \lrhidden_{t-1}\,\left(\vI - \vK_{\phidden, t}\,\gradhidden_t\right)^\intercal,\,
        \vR_t^{1/2}\vK_{\phidden, t}^\intercal
    \right) \in \mat{(\dimhiddensub + \dimobs)\times\dimhidden}.
\end{equation}
So that
\begin{equation}
    \hat{\vSigma}_{\phidden,t} = \vC_t^\intercal\,\vC_t.
\end{equation}

\end{proof}

\subsection{Proof of Proposition \ref{prop:hilofi-cov-per-step-error}}
\begin{proof}\label{proof:hilofi-cov-per-step-error}
We seek to bound
\begin{equation}
    \| \vSigma_{t|t} - \hat{\vSigma}_{t|t} \|_{\rm F}.
\end{equation}
Following the arguments from Proposition \ref{prop:lrkf-upper-bound-error},
we obtain
\begin{equation}
    \|\vSigma_{t|t} - \hat{\vSigma}_{t|t}\|_{\rm F} \leq
    \|
    \underbrace{
    \vSigma - \tilde{\vSigma}_{t|t}
    }_{=\vE_{\rm surr}}
    \|_{\rm F}
    +
    \|
    \underbrace{
    \tilde{\vSigma}_{t|t} - \hat{\vSigma}_{t|t}
    }_{=\vE_{\rm proj}}
    \|_{\rm F},
\end{equation}
where
\begin{equation}
    \vE_{\rm surr}
    = q_{\phidden,t}\left(2\|\vK_{\phidden,t}\gradhidden_t\|_{\rm F} + \|\vK_{\phidden,t}\gradhidden_t\|_{\rm F}^2\right)
    + q_{\plast,t}\,\left\|\left(\vI - \vK_{\plast, t}\,\gradlast_t\right)\,\left(\vI - \vK_{\plast, t}\,\gradlast_t\right)^\intercal\right\|_{\rm F}
\end{equation}
and
\begin{equation}
    \vE_{\rm proj} = 
    \|\tilde{\vSigma}_{\phidden,t} - \lrhidden_t^\intercal\,\lrhidden_t\|_{\rm F}
    + 2\,\|\vK_{\plast,t}\,\vR_t\,\vK_{\phidden,t}^\intercal\|_{\rm F}
\end{equation}
Given that  $\vC_t^\intercal\,\vC_t$ is the best rank-$d$ approximation of $\tilde{\vSigma}_{\phidden, t}$,
it follows from a similar derivation to \eqref{eq-part:proj-error-lrkf} that
\begin{equation}
    \|\tilde{\vSigma}_{\phidden,t} - \lrhidden_t^\intercal\,\lrhidden_t\|_{\rm F}
    = \sqrt{\sum_{k=\dimhiddensub+1}^\dimhidden \lambda_k^2},
\end{equation}
where $\{\lambda_k\}_{k=\dimhiddensub+1}^\dimhidden$ are bottom $(\dimhidden - \dimhiddensub)$ eigenvalues of $\tilde{\vSigma}_t$.
\end{proof}
\section{Further algorithms}
\label{sec:algo}

\subsection{\newmethodlinlow}
\label{sec:LoLoFi}

Algorithm \ref{algo:low-rank-low-rank-update} shows a single step of \newmethodlinlow.
The low-rank factor for last layers $\vC_{\plast,0} \in \mat{\dimlastsub\times\dimlast}$ and
the low-rank factor for the hidden layers $\vC_{\phidden,0} \in \mat{\dimhiddensub\times\dimhidden}$
are chosen at initialization.

\begin{algorithm}[htb]
\begin{algorithmic}[1]
    \REQUIRE $\vR_t$ \texttt{// measurement variance}
    \REQUIRE $(q_{\plast,t}, q_{\phidden, t})$ dynamics covariance for last layer and hidden layers
    \REQUIRE $(\vy_t, \vx_t)$ \texttt{// observation and input}
    \REQUIRE $\vb_{t-1} = \left(\meanlast{t-1}, \meanhidden{t-1}, \vC_{\plast,t-1}, \vC_{\phidden,t-1}\right)$ \texttt{// previous belief}
    \Statex \texttt{// predict step}
    \STATE $\hat{\vy}_t \gets \nn(\plast_{t-1|t-1}, \phidden_{t-1|t-1}, \vx_t)$
    \STATE $\gradlast_t = \nabla_{\plast} \nn(\plast_{t-1|t-1},\,\phidden_{t-1|t-1},\,\vx_{t})$
    \STATE $\gradhidden_t = \nabla_{\phidden} \nn(\plast_{t-1|t-1},\,\phidden_{t-1|t-1},\,\vx_{t})$
    \Statex \texttt{// innovation(one-step-ahead error) and Cholesky innovation variance}
    \STATE $\vepsilon_t \gets \vy_t - \hat{\vy}_t$
    \STATE $
        \vS_t^{1/2}  \gets {\cal Q}_R\left(
        {\vC}_{\plast, t-1}\,\gradlast_t^\intercal,\,
        \sqrt{q_{\plast,t}}\,\gradlast_t^\intercal,\,
        \vC_{\phidden, t-1}\,\gradhidden_t^\intercal,\,
        \sqrt{q_{\phidden, t}}\,\gradhidden_t^\intercal,\,
        \vR_t^{1/2}
    \right),
    $
    \Statex \texttt{// gain for hidden layers}
    \STATE $\vV_{\phidden,t} \gets \varinnov_t^{-1/2}\,\varinnov_{t}^{-\intercal/2}\gradhidden_t$
    \STATE $\vK_{\phidden, t}^\intercal \gets \vV_{\phidden,t}\,\vC_{\phidden, t-1}\,\vC_{\phidden, t-1}^\intercal + q_{\phidden,t}\,\vV_{\phidden, t}$
    \Statex \texttt{// gain for last layer}
    \STATE $\vV_{\plast,t} \gets \varinnov_t^{-1/2}\,\varinnov_{t}^{-\intercal/2}\gradlast_t$   
    \STATE $\vK_{\plast, t}^\intercal \gets \vV_{\plast, t}\,\covlast{t|t-1} + q_{\plast,t}\,\vV_{\plast, t}$
    \STATE $\vK_{\plast, t}^\intercal \gets \vV_{\plast,t}\,\vC_{\plast, t-1}\,\vC_{\plast, t-1}^\intercal + q_{\plast,t}\,\vV_{\plast, t}$
    \Statex \texttt{// mean update step}
    \STATE $\phidden_{t|t} \gets \phidden_{t-1|t-1} + \vK_{\phidden, t}\vepsilon_t$
    \STATE $\plast_{t|t} \gets \plast_{t-1|t-1} + \vK_{\plast, t}\vepsilon_t $
    \Statex \texttt{// low-rank updates}
    \STATE ${\vC}_{\plast, t} \gets {\cal P}_{\dimlastsub, +q_{\plast,t}}\left( \vC_{\plast, t-1}\,\left(\vI - \vK_{\plast, t}\,\gradlast_t\right)^\intercal,\, \vR_t^{1/2}\vK_{\plast, t}^\intercal\right)$
    \STATE ${\vC}_{\phidden, t} \gets {\cal P}_{\dimhiddensub, +q_{\phidden,t}}\left( \vC_{\phidden, t-1}\,\left(\vI - \vK_{\phidden, t}\,\gradhidden_t\right)^\intercal,\, \vR_t^{1/2}\vK_{\phidden, t}^\intercal\right)$
    \RETURN $\vb_t = \left(\meanlast{t}, \meanhidden{t}, \vC_{\plast,t}, \vC_{\phidden,t}\right)$
\end{algorithmic}
\caption{
Single step of \newmethodlinlow for online learning for $t \geq 1$.
}
\label{algo:low-rank-low-rank-update}
\end{algorithm}

\subsection{Low-rank Kalman filter (\methodlrkf)}\label{sec:LRKF}

Inspired by \cite{tracy2022qrkf},
we derive a novel version of the low-rank Kalman filter
which is
(i) more computationally efficient,
(ii) simple to implement, and
(iii) provably optimal (in terms of Frobenius norm).
We call this
the \methodlrkf method.
Instead of keeping track
 of $\vSigma_{t|t}$ or $\vS_t$,
 we keep track of a low-rank factor for the covariance,
and compute the Cholesky factor for the innovation variance $\vS_t^{1/2}$.
Algorithm \ref{algo:low-rank-filter} summarizes the predict and update steps for \methodlrkf.
Next we turn to the derivation.

We assume the model parameters $\vtheta$ and the observations $\vy$ follow
\begin{equation}
\begin{aligned}
    \vtheta_t &= \vtheta_{t-1} + \vu_t,\\
    \vy_t &= \vH_t\,\vtheta_t + \ve_t,
\end{aligned}
\end{equation}
where $\var(\ve_t) = \vR_t$ is known and $\vH_t$ is the known projection matrix.
Both, $\ve_t$ and $\vu_t$ are zero-mean random vectors.

In the Kalman filter update equations, the so-called posterior mean and covariances are given by
\begin{equation}
\begin{aligned}
    \vmu_t &= \vmu_{t-1} + \vK_t\,(\vy_t - \hat{\vy}_t)\,,\\
    \vSigma_t &= \var(\vtheta_t - \vmu_t),
\end{aligned}
\end{equation}
where $\vmu_0$ is given,
\begin{equation}
\begin{aligned}
    \hat{\vy}_t &= \vH_t\,\vmu_{t|t-1} = \vH_t\,\vmu_{t-1|t-1},\\
    \vK_t &= \cov(\vtheta_t, \vepsilon_t)\var(\vepsilon_t)^{-1} = \vSigma_{t|t-1}\,\vH_t^\intercal\vS_t^{-1},\\
    \vepsilon_t &= \vy_t - \vH_t\,\vmu_{t|t-1},\\
    \vS_t &= \vH_t\,\vSigma_{t|t-1}\,\vH_t^\intercal + \vR_t,\\
    \vR_t &= \var(\ve_t).
\end{aligned}
\end{equation}

The memory requirements of the posterior covariance is  $O(D^2)$, with $D$ number of parameters.
This makes it unfeasible to store for moderately-sized neural networks.

Thus, to reduce the computational cost, we maintain the best rank $d$ approximation of the covariance matrix by
\begin{equation}
    \vSigma_{t|t} \approx \vW_{t}^\intercal\,\vW_{t},
\end{equation}
where $\vW_{t}\in \reals^{d\times D}$ is the best rank-$d$ approximation to $\vSigma_t$ (in a Frobenius norm sense).
Furthermore, to maintain a numerically stable method, we work with the Cholesky of the innovation variance.
In effect, this is a variant square-root low-rank Kalman filter method.
We explain the details of this method below.

The predict step equations are given by
\begin{align}
\vmu_{t|t-1}&= \vmu_{t-1},\\
\vSigma_{t|t-1} &= \lrhidden_{t-1}^\intercal\,\lrhidden_{t-1} + q_t\vI.
\end{align}

\begin{proposition}[Kalman gain and innovations]
The variance  $\vS_t$  of the innovation  is
\begin{equation}
\begin{aligned}
    \vS_t
    &= \vH_t\,\vSigma_{t|t-1}\,\vH_t^\intercal + \vR_t\\
    &= \vH_t\,(\lrhidden_{t-1}^\intercal\,\lrhidden_{t-1} + q_t\vI)\,\vH_t^\intercal + \vR_t\\
    &= \vH_t\,\lrhidden_t^\intercal\,\lrhidden_t\,\vH_t^\intercal + q_t\,\vH_t\,\vH_t^\intercal + \vR_t\\
    &=
    \begin{bmatrix}
        \vH_t\,\lrhidden_t^\intercal &
        \sqrt{q_t}\,\vH_t &
        \vR_t^{\intercal/2}
    \end{bmatrix}
    \begin{bmatrix}
    \lrhidden_t\,\vH_t^\intercal \\
    \sqrt{q_t}\vH_t^\intercal \\
    \vR_t^{1/2}
    \end{bmatrix}\\
    &= \varinnov_t^{\intercal/2}\varinnov_t^{1/2}\,,
\end{aligned}
\end{equation}
with $\varinnov_t^{1/2}$ given by
\begin{equation}
    \varinnov_t = {\cal Q}_R(\lrhidden_t\,\vH_t^\intercal,\,\sqrt{q_t}\,\vH_t^\intercal\,,\,\vR_t^{1/2}),
\end{equation}
where $ {\cal Q}_R$ returns the $R$ matrix from the QR decomposition row-stacked arguments and  $\vR_t^{1/2}$ is the upper-triangular Cholesky decomposition of $\vR_t$.
The Kalman gain is given by
\begin{equation}
\begin{aligned}
    \vK_t
    &= \vSigma_{t|t-1}\,\vH_t^\intercal\,\vS_t^{-1}\\
    &= \left(\vS_t^{-1}\,\vH_t\,\vSigma_{t|t-1}\right)^\intercal\\
    &= \left(\varinnov_{t}^{-1}\,\varinnov_{t}^{-\intercal}\,\vH_t\,
    \left(\lrhidden_{t-1}^\intercal\,\lrhidden_{t-1} + q_t\,\vI_{\dimall}\right)\right)^\intercal.
\end{aligned}
\end{equation}
\end{proposition}

\begin{proposition}[Update step]
    The updated mean is
\begin{equation}
    \vmu_{t} = \vmu_{t|t-1} + \vK_t(\vy_t - \hat{\vy}_t).
\end{equation}

The update covariance $\vSigma_{t|t}$,
approximated through a two-step procedure:
first, a surrogate covariance that ignores the cross-term effect of $\vK_t\,\vH_t$ on the artificial noise covariance $q_t$, and
second, a best rank-$d$ approximation (in Frobenious norm) to the surrogate matrix
takes the low-rank form
\begin{equation}
    \hat{\vSigma}_{t|t} = \lrhidden_t^\intercal\,\lrhidden_t,
\end{equation}
with
\begin{equation}
    \lrhidden_t = {\cal P}_d\left(\lrhidden_{t-1}\,\left(\vI - \vK_t\,\vH_t\right)^\intercal,\,\vR_t^{1/2}\vK_t^\intercal\right)
    \in \mat{d\times\dimall},
\end{equation}
the low-rank (rectangular decomposition) matrix.

\end{proposition}

\begin{proof}
The updated covariance mean follows directly from \eqref{eqn:EKFupdate1}.

The updated covariance is
\begin{equation}\label{eq-part:lrkf-true-cov}
\begin{aligned}
    \vSigma_t
    &= \left(\vI - \vK_t\,\vH_t\right)\,\vSigma_{t|t-1}\,\left(\vI - \vK_t\,\vH_t\right)^\intercal + \vK_t\,\vR_t\vK_t^\intercal\\
    &= \left(\vI - \vK_t\,\vH_t\right)\,\left(\lrhidden_{t-1}^\intercal\,\lrhidden_{t-1} + q_t\,\vI_\dimall\right)\,\left(\vI - \vK_t\,\vH_t\right)^\intercal
    + \vK_t\,\vR_t^{\intercal/2}\,\vR_t^{1/2}\vK_t^\intercal\\
    &= \left(\vI - \vK_t\,\vH_t\right)\,\left(\lrhidden_{t-1}^\intercal\,\lrhidden_{t-1} + q_t\,\vI_\dimall\right)\,\left(\vI - \vK_t\,\vH_t\right)^\intercal
    + \vK_t\,\vR_t^{\intercal/2}\,\vR_t^{1/2}\vK_t^\intercal\\
    &=
    \begin{bmatrix}
        \left(\vI - \vK_t\,\vH_t\right)\,\lrhidden_{t-1}^\intercal & \vK_t\,\vR_t^{\intercal/2}
    \end{bmatrix}
    \begin{bmatrix}
        \lrhidden_{t-1}\,\left(\vI - \vK_t\,\vH_t\right)^\intercal \\
        \vR_t^{1/2}\vK_t^\intercal.
    \end{bmatrix}
    + q_t\,\left(\vI - \vK_t\,\vH_t\right)\left(\vI - \vK_t\,\vH_t\right)^\intercal.
\end{aligned}
\end{equation}
Next, we consider the \textit{surrogate} EVC matrix
\begin{equation}\label{eq-part:lrkf-surrogate-cov}
\begin{aligned}
    \tilde{\vSigma}_t
    &=
    \begin{bmatrix}
        \left(\vI - \vK_t\,\vH_t\right)\,\lrhidden_{t-1}^\intercal & \vK_t\,\vR_t^{\intercal/2}
    \end{bmatrix}
    \begin{bmatrix}
        \lrhidden_{t-1}\,\left(\vI - \vK_t\,\vH_t\right)^\intercal \\
        \vR_t^{1/2}\vK_t^\intercal.
    \end{bmatrix}
    + q_t\vI_\dimall\\
    &= \tilde{\vC}_t^\intercal\,\tilde{\vC}_t + q_t\vI_\dimall,
\end{aligned}
\end{equation}
where
\begin{equation}
    \tilde{\vC}_t = \begin{bmatrix}
        \lrhidden_{t-1}\,\left(\vI - \vK_t\,\vH_t\right)^\intercal \\
        \vR_t^{1/2}\vK_t^\intercal
    \end{bmatrix} \in \reals^{(d+o)\times D}.
\end{equation}

Next, the best rank-$d$ matrix is

\begin{equation}\label{eq-part:lrkf-approx-cov}
\begin{aligned}
    \hat{\vSigma}_{t}
    &= \argmin_{\vSigma:{\rm rank}(\vSigma) = d}
    \left\|\tilde{\vSigma}_{t} - \vSigma\right\|_{\rm F}^2\\
    &= \argmin_{\vSigma:{\rm rank}(\vSigma) = d}
    \left\|\tilde{\vC}_t^\intercal\,\tilde{\vC}_t + q_t\,\vI_\dimall - \vSigma\right\|_{\rm F}^2\\
    &= \lrhidden_t^\intercal\,\lrhidden_t,
\end{aligned}
\end{equation}
where
\begin{equation}
    \lrhidden_t = {\cal P}_{d, +q_t}\left(\lrhidden_{t-1}\,\left(\vI - \vK_t\,\vH_t\right)^\intercal,\,\vR_t^{1/2}\vK_t^\intercal\right)
    \in \mat{d\times\dimall}
\end{equation}
is the best $d$-dimensional low-rank matrix given the stacked matrices $\tilde{\vC}_t$ and the dynamics covariance $q_t$.

\end{proof}

\begin{proposition}\label{prop:lrkf-upper-bound-error}
The per-step error induced by the approximation of the covariance at time $t$ is  bounded by
\begin{equation}
    \| \vSigma_{t|t} - \hat{\vSigma}_{t|t} \|_{\rm F}
    \leq q_t\,\left(2\,\|\vK_t\,\vH_t\|_{\rm F} +  \|\vK_t\,\vH_t\|_{\rm F}^2\right) + \sqrt{\sum_{k=d+1}^\dimall \lambda_k^2},
\end{equation}
\end{proposition}
where $\{\lambda_k\}_{k=d+1}^\dimall$ are the bottom $(\dimall - d)$ eigenvalues of $\tilde{\vSigma}_{t}$.

\begin{proof}
We seek to bound
\begin{equation}
    \| \vSigma_{t|t} - \hat{\vSigma}_{t|t} \|_{\rm F},
\end{equation}
where
$\vSigma_{t|t}$ is given by \eqref{eq-part:lrkf-true-cov} and
$\hat{\vSigma}_{t|t}$ is given by \eqref{eq-part:lrkf-approx-cov}.
We first note that
\begin{equation}
    \vSigma_{t|t} - \hat{\vSigma}_{t|t} = \left(\vSigma - \tilde{\vSigma}_{t|t}\right)
    + \left(\tilde{\vSigma}_{t|t} - \hat{\vSigma}_{t|t}\right),
\end{equation}
where $\tilde{\vSigma}_{t|t}$ is the surrogate covariance matrix \eqref{eq-part:lrkf-surrogate-cov}.
Then,
\begin{equation}
    \|\vSigma_{t|t} - \hat{\vSigma}_{t|t}\|_{\rm F} \leq
    \|
    \underbrace{
    \vSigma - \tilde{\vSigma}_{t|t}
    }_{=\vE_{\rm surr}}
    \|_{\rm F}
    +
    \|
    \underbrace{
    \tilde{\vSigma}_{t|t} - \hat{\vSigma}_{t|t}
    }_{=\vE_{\rm proj}}
    \|_{\rm F}
\end{equation}
The norm for $\vE_{\rm proj}$ follows directly from the definition of the low-rank approximation.
Let $\vU\,\vS\,\vV^\intercal$ be the SVD decomposition of $\tilde{\vSigma}_{t|t}$ and
let $\vU\,\vS_{:d}\,\vV$ be the SVD decomposition of $\hat{\vSigma}_{t|t}$.
Here,
$\vS = {\rm diag}(\lambda_1, \ldots, \lambda_\dimall)$
and
$\vS_{:d} = {\rm diag}(\lambda_1, \ldots, \lambda_d, 0, \ldots, 0)$
are the singular values (eigenvalues) of the matrices $\tilde{\vSigma}_{t|t}$ and $\hat{\vSigma}_{t|t}$ respectively,
ordered in descending order.
Then,
\begin{equation}\label{eq-part:proj-error-lrkf}
\begin{aligned}
    \|\vE_{\rm proj}\|_{\rm F}
    &= \|\vU\,\vS\,\vV^\intercal - \vU\,\vS_{:d}\,\vV^\intercal\|_{\rm F}\\
    &= \|\vU\left(\vS - \vS_{:d}\right)\vV^\intercal\|_{\rm F}\\
    &= \sqrt{{\rm Tr}[\left(\vU\left(\vS - \vS_{:d}\right)\vV^\intercal\right)\left(\vU\left(\vS - \vS_{:d}\right)\vV^\intercal\right)^\intercal]}\\
    &= \sqrt{{\rm Tr}[\vU\left(\vS - \vS_{:d}\right)\vV^\intercal\vV\left(\vS - \vS_{:d}\right)\vU^\intercal]}\\
    &= \sqrt{{\rm Tr}[\left(\vS - \vS_{:d}\right)\vV^\intercal\vV\left(\vS - \vS_{:d}\right)\vU^\intercal\vU]}\\
    &= \sqrt{{\rm Tr}[\left(\vS - \vS_{:d}\right)^2]}\\
    &= \sqrt{\sum_{k=d+1}^\dimall \lambda_k^2},
\end{aligned}
\end{equation}
where $\{\lambda_k\}_{k=d+1}^\dimall$ are the bottom $(\dimall - d)$ eigenvalues of $\tilde{\vSigma}_{t}$.

Next, an upper bound for
the norm $\vE_{\rm surr}$ is as follows
\begin{equation}
\begin{aligned}
    \| \vE_{\rm surr} \|_{\rm F}
    &=  \| \vSigma_{t|t} - \tilde{\vSigma}_{t|t} \|_{\rm F}\\
    &=  \| q_t\,(\vI - \vK_t\,\vH_t)\,(\vI - \vK_t\,\vH_t) - q_t\,\vI \|_{\rm F}\\
    &= q_t\,\| (\vK_t\,\vH_t)\,(\vK_t\,\vH_t)^\intercal - (\vK_t\,\vH_t) - (\vK_t\,\vH_t)^\intercal \|_{\rm F}\\
    &\leq q_t\left( \|(\vK_t\,\vH_t)\|_{\rm F}^2 + 2\,\|\vK_t\,\vH_t\|_{\rm F} \right).
\end{aligned}
\end{equation}
\end{proof}


\begin{algorithm}[htb]
\begin{algorithmic}[1]
    \REQUIRE $\vR_t$ \texttt{// measurement variance}
    \REQUIRE $q_t$ \texttt{// dynamics covariance}
    \REQUIRE $(\vy_t, \vx_t)$ \texttt{// observation and input}
    \REQUIRE $\vb_{t-1} = (\vmu_{t-1}, \lrhidden_{t-1})$ \texttt{// previous belief}
    \Statex \texttt{// predict step}
    \STATE $\hat{\vy}_t \gets h(\vtheta_t, \vx_t)$
    \STATE $\vH_t \gets \nabla_\vtheta h(\vmu_{t-1}, \vx_t)$
    \Statex \texttt{// innovation and (Cholesky) innovation variance}
    \STATE $\vepsilon_t \gets \vy_t - \hat{\vy}_t$
    \STATE $\varinnov_t^{1/2}  \gets {\cal Q}_R(\lrhidden_t\,\vH_t^\intercal,\sqrt{q_t}\,\vH_t,\,\vR_t^{1/2})$
    \Statex \texttt{// gain matrix}
    \STATE $\vV_t \gets \varinnov_{t}^{-1/2}\,\varinnov_{t}^{-\intercal/2}\,\vH_t$
    \STATE $\vK_t^\intercal \gets \vV_t\,\lrhidden_{t-1}\,\lrhidden_{t-1}^\intercal + q_t\,\vV_t$ \texttt{// low-rank update}
    \Statex \texttt{// mean and low-rank factor updates}
    \STATE $\vmu_{t} \gets \vmu_{t-1} + \vK_t\,\vepsilon_t$
    \STATE $\lrhidden_t \gets {\cal P}_d\left(\lrhidden_{t-1}\,\left(\vI - \vK_t\,\vH_t\right)^\intercal,\,\vR_t^{1/2}\vK_t^\intercal\right)$
    \RETURN $\vb_t = (\vmu_t, \lrhidden_t)$ \texttt{// updated belief}
\end{algorithmic}
\caption{Predict and update steps in the low-rank Kalman filter (\methodlrkf)}
\label{algo:low-rank-filter}
\end{algorithm}

\subsubsection{Error analysis for \methodlrkf}
\label{sec:error-analysis-lrkf}
Below,
we analyze the single-step error
incurred by \methodlrkf
without dynamics in the model parameters and scalar linear model, i.e.,
$\vtheta_t = \vtheta_{t-1} = \ldots = \vtheta$
and
$y_t = \vtheta^\intercal\,\vx_t + \ve_t$ with $\var[e_t] = r^2$.
\begin{proposition}
\label{prop:lrkf-approx-error}
Consider the linear model
\[
    y_t = \vtheta^\intercal \vx_t + e_t,
\]
where $e_t$ is zero-mean with $\var[e_t]=r^2$, and let
$\vSigma_{t-1} = \var(\vtheta - \vtheta_{t-1|t-1})$.
Let $\widehat{\vSigma}_{t-1}$ denote the rank-$d$ covariance used by \methodlrkf,
and define
\[
\gamma_t \;=\; \min\!\bigl\{\vx_t^\intercal \vSigma_{t-1}\vx_t,\ \vx_t^\intercal \widehat{\vSigma}_{t-1}\vx_t\bigr\} + r^2
\quad(>0),
\qquad
\epsilon_t \;=\; y_t - \vtheta_{t-1|t-1}^\intercal \vx_t.
\]
Then the one-step difference between the BLUP (full-rank update) and the
rank-$d$ \methodlrkf update satisfies
\begin{equation}
    \|\vtheta_{t|t} - \hat{\vtheta}_{t|t}\|_2
    \;\le\;
    |\epsilon_t|\,\sigma_{d+1}(\vSigma_{t-1})\,\|\vx_t\|_2
    \left(
        \frac{1}{\gamma_t}
        \;+\;
        \frac{\sigma_{1}(\vSigma_{t-1})\,\|\vx_t\|_2^{2}}{\gamma_t^{2}}
    \right).
\end{equation}
Here $\sigma_{1}(\vSigma_{t-1}) \ge \cdots \ge \sigma_{D}(\vSigma_{t-1})$ are the singular values of $\vSigma_{t-1}$,
$\vtheta_{t|t}$ is the BLUP, and $\hat{\vtheta}_{t|t}$ is the rank-$d$ \methodlrkf estimate.
\end{proposition}

In this setting, the surrogate covariance introduced by \methodlrkf is unnecessary,
and $\hat{\vSigma}_{t} = \vC_t^\intercal \vC_t$
reduces to the best rank-$d$ approximation of the EVC matrix $\vSigma_{t-1} = \var(\vtheta - \vtheta_{t-1|t-1})$.
As Proposition \ref{prop:lrkf-approx-error} shows,
a single step of \methodlrkf deviates from the dense (i.e., full-KF) update by a factor governed entirely by the rank truncation.
At time $t$, the residual $\epsilon_t$, the maximum eigenvalue $\sigma_1(\vSigma_{t-1})$, and the feature $\vx_t$ are all fixed.
Consequently, the only way to tighten the bound and reduce the gap to the full update is to increase the chosen rank $d$.

\begin{proof}
We first note that
\begin{equation}
\begin{aligned}
    \vtheta_{t|t} &= \vtheta_{t-1|t-1} + \vK_t\,\varepsilon_t,\\
    \hat{\vtheta}_{t|t} &= \vtheta_{t-1|t-1} + \hat{\vK}_t\,\varepsilon_t,
\end{aligned}
\end{equation}
with
\begin{equation}
\begin{aligned}
    \vK_t = \frac{\vSigma_{t-1}\,\vx_t}{S_t},\quad
    & S_t = \vx_t^\intercal\,\vSigma_{t-1}\,\vx_{t} + r^2,\\
    \hat{\vK}_t = \frac{\hat{\vSigma}_{t-1}\,\vx_t}{\hat{S}_t},\quad
    & \hat{S}_t = \vx_t^\intercal\,\hat{\vSigma}_{t-1}\,\vx_{t} + r^2.
\end{aligned}
\end{equation}
Then
\begin{equation}
\begin{aligned}
    \vK_t - \hat{\vK}_t
    &= \frac{\vSigma_{t-1}\,\vx_t}{S_t} - \frac{\hat{\vSigma}_{t-1}\,\vx_t}{\hat{S}_t} \\
    &= \frac{(\vSigma_{t-1} - \hat{\vSigma}_{t-1})\,\vx_t}{S_t}
       + \hat{\vSigma}_{t-1}\,\vx_t\!\left(\frac{1}{S_t} - \frac{1}{\hat{S}_t}\right).
\end{aligned}
\end{equation}
Thus,
\begin{equation}\label{eq:part-upper-bound-raw}
\begin{aligned}
    \|\vtheta_{t|t} - \hat{\vtheta}_{t|t}\|_2
    &= \| (\vK_t - \hat{\vK}_t)\,\varepsilon_t \|_2
    \le |\varepsilon_t|\,
    \Bigg(
        \underbrace{\left\|\frac{(\vSigma_{t-1} - \hat{\vSigma}_{t-1})\,\vx_t}{S_t}\right\|_2}_{\mathrm{(I)}}
        +
        \underbrace{\left\|\hat{\vSigma}_{t-1}\,\vx_t\!\left(\frac{1}{S_t} - \frac{1}{\hat{S}_t}\right)\right\|_2}_{\mathrm{(II)}}
    \Bigg).
\end{aligned}
\end{equation}

\paragraph{Bound for \(\mathrm{(I)}\).}
\[
\left\|\frac{(\vSigma_{t-1} - \hat{\vSigma}_{t-1})\,\vx_t}{S_t}\right\|_2
\le \frac{\|\vSigma_{t-1} - \hat{\vSigma}_{t-1}\|_2\,\|\vx_t\|_2}{|S_t|}
\le \frac{\sigma_{d+1}(\vSigma_{t-1})\,\|\vx_t\|_2}{\gamma_t},
\]
where we used Eckart–Young–Mirsky \cite{eckart1936lowrank}
and \(|S_t|\ge \gamma_t := \min\{\vx_t^\top\vSigma_{t-1}\vx_t,\ \vx_t^\top\hat{\vSigma}_{t-1}\vx_t\}+r^2>0\).

\paragraph{Bound for \(\mathrm{(II)}\).}
Since
\[
\left|\frac{1}{S_t} - \frac{1}{\hat{S}_t}\right|
= \frac{|\,\hat{S}_t - S_t\,|}{|S_t\,\hat{S}_t|}
= \frac{|\,\vx_t^\top(\hat{\vSigma}_{t-1}-\vSigma_{t-1})\vx_t\,|}{|S_t\,\hat{S}_t|}
\le \frac{\|\vx_t\|_2^2\,\|\hat{\vSigma}_{t-1}-\vSigma_{t-1}\|_2}{\gamma_t^2},
\]
we obtain
\[
\mathrm{(II)} \le \|\hat{\vSigma}_{t-1}\|_2\,\|\vx_t\|_2\,
\frac{\|\vx_t\|_2^2\,\|\hat{\vSigma}_{t-1}-\vSigma_{t-1}\|_2}{\gamma_t^2}
= \frac{\|\hat{\vSigma}_{t-1}\|_2\,\|\vx_t\|_2^3\,\sigma_{d+1}(\vSigma_{t-1})}{\gamma_t^2}.
\]
Since \(\|\hat{\vSigma}_{t-1}\|_2=\sigma_1(\vSigma_{t-1})\), this becomes
\[
\mathrm{(II)} \le \frac{\sigma_1(\vSigma_{t-1})\,\|\vx_t\|_2^3\,\sigma_{d+1}(\vSigma_{t-1})}{\gamma_t^2}.
\]

Plugging the bounds for \(\mathrm{(I)}\) and \(\mathrm{(II)}\) into \eqref{eq:part-upper-bound-raw} yields
\[
\|\vtheta_{t|t}-\hat{\vtheta}_{t|t}\|_2
\le |\varepsilon_t|\,\sigma_{d+1}(\vSigma_{t-1})\,\|\vx_t\|_2
\left(\frac{1}{\gamma_t} + \frac{\sigma_1(\vSigma_{t-1})\,\|\vx_t\|_2^{2}}{\gamma_t^{2}}\right),
\]
as claimed.
\end{proof}

\subsection{Replay-buffer variational Bayesian last layer  (\newmethodvb)}
\label{sec: new method vb}

The \methodvbll method has explicit posterior predictive
\begin{equation}
    p(\vy \cond \vell, \vh, \vSigma, \vR) =
    {\cal N}\left(\vy \cond \vell^\intercal\,\psi(\vh, \vx+),\,
    \vell^\intercal\,\psi(\vh, \vx)\,\vSigma\,\psi(\vh, \vx)^\intercal\,\vell + \vR\right). \nonumber
\end{equation}

Building on \cite{harrison2024variational,brunzema2024bayesian},
we observe that for any given $q(\bar{\plast}) = {\cal N}(\bar{\plast} \cond \plast, \vSigma)$,
 a lower-bound for the marginal log-likelihood of a single datapoint is given by
\begin{equation}
\begin{aligned}
    \log p(\vy_t \cond \vx_t, \phidden, \vR)
    &\geq \mathbb{E}_{q(\bar{\plast})}[\log p(\vy_t \cond \vx_t, \bar{\plast}, \phidden, \vR)]\\
    &= \log {\cal N}\left(\vy_t \cond \vell^\intercal\phi(\vh, \vx_t), \vR\right)
    - \frac{1}{2}\phi(\vh, \vx_t)\vSigma\,\phi(\vh, \vx_t)\Tr\left(\vR^{-1}\right)\\
    &=: -{\cal L}_t(\plast, \phidden, \vR, \vSigma).
\end{aligned}
\end{equation}

Following \cite{knoblauch2022optimization, khan2023bayesian, jones2024bayesian}
a \textit{generalized} posterior with loss function ${\cal L}$ is given by
\begin{equation}
    q_t(\plast, \phidden, \vR, \vSigma) \propto q_{t-1}(\plast, \phidden, \vR, \vSigma)\,\exp(-{\cal L}_t(\plast, \phidden, \vR, \vSigma)).
\end{equation}

Estimation of the posterior mean (MAP filtering) involves estimating
\begin{equation}
    \plast_t, \phidden_t, \vR_t, \vSigma_t = \argmax_{\plast, \phidden, \vR, \vSigma}q_t(\plast, \phidden, \vR, \vSigma)
\end{equation}
which can be done implicitly through adaptive optimization methods  as shown in \cite{bencomo2023implicit}.
Next, to improve the performance,
we consider a replay-buffer which has been shown to be much more efficient than doing fully-online SGD
\cite{arnaboldi2024repeatsgd,dandi2024repeatsgd,lee2024repeatsgd}.
 This assumption breaks with the fully-online assumption,
but it is a good contender in sequential decision making problems in stationary environments.

\section{Additional experiments}
\label{sec:experiments-further-results}

\subsection{Online classification and sequential decision making on the MNIST dataset}
\label{sec:further-results-mnist}

The results in Section \ref{experiment:mnist-bandits} consider the LeNet5 convolutional neural network (CNN) architecture
\cite{lecun1998gradient}.
In this architecture, the last layer is $80$-dimensional and the output layer is $10$-dimensional,
which corresponds to each of the $10$ possible classes.
Applying \newmethodlin in this setting would need the storage and update of an $800\times800$ covariance matrix
and $800$-dimensional mean vector.
Sequential update of the Cholesky representation of this covariance matrix, although feasible,
takes most of the computational cost in a single step of \newmethodlin.
Given this, here we consider \newmethodlinlow to offset the computational cost of the layer layer.
Another reason to prefer \newmethodlinlow over \newmethodlin in this setting is that the rank of the last layer is $800$-dimensional, 
whereas the rank for the hidden layers is $50$ dimensional.
This corresponds to an overparametrized linear regression, which has been shown to have some pathologies in the offline setting \cite{xu2023linregoverparameterised}.

As a simpler baseline, which does not model uncertainty,
we consider  \textbf{Muon} (\methodshampoo) \cite{jordan2024muon},
which a special case of the Shampoo optimizer \cite{gupta2018shampoo}.
This is  a quasi second-order optimization method that is  scalable and shows strong empirical performance. 
We use the implementation in the Optax library \cite{deepmind2020jax}.
Unlike the other methods, it only computes a point estimate, so we cannot use it for computing the posterior predictive.
We use this method for the $\epsilon$-greedy bandit in Appendix \ref{sec:further-results-mnist-bandit}
and as an additional choice of optimizer in Section \ref{sec:classification-as-regression}.

\subsubsection{Online classification on the MNIST dataset}
\label{sec:classification-as-regression}

In this experiment, we consider the problem online learning and classification on the MNIST dataset using the LeNet5 CNN.
We use the optimizers detailed in Section \ref{experiment:mnist-bandits}.
We take the data $\data_{1:T}$ with $T=60,000$ and $\data_t = (\vx_t, \vy_t)$,
where $\vx_t$ is a $(28\times28\times 1)$ array and $\vy_t \in \{0,1\}^{10}$ a one-hot-encoded
vector such that $(\vy_t)_i = 1$ if $\vx_t$ represents the digit $i$ and $0$ otherwise.
At every timestep $t=1,\ldots, T$
each agent is presented the image $\vx_t$ which it has to classify.
A predicted classification is made through the prediction
$\vy_{t|t-1} = \nn(\vmu_{t|t-1}, \vx_t)$ with $\vmu_{t|t-1} = \mathbb{E}[\vtheta_t \cond \data_{1:t-1}]$
and then updates its beliefs given the (true) reward $\vy_t$. 

Having the BLUP over parameters $\vtheta_{t|t-1} = \mathbb{E}[\vtheta_t \cond \data_{1:t-1}]$,
the linearized model has mean and variance
\begin{equation}\label{eq:linearised-multi-output-clf-model}
\begin{aligned}
    \vm(\vtheta_t, \vx_t) &= {\rm softmax}(\nn(\vtheta_t, \vx)),\\
    \bar{\vm}_t &= \vm(\vtheta_{t|t-1}, \vx_t) + \nabla_\vtheta \vm(\vtheta_{t|t-1}, \vx)(\vtheta - \vm_{t-1}),\\
    \bar{\vs}_t &= {\rm diag}(\vm(\vtheta_{t|t-1}, \vx_t)) + \vm(\vtheta_{t|t-1}, \vx_t)\,\vm(\vtheta_{t|t-1}, \vx_t)^\intercal +  \varepsilon\,\vI_\dimobs.
\end{aligned}
\end{equation}
This moment-matched linearization corresponds, in the Bayesian setting,
to a Gaussian likelihood whose first and second moments match that of a Multinomial distribution.
This representation was introduced in the context of online learning in \cite{ollivier2018expfamekf}.

In \eqref{eq:linearised-multi-output-clf-model},
${\rm diag}(\vv)$ is a function that takes as input the vector $\vv \in \reals^m$ and outputs a
diagonal matrix with entries ${\rm diag}(\vv)_{i,j} = \vv_i\,\delta(i - j)$, 
and 
the term $\varepsilon > 0$ is a small constant that ensures non-zero variance.

\paragraph{Online learning results.}
Figure \ref{fig:mnist-online-classification}
shows the 5000-step rolling mean of the one-step-ahead classification outcome
(one for correct, zero for incorrect)
using the various learning algorithms.
We observe that there is no clear difference
between the methods.
This shows  that one can tackle this complete information problem 
without needing to model uncertainty.
This is not the case in the incomplete information case that we study in any of the experiments presented in Section \ref{sec:experiments}.

\begin{figure}[htb]
    \centering
    \includegraphics[width=0.9\linewidth]{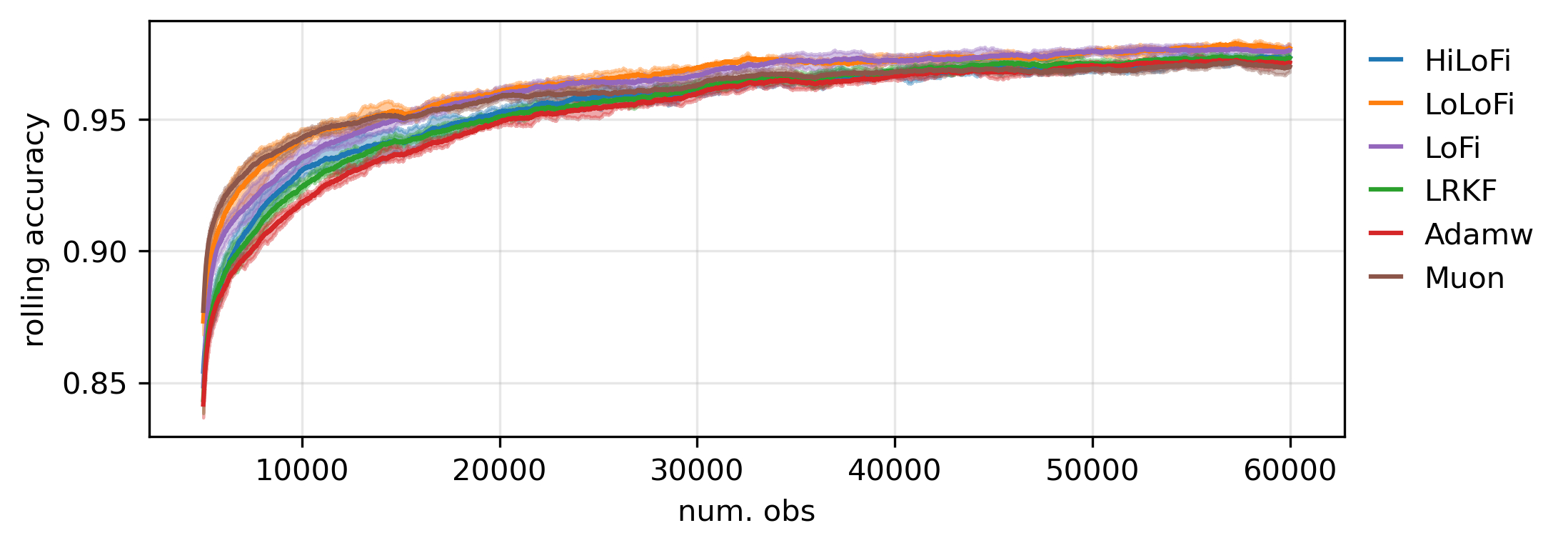}
    \caption{
        Rolling one-step-ahead accuracy for the MNIST dataset.
    }
    \label{fig:mnist-online-classification}
\end{figure}

Figure \ref{fig:mnist-online-classification-rank-comparison}
shows the one-step-ahead accuracy of the last $10,000$ images and the running time of processing all $60,000$ images
for \newmethodlin, \newmethodlinlow, \methodlrkf, and \methodlofi as a function of their rank.
For \newmethodlinlow, we fix the rank of the last layer to be $50$ and vary the rank of the hidden layers.
\begin{figure}[htb]
    \centering
    \includegraphics[width=\linewidth]{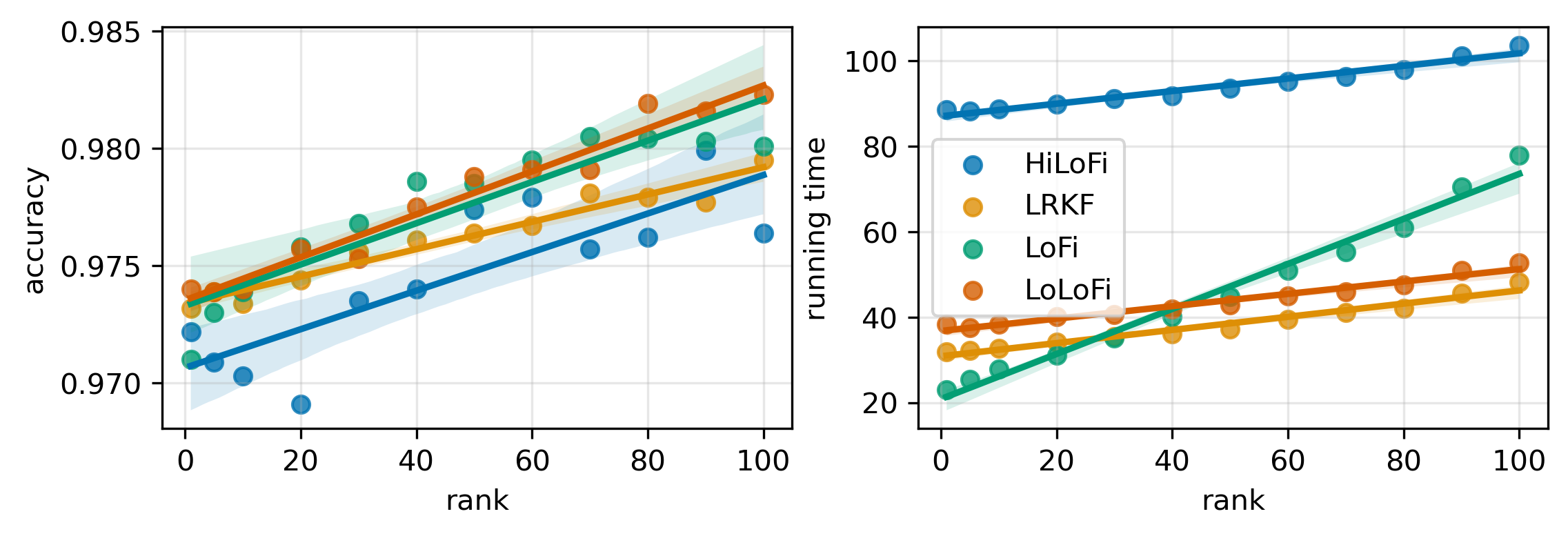}
    \caption{
        Results of MNIST for online classification.
        \textit{Left panel}: accuracy as a function of rank.
        \textit{Right panel:} running time as a function of rank.
    }
    \label{fig:mnist-online-classification-rank-comparison}
\end{figure}
We observe that for low-ranks up to dimension $20$, \methodlofi is faster than both \newmethodlinlow and \methodlrkf.
However, the running time of \methodlofi increases much more rapidly (as a function of rank) than either \newmethodlinlow and \methodlrkf. 
This is a consequence of the computational costs associated 
with points (1,2,3) in Appendix \ref{sec:computational-complexity}.
Next, we observe that the running time of \newmethodlinlow is sightly above \methodlrkf for varying rank;
this is because of the double approximation of the covariance matrix: one for the hidden layers and another one for the last layer.
Finally, we observe that \newmethodlin is the method with highest computational cost in this experiment.
This is due the high-dimensionality of the last-layer, relative to the rest of the methods.

Lastly, we observe that all methods increase their performance as a function of rank, albeit marginally.
The method that most benefit from an increase in performance in \newmethodlin.
We hypothesize that this is due to the total rank of the hidden layer, relative to the rank of the last layer,
which in this experiment is $800$.

\subsubsection{Error analysis for \methodlrkf}
\label{sec:error-analysis-lrkf-mnist}

The top panel of Figure \ref{fig:mnist-lenet-error-analysis} shows
the rolling mean for the two sources of approximation error in the covariance
matrix for the \methodlrkf method (detailed in Proposition \ref{prop:lrkf-upper-bound-error}),
as well as the rolling one-step-ahead accuracy.
\begin{figure}[htb]
    \centering
    \includegraphics[width=0.9\linewidth]{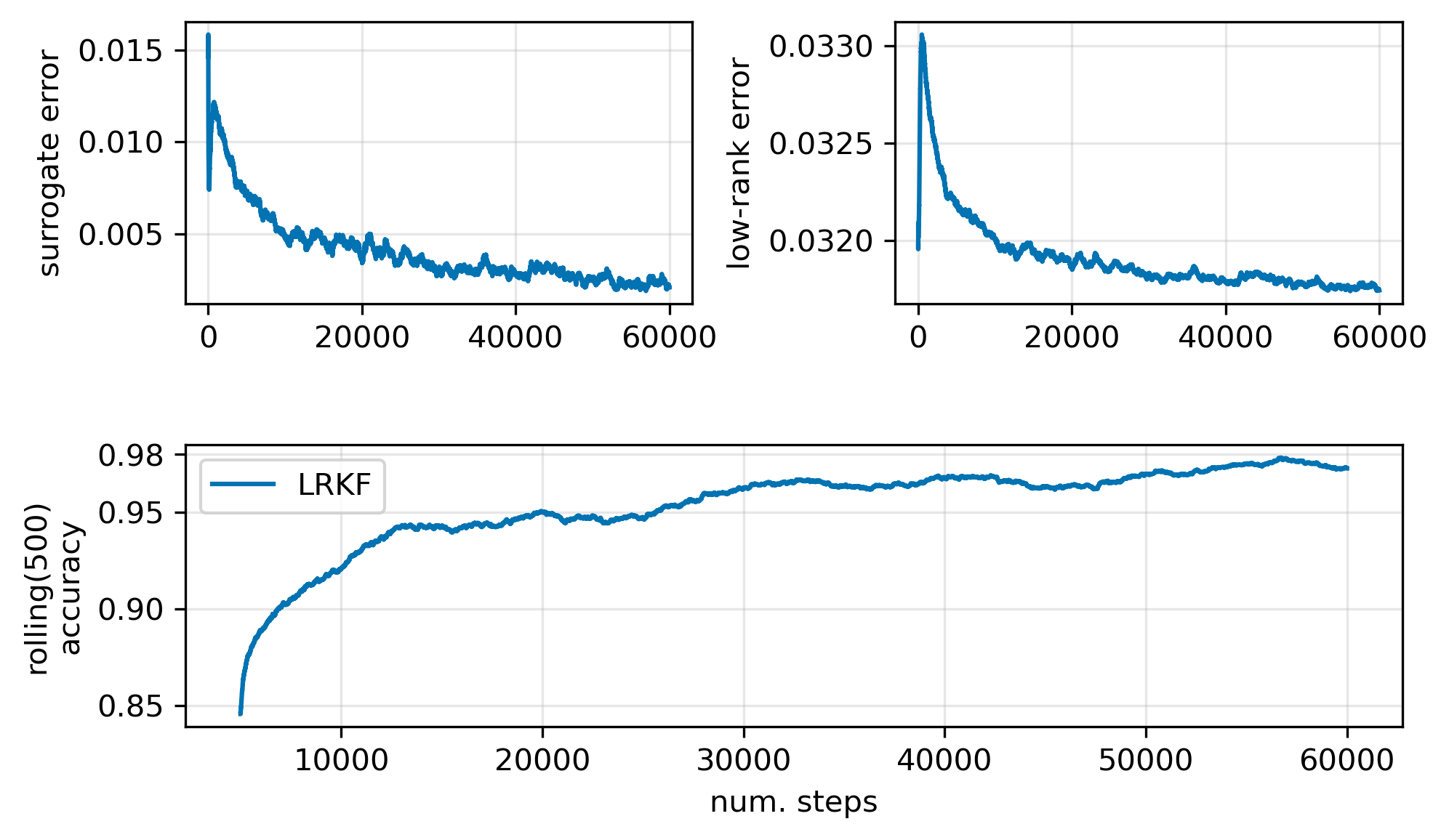}
    \caption{
        Top: upper bound error for the surrogate error and the low-rank approximation.
        Bottom: $5,000$-step rolling one-step-ahead accuracy for the MNIST dataset.
    }
    \label{fig:mnist-lenet-error-analysis}
\end{figure}
We observe that, on average, both sources of error decrease over time and the accuacy of \methodlrkf improves over time.
Here, the only sources of error are the surrogate error and the low-rank error.
This is because \methodlrkf does not distinguish between last-layer and hidden-layer parameters.

\subsubsection{MNIST as multi-armed bandit}
\label{sec:further-results-mnist-bandit}

Here, we present further results from Section \ref{experiment:mnist-bandits}
More precisely,  we study the multi-armed bandit approach to solving the MNIST classification problem.

\paragraph{Choice of hyperparameters.}
The hyperparameters for this experiment were chosen as follows:
For Adamw and \methodshampoo (used in $\varepsilon$-greedy and in conjunction with \methodlaplace),
we use a learning rate of $10^{-4}$, and $\varepsilon=0.05$
For Adamw, we take 5 inner iterations and a buffer size of one and for \methodshampoo, we take 1 inner iteration and a buffer size of one.
Next, \newmethodlin considers a rank of $50$ for the hidden parameters, we take $q_{\phidden, t} = 10^{-6}$ and $q_{\plast} = 10^{-6}$.
The initial covariances $\vSigma_{\phidden, 0}$ and $\vSigma_{\plast, 0}$ are both initialized as identity times a factor of $10^{-1}$.
Similarly, \newmethodlinlow is initialized as with \newmethodlin, but with rank in the last-layer to be $100$.
For \methodlrkf, we consider a rank of $50$, $\vSigma_{0}$ is initialized as the identity matrix, and $q_t = 10^{-6}$.
Finally, for \methodlofi, we also considered rank $50$;
then, following the experiments in \cite{chang23lofi},
we considered the initial covariance to be $a\vI$, with $a=\exp(-8)$
(this corresponds to low-rank comprised of a matrix of zeros and diagonal terms, all set to $a$); lastly,
 the dynamics covariance is $q_t = 0$.

All initial model parameters from the neural network are shared across methods and trials.

\paragraph{Regret analysis.}
Below, we present results considering regret. Here, the reward is either $1$,
if the classification is done correctly and $0$ otherwise.
Let $y_t \in \{0, 1\}$ be the reward obtained at time $t$,
and $\vy_{t,a}$ be the value of arm $a$.
Then the regret obtain at time $t$ is given by
$$
    \sum_{\tau=1}^t (\max_a \vy_{\tau,a} - y_\tau) = t - \sum_{\tau=1}^t y_{\tau}.
$$
Figure \ref{fig:mnist-bandit-regret} shows the average regret (across 10 trials)
for \newmethodlin, \newmethodlinlow, \methodlrkf, \methodlofi, and \methodlaplace.
Here, we consider $\epsilon$-greedy variants for \newmethodlin, \newmethodlinlow, and \methodlrkf.
Next,
TS with \methodlaplace and Adamw optimizer is denoted adamw-TS,
TS with \methodlaplace and   \methodshampoo optimizer is denoted muon-TS.
$\epsilon$-greedy with Adamw optimizer is the denoted adamw-eps,
and
$\epsilon$-greedy with \methodshampoo optimizer is denoted muon-eps.
\begin{figure}[htb]
    \centering
    \includegraphics[width=\linewidth]{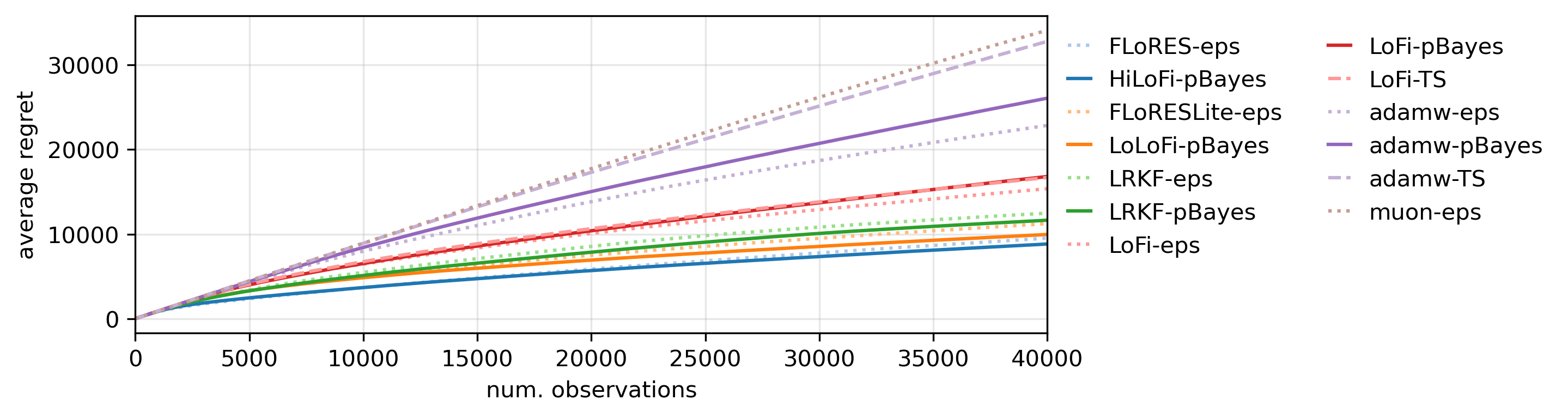}
    \caption{
        Total running time for experiments ($x$-axis) versus
        average regret across ten trials ($y$-axis).
    }
    \label{fig:mnist-bandit-regret}
\end{figure}

\paragraph{Time-regret analysis.}
Figure \ref{fig:mnist-bandit-time-reward}
shows the total running time to run the ten trials in the $x$-axis,
the regret around one-standard deviation in the $y$-axis.
\begin{figure}[htb]
    \centering
    \includegraphics[width=\linewidth]{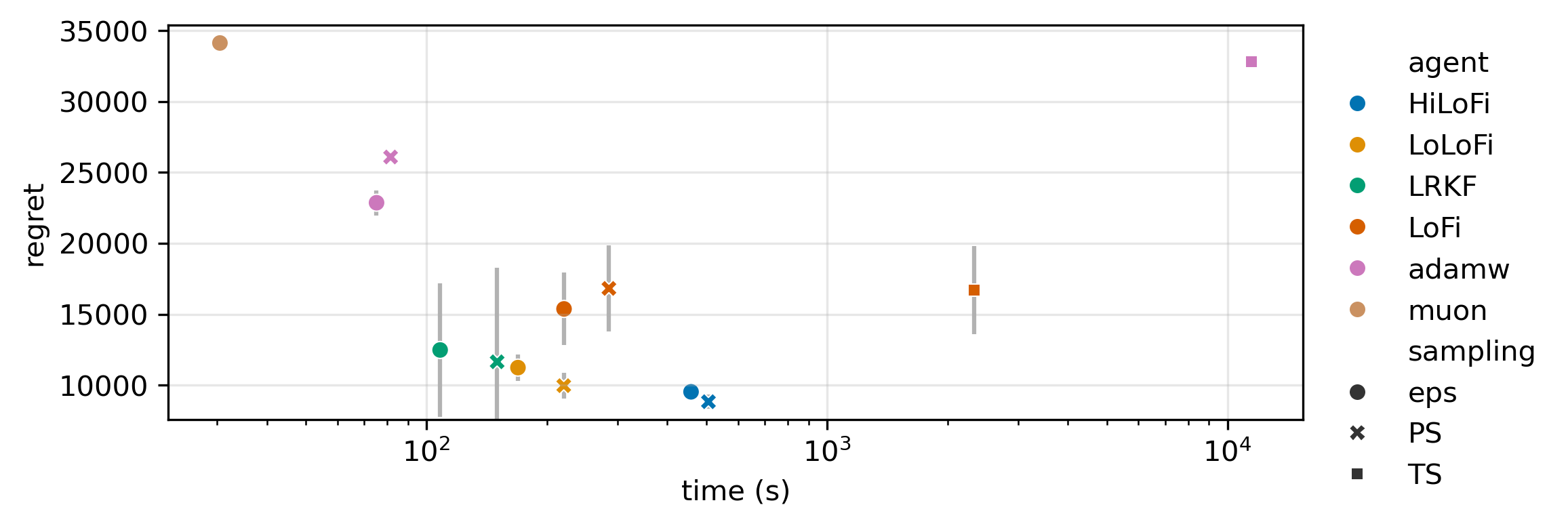}
    \caption{
        Total running time for experiments ($x$-axis) versus
        mean relative cumulative reward ($y$-axis).
    }
    \label{fig:mnist-bandit-time-reward}
\end{figure}
We observe Adamw and \methodshampoo are the methods with lowest performance and lowest relative cumulative reward
for either $\epsilon$-greedy or TS using \methodlaplace.
Next, we observe that \methodlrkf is faster than \newmethodlinlow, and \newmethodlinlow is faster than \newmethodlin.
This result is expected and is explained in Appendix \ref{sec:computational-complexity}.
We also observe that \methodlofi with $\epsilon$-greedy matches \newmethodlinlow with PS in time;
however, \newmethodlinlow has lower regret.
We conjecture that this result is due to the explicit modeling of the last layer done for \newmethodlinlow.

Finally, we observe that for all methods, $\epsilon$-greedy is faster than TS.
This is because $\epsilon$-greedy only requires the evaluation of the posterior predictive mean,
whereas TS samples from the posterior predictive,
which in turn requires building the posterior predictive covariance shown in \eqref{eq:pp-hilofi}.

\subsection{Online classification using \methodlrkf on the CIFAR-10 dataset}
\label{sec:cifar-10-vgg}
In this Section, we study the ability of \methodlrkf to scale to million-dimensional parameters
for online classification on the CIFAR-10 dataset.
We consider two VGG-style convolutional networks of different capacity. Both follow the pattern of stacked
$3\times 3$ convolutions, ELU activations.
We summarize the architectures below
\begin{table}[htb]
    \centering
    \begin{tabular}{lccc}
    \toprule
    \textbf{Model} & \textbf{Conv Blocks} & \textbf{Dense Layers} & \textbf{Params} \\
    \midrule
    VGG        & $64\times2 \rightarrow 128\times2 \rightarrow 256$ & $256 \rightarrow 256$ & 1.7M \\
    VGG+Block  & $64\times2 \rightarrow 128\times2 \rightarrow 256 \rightarrow 512\times2$ & $256 \rightarrow 256$ & 4.7M \\
    \bottomrule
    \end{tabular}
    \caption{Summary of the two VGG-style architectures used in our experiments. 
    ``$C\times n$'' denotes $n$ convolutional layers with $C$ channels.}
    \label{tab:vgg-architectures}
\end{table}
We follow a similar setup as that of Appendix \ref{sec:classification-as-regression}.
We take the data $\data_{1:T}$ with $T=10,000$ and $\data_t = (\vx_t, \vy_t)$,
where
$\vx$ is a $(32\times32\times 3)$-dimensional array and $\vy_t \in \{0,1\}^{10}$
a one-hot encoded vector such that $\vy_t$.
At every timestep $t=1,\ldots, T$
each agent is presented the image $\vx_t$ which it has to classify.
A predicted classification is made through the prediction
$\vy_{t|t-1} = \nn(\vmu_{t|t-1}, \vx_t)$ with $\vmu_{t|t-1} = \mathbb{E}[\vtheta_t \cond \data_{1:t-1}]$
and then updates its beliefs given the (true) reward $\vy_t$. 
Updates are done as outlined in Appendix \ref{sec:classification-as-regression}.

Table \ref{tab:vgg-results} shows the one-step-ahead accuracy for \methodlrkf
for ranks $1$, $5$, and $10$.
\begin{table}[htb]
    \centering
    \begin{tabular}{lccc}
    \toprule
    \textbf{Model} & \textbf{\methodlrkf-1} & \textbf{\methodlrkf-5} & \textbf{\methodlrkf-10} \\
    \midrule
    VGG (1.7M)       & 0.3355 & 0.3384 & 0.3379 \\
    VGG+Block (4.7M) & 0.3129 & 0.3141 & 0.3148 \\
    \midrule
    \textbf{AdamW (baseline)} & \multicolumn{3}{c}{VGG: 0.3442 \quad\; VGG+Block: 0.3097} \\
    \bottomrule
    \end{tabular}
    \caption{
    Online classification results on CIFAR-10 using \methodlrkf with varying ranks (1, 5, 10),
    compared against AdamW baselines.
    Lower values are better.
    }
    \label{tab:vgg-results}
\end{table}
These results show that \methodlrkf scales reliably to multi-million parameter networks
while maintaining stable performance across different low-rank configurations.
This highlights the broader applicability of our approach (including \newmethodlin and \newmethodlinlow)
to large-scale architectures in vision and related domains where sequential decision making is needed.

\subsection{Bandits as recommendation systems}
\label{sec:further-results-recommend}
Further results from Section \ref{experiment:bandits}

\paragraph{Choice of hyperparameters.}
In this experiment, all agents shared a \textit{rank} of $20$.
The choice of hyperparameters for \methodlaplace, \methodlofi, \methodvbll, and \newmethodvb
were chosen following analysis of performance on the first $10,000$ observations of the datasets.
In particular, we tried Bayesian optimization techniques, but found that manual tuning of
hyperparameters on the first $10,000$ observations yielded higher overall performance.

For \newmethodvb, we considered a buffer size of $10$, a regularization weight of $1.0$, and $50$ inner iterations.
For \methodvbll, we considered a buffer size of $1000$,
a regularization weight of $10^{-3}$,
and
$100$ inner iterations.

Next, for \newmethodlin, we take $q_{\phidden,t} = q_{\plast,t} = 0.0$ and set the initial covariances to
be the identity.
Next, for \methodlrkf, we take $q_t = 0$ and identity matrix as initial covariance.

\paragraph{Further results.}
Figure \ref{fig:bandits-cumulative-rewards} shows the daily cumulative reward for all methods.
To produce this plot, we sum the cumulative rewards for all users at the end of each day and then perform
a cumulative sum over the end-of-day reward.
\begin{figure}[htb]
    \centering
    \includegraphics[width=0.8\linewidth]{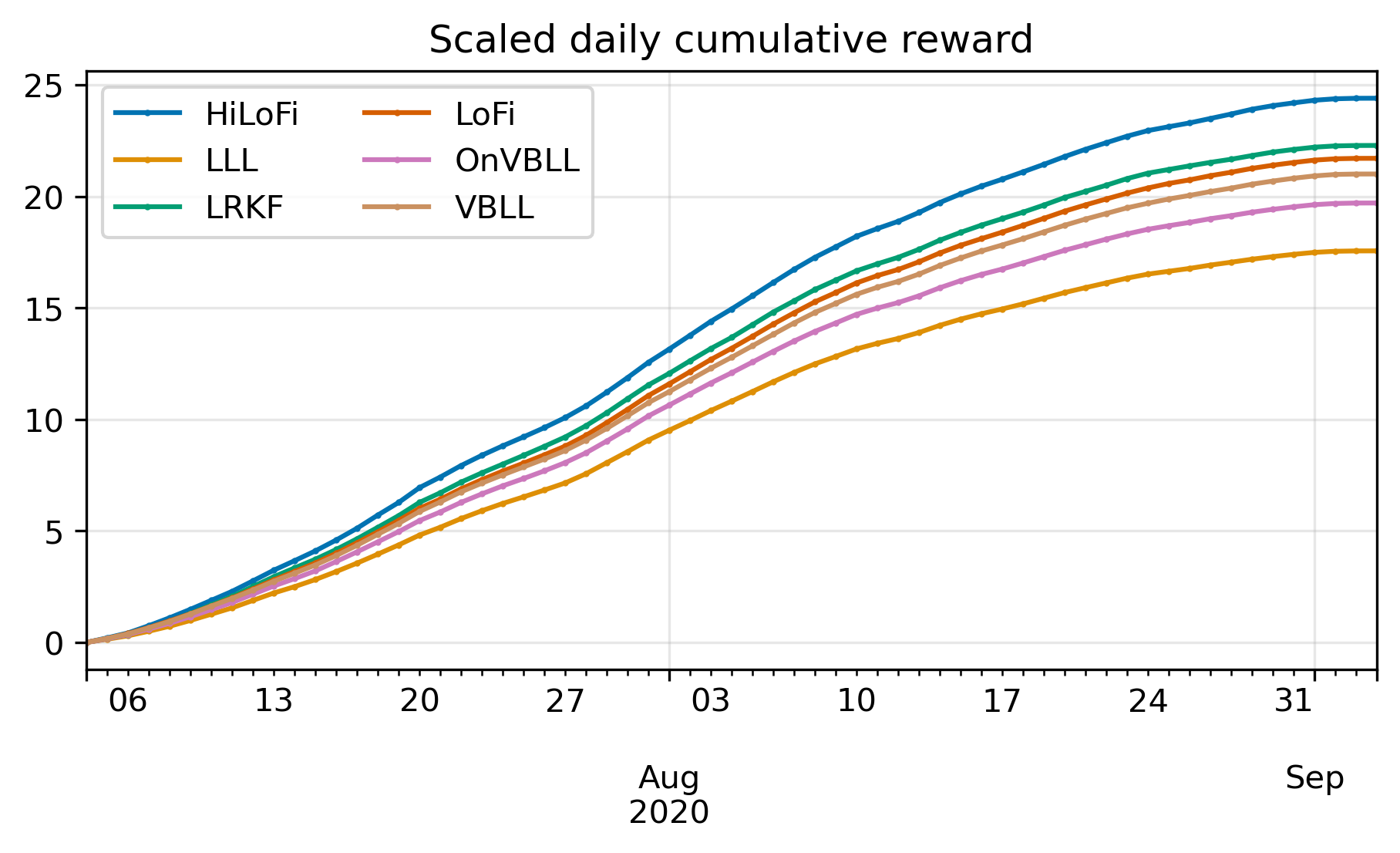}
    \caption{
        Daily cumulative reward.
    }
    \label{fig:bandits-cumulative-rewards}
\end{figure}
We observe that \newmethodlin obtains the highest daily cumulative reward.
The method with second highest daily cumulative reward is \methodlrkf, closely followed by \methodlofi and \methodvbll.
However, as shown in Figure \ref{fig:bandits-reward-boxplot}. \methodlrkf is more than an order of magnitude faster than \methodvbll.
From Figure \ref{fig:bandits-cumulative-rewards} and Figure \ref{fig:bandits-cumulative-rewards},
we observe that the top performing methods in this experiment are \newmethodlin and \methodlrkf.
However, \methodlrkf is slightly faster than \newmethodlin relative to all other methods.

\subsection{Bayesian optimization}
\label{sec:further-results-BO}
Further results from Section \ref{experiment:bayesopt}

\paragraph{Choice of hyperparameters.}  We fix hyperparameters across methods to ensure fairness and reproducibility. 
For rank-based methods, we use a rank of 50. 
Methods requiring buffers (e.g., \methodlaplace, \newmethodvb) use a FIFO replay buffer of size 20, 
with 50 inner iterations per update. 
\methodvbll uses the full dataset and 100 iterations per step. 


Where applicable, aleatoric uncertainty is discarded ($\vR_t = 0$) to isolate epistemic effects. 
We consider equal learning rates across optimizer-based methods (e.g., \methodlaplace, \newmethodvb) at $10^{-4}$. 
In filtering-based methods (e.g., \newmethodlin, {\methodlofi}), this learning rate corresponds to the initial hidden-layer covariance; 
we set the final-layer prior variance to 1, emphasising epistemic modelling in the output layer. 
Hyperparameters for {\methodvbll} (Wishart scale, regularization weight) follow prior work \cite{brunzema2024bayesian}.

\paragraph{Per-step results.}
Figure \ref{fig:bayesopt-steps} shows the median performance and the interquartile range
of the best value found by each method for all test datasets.
\begin{figure}[htb]
    \centering
    \includegraphics[width=\linewidth]{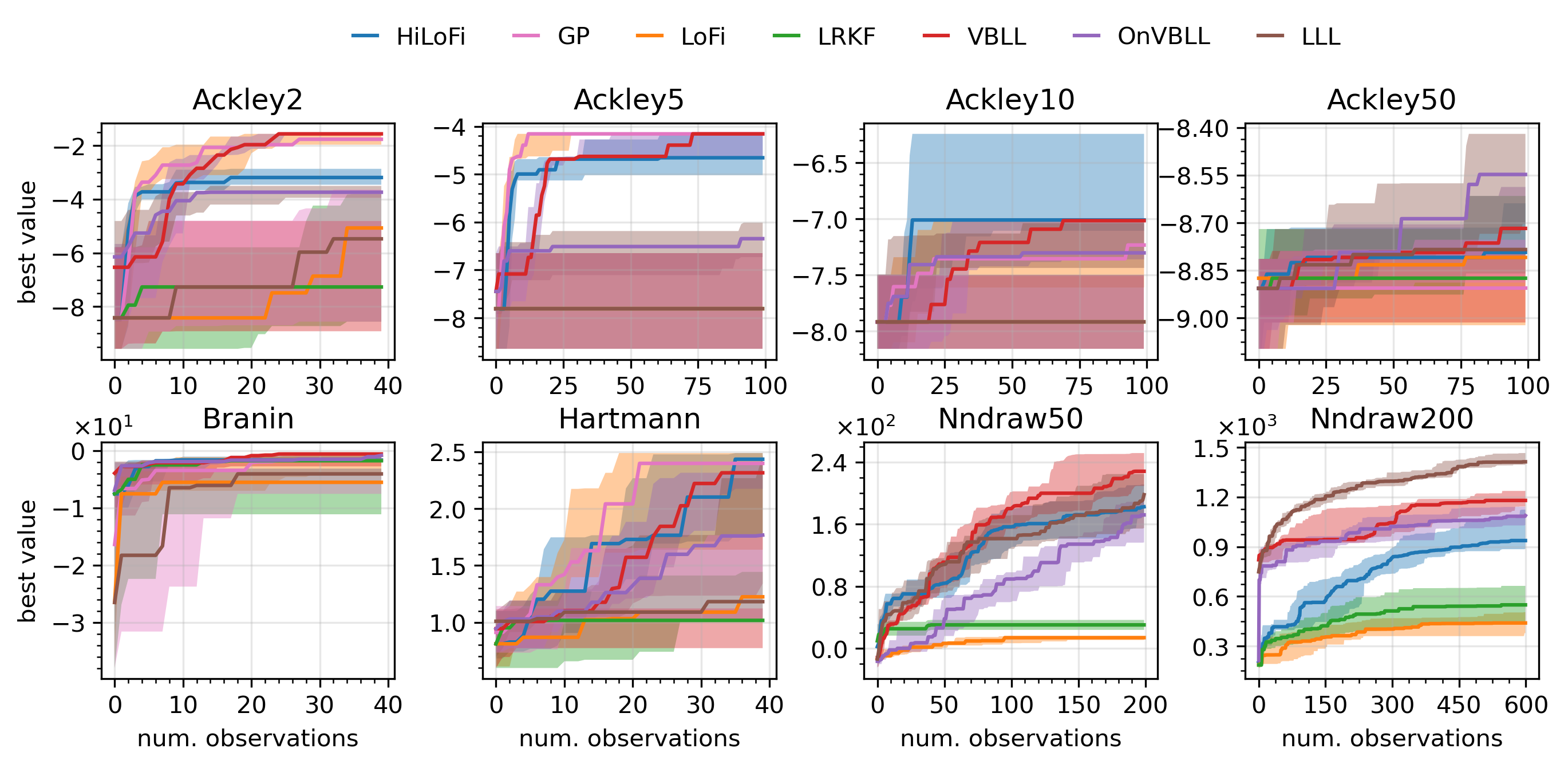}
    \caption{
        Bayesian optimization benchmark on stationary environments.
    }
    \label{fig:bayesopt-steps}
\end{figure}
We observe that \methodvbll, \newmethodvb, and \newmethodlin are among the top performers for all datasets.
Next, \methodlaplace is the least performant method for the low-dimensional datasets, but among the top performers
for Nndraw50 and the top perfomer in Nndraw200.

Next, Figure \ref{fig:bayesopt-iterations-expected-improvement} shows the median performance and the
interquartile range of the best value found by each method for all test datasets using expected improvement.
\begin{figure}[htb]
    \centering
    \includegraphics[width=\linewidth]{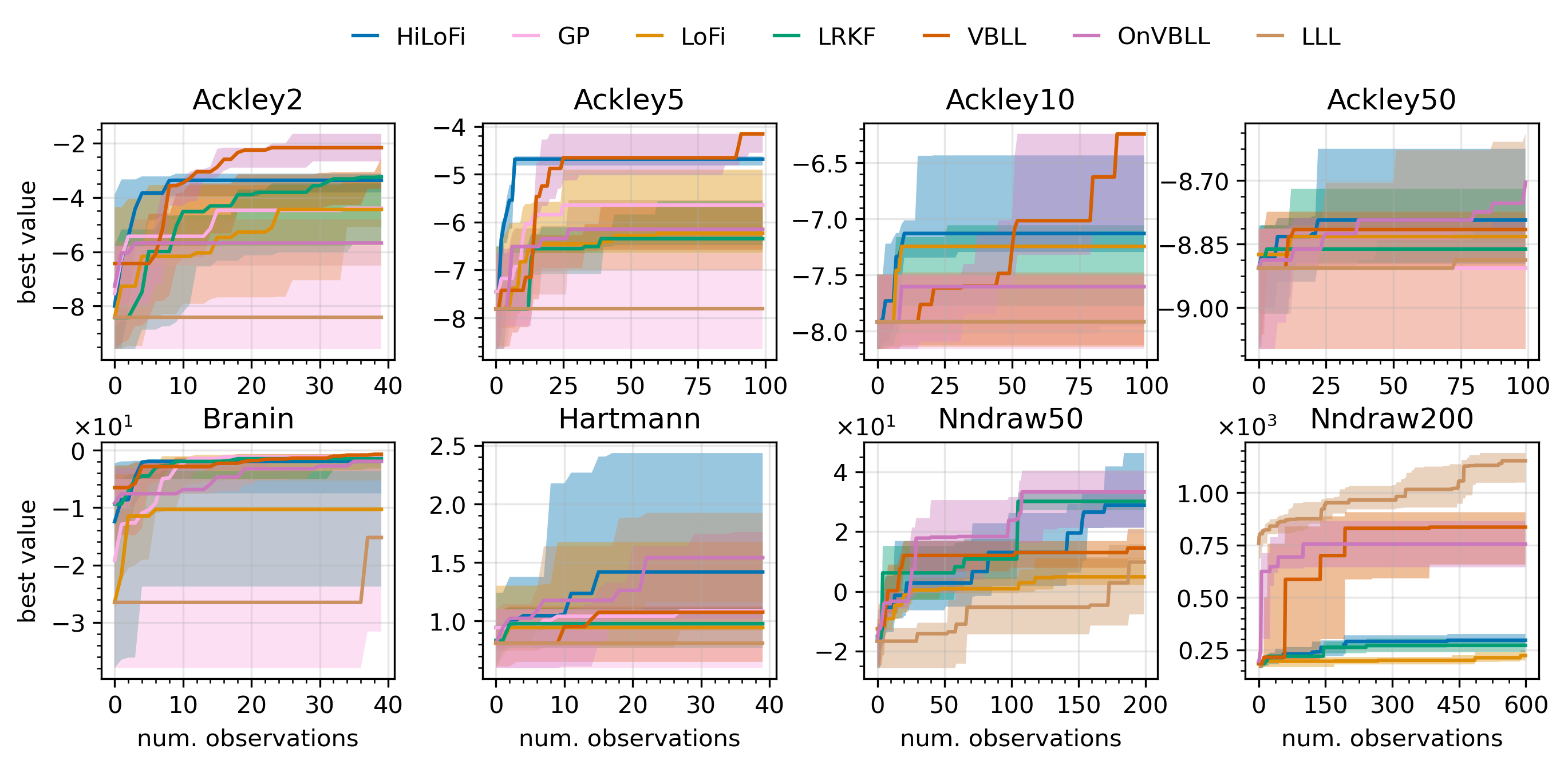}
    \caption{
        Bayesian optimization benchmark on stationary environments.
        Query points are chosen using expected improvement.
    }
    \label{fig:bayesopt-iterations-expected-improvement}
\end{figure}

\subsubsection{NN Draw}
Here, we provide further results for the performance of \newmethodlin and \methodlrkf
for the Bayesian optimization problem on the DrawNN dataset.

Figure~\ref{fig:nn-draw-hilofi-varying-dim} reports the best value found over the course of the optimization as a function of iteration, for different ranks.
\begin{figure}[htb]
    \centering
    \includegraphics[width=0.8\linewidth]{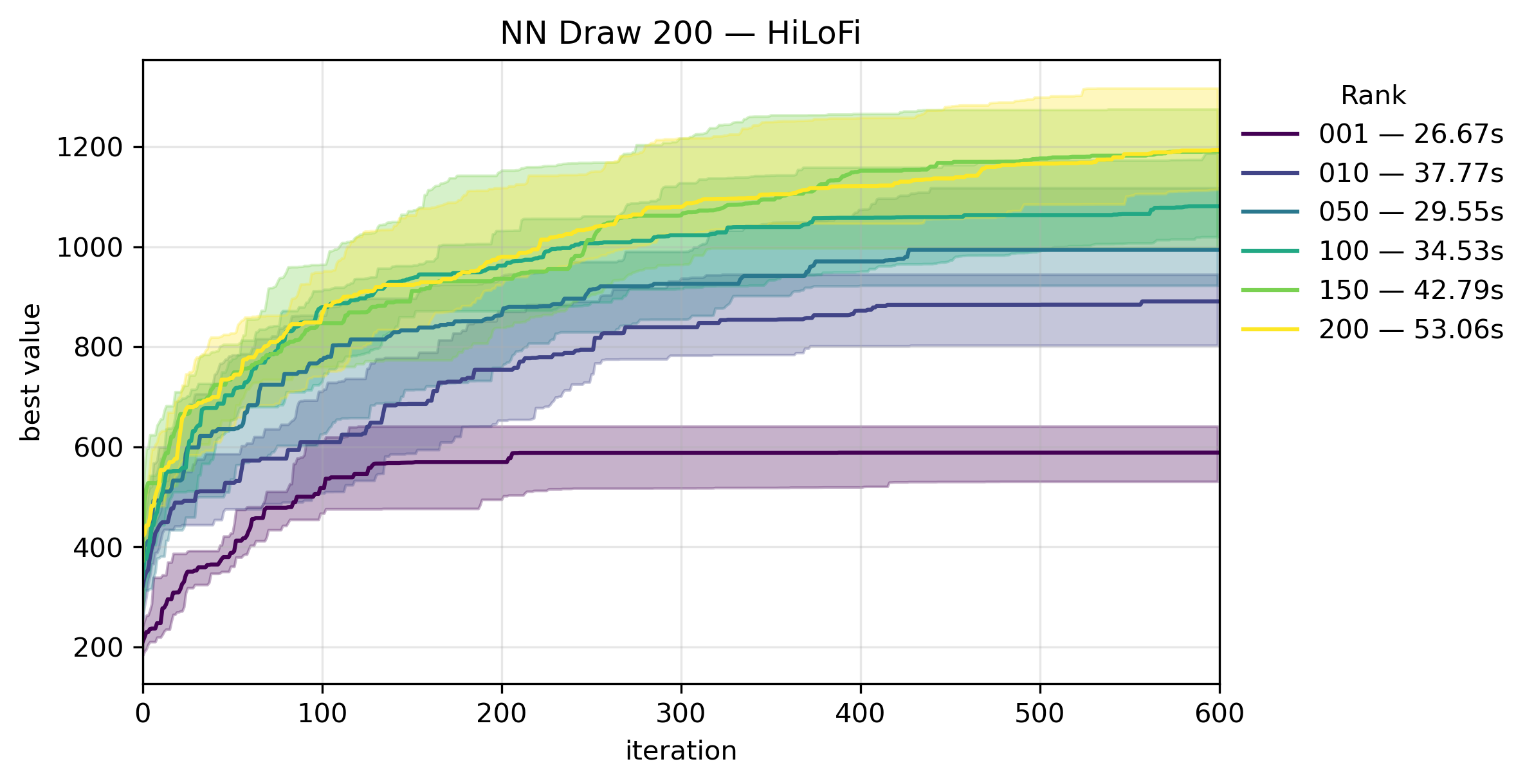}
    \caption{
        Solid lines show the median best value over 20 runs, and shaded regions denote the interquartile range.
    }
    \label{fig:nn-draw-hilofi-varying-dim}
\end{figure}
We observe that increasing the rank generally improves performance,
with higher ranks yielding sharper improvements early in the search that compound into stronger final results.
This comes with a trade-off between computational cost and the ability to capture uncertainty more effectively,
as illustrated in Table~\ref{tab:nn-draw-hilofi-varying-dim}, which reports the final performance and running time across ranks.
\begin{table}[htb]
\centering
\begin{tabular}{rcc}
\toprule
\textbf{Rank} & \textbf{Time (s)} & \textbf{Final $y_{\text{best}}$} \\
\midrule
1   & 26.226  & 656.558 \\
10  & 38.4492 & 630.556 \\
20  & 28.8612 & 670.913 \\
50  & 30.6409 & 684.009 \\
100 & 34.5302 & 782.205 \\
110 & 35.1504 & 774.416 \\
120 & 39.2540 & 807.298 \\
130 & 38.2288 & 846.218 \\
140 & 39.1417 & 848.219 \\
150 & 40.7939 & 864.806 \\
160 & 42.1907 & 911.116 \\
170 & 44.5485 & 900.566 \\
180 & 45.5640 & 864.260 \\
190 & 47.1553 & 911.262 \\
200 & 48.2127 & 921.062 \\
\bottomrule
\end{tabular}
\caption{Performance metrics across different low-rank configurations.}
\label{tab:nn-draw-hilofi-varying-dim}
\end{table}
We observe that both performance and running time generally increase with rank.

Finally, Figure~\ref{fig:nn-draw-lrkf-varying-dim} shows the best value obtained during optimization as a function of the number of observations, for different ranks using \methodlrkf.
\begin{figure}[htb]
\centering
\includegraphics[width=0.85\linewidth]{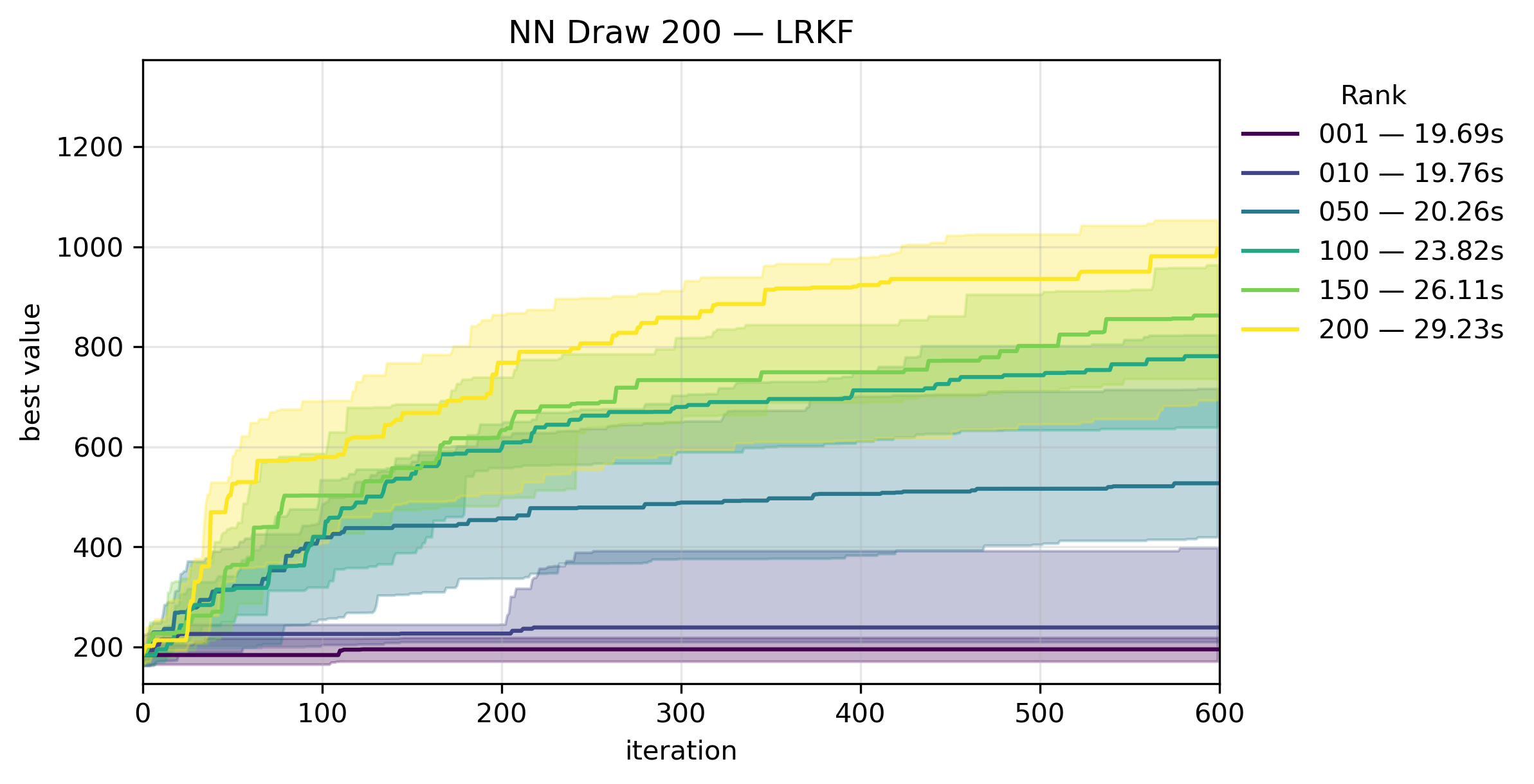}
\caption{Bayesian optimization on the NNDraw dataset using \methodlrkf.
Solid lines show the median best value over 20 runs, and shaded regions denote the interquartile range.}
\label{fig:nn-draw-lrkf-varying-dim}
\end{figure}
Compared to \newmethodlin, \methodlrkf requires substantially higher ranks to achieve similar median performance.
For instance, performance is comparable at rank~300 for \methodlrkf (73s) and rank~200 for \newmethodlin (under 50s).
This demonstrates that explicitly modeling the last layer in full rank,
as done in \newmethodlin, yields a more favorable accuracy–efficiency trade-off.

\subsection{In-between uncertainty}
\label{experiment:in-between-uncertainty}

In this experiment, we test the ability of \newmethodlin to capture the \textit{in-between uncertainty} \cite{foong2019between}
after a single pass of the data.
We consider the one-dimensional dataset introduced in \cite{van2021feature} (Figure 1).
For this experiment, we consider a four hidden-layer MLP with 128 units per layer and ELU activation function.

\paragraph{Result for \newmethodlin.}
We take $q_{\phidden, t} = q_{\plast, t} = 0$,
initial covariances for last and hidden layers to be identity times $1/2$, and
$\vR_t = 0.0$, which corresponds to having no aleatoric uncertainty in the data generating process.

Figure \ref{fig:in-between-uncertainty-by-rank}
shows the posterior predicted mean
surrounded by the two-standard deviation posterior predictive for various ranks in the hidden layer.

\begin{figure}[htb]
    \centering
    \includegraphics[width=\linewidth]{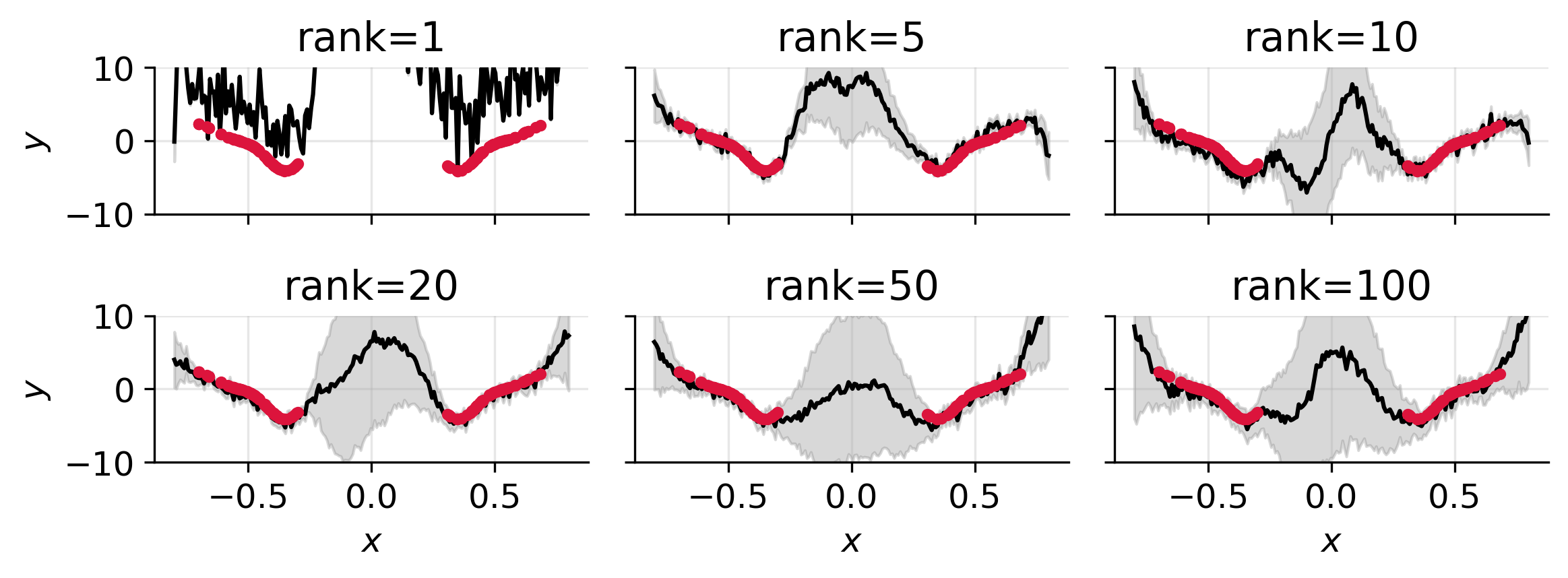}
    \caption{
        In-between uncertainty of the posterior predictive induced by \newmethodlin as a function of rank.
    }
    \label{fig:in-between-uncertainty-by-rank}
\end{figure}

We observe that a rank of $1$ is not able to capture any uncertainty around the region with no observations;
however, as we increase the rank, we observe that the posterior predictive becomes less and less confident about
the \textit{true} value of the mean on regions without data.

Next, Figure \ref{fig:in-between-uncertainty-by-timestep}
shows the evolution of the posterior predictive mean and the posterior predictive variance
as a function of the number of seen observations.

\begin{figure}[htb]
    \centering
    \includegraphics[width=\linewidth]{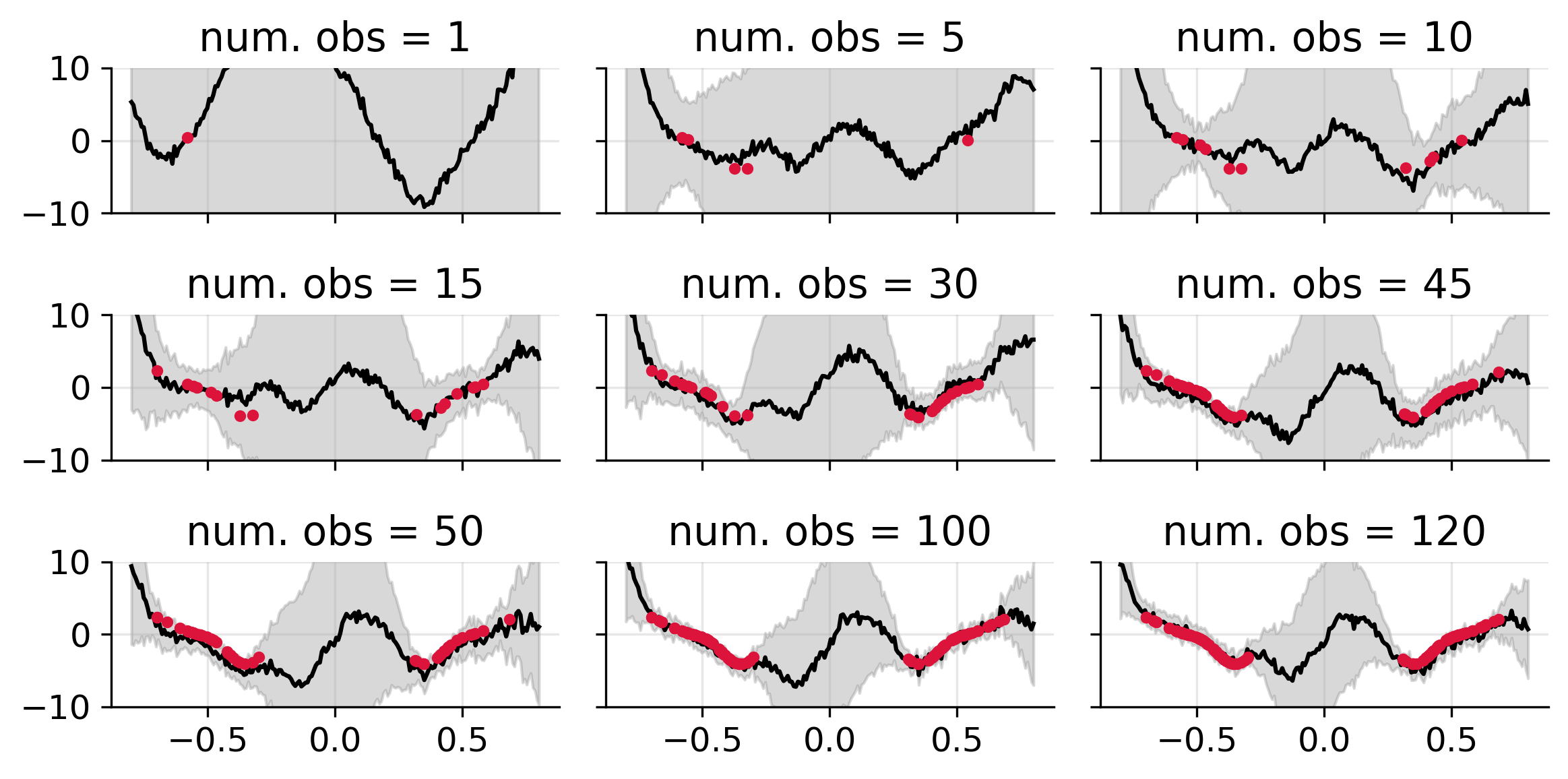}
    \caption{
        In-between uncertainty of the posterior predictive induced by \newmethodlin as a function of seen observations.
    }
    \label{fig:in-between-uncertainty-by-timestep}
\end{figure}

We observe that the uncertainty around the posterior predictive mean is wide
(covering the limit from $-10$ to $10$) and starts to decrease after $10$ observations.
By $30$ observations, the uncertainty is narrower over regions where data has been observed,
and
by $120$ observations, most of the posterior predictive uncertainty is on regions where no data has been observed.

\paragraph{Result for \methodlrkf}

We repeat the experiment above for \methodlrkf.
Figure \ref{fig:in-between-uncertainty-by-rank-lrkf}
shows the posterior predicted mean
surrounded by the two-standard deviation posterior predictive for various ranks.
\begin{figure}[htb]
    \centering
    \includegraphics[width=\linewidth]{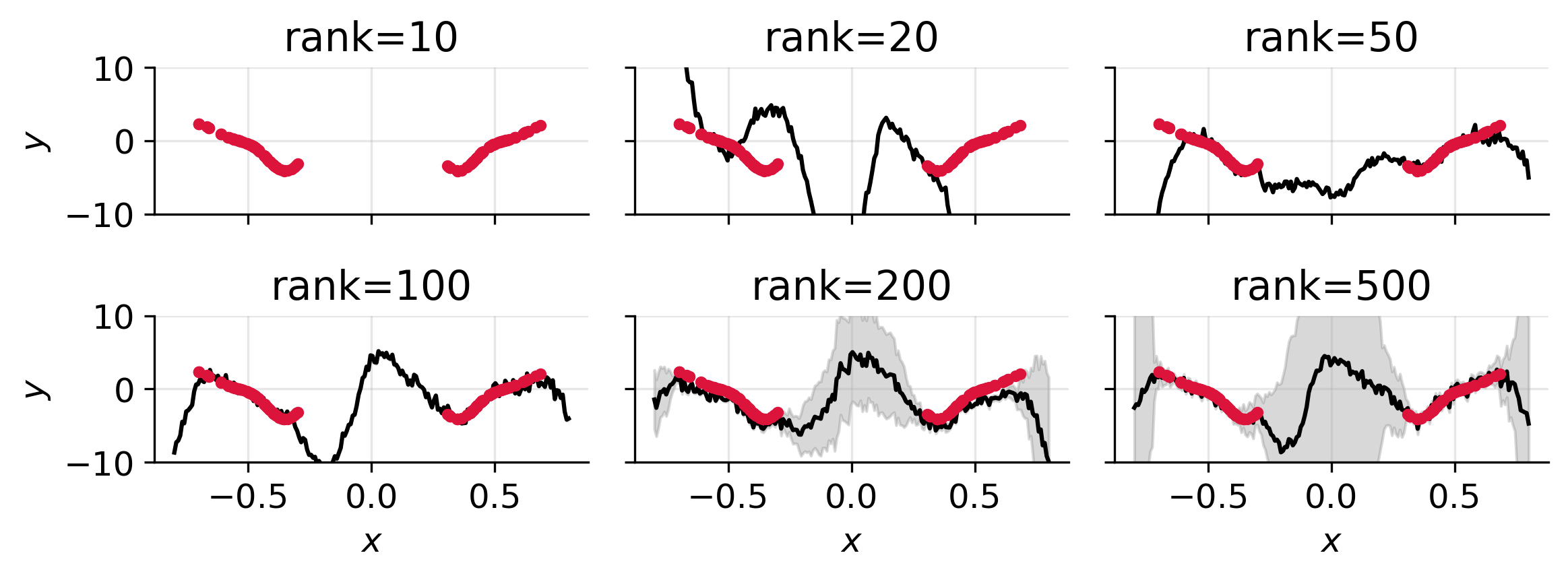}
    \caption{
        In-between uncertainty of the posterior predictive induced by \methodlrkf as a function of rank.
    }
    \label{fig:in-between-uncertainty-by-rank-lrkf}
\end{figure}
In contrast to \newmethodlin,
we observe that \methodlrkf requires a much higher rank to capture a reasonable level of \textit{in-between} uncertainty.
The panel for rank $10$ is empty because the posterior predictive (with $\vR_t = 0$) is psd,
so it is not defined.

An important distinction between \methodlrkf and \newmethodlin is the \newmethodlin considers a full-rank covariance
matrix in the last layer. Despite this, \methodlrkf with rank $200$ (which shows some notion of uncertainty)
requires around $3$ times more coefficients in the covariance than \newmethodlin with rank $50$ in the hidden layers.

We show the evolution of the posterior predictive mean and variance with rank $200$ of \methodlrkf
as a function of seen observations in Figure \ref{fig:in-between-uncertainty-by-timestep-lrkf}.

\begin{figure}[htb]
    \centering
    \includegraphics[width=\linewidth]{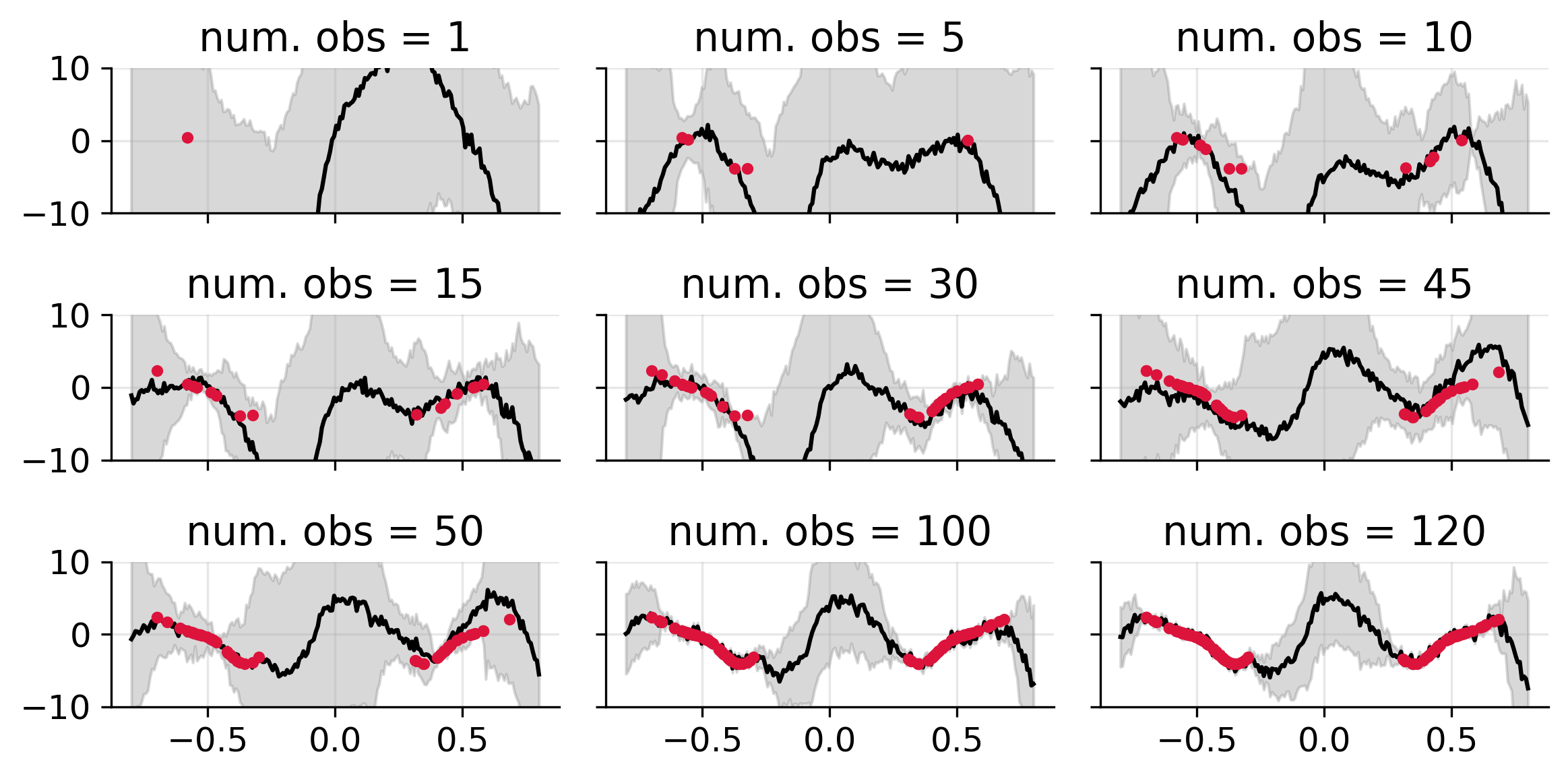}
    \caption{
        In-between uncertainty of the posterior predictive induced by
         \methodlrkf
        as a function of seen observations.
    }
    \label{fig:in-between-uncertainty-by-timestep-lrkf}
\end{figure}
We see that the evolution of the posterior predictive resembles that of \newmethodlin,
but at the cost of $3$ times more number of coefficients in the covariance.

\paragraph{Result for \methodvbll.}

We contrast the result of \newmethodlin with the offline \methodvbll method.
For \methodvbll, the data is presented all at once and we perform full-batch gradient descent using Adamw with learning rate $10^{-3}$.
Figure \ref{fig:in-between-uncertainty-vbll} shows the posterior predictive mean and two standard deviations
around the posterior mean, as well as the loss curve for various epochs.

We observe that \methodvbll takes around $10^4$ epochs to converge, whereas \newmethodlin takes only one.
We emphasize that each epoch in \methodvbll considers all datapoints at once,
so the gradient update is more informative, whereas for \newmethodlin,
the method only gets to observe the data once and in a sequence.
This highlights the efficiency of our method to produce approximate posterior predictives that
can be used online.

\begin{figure}[htb]
    \centering
    \includegraphics[width=0.8\linewidth]{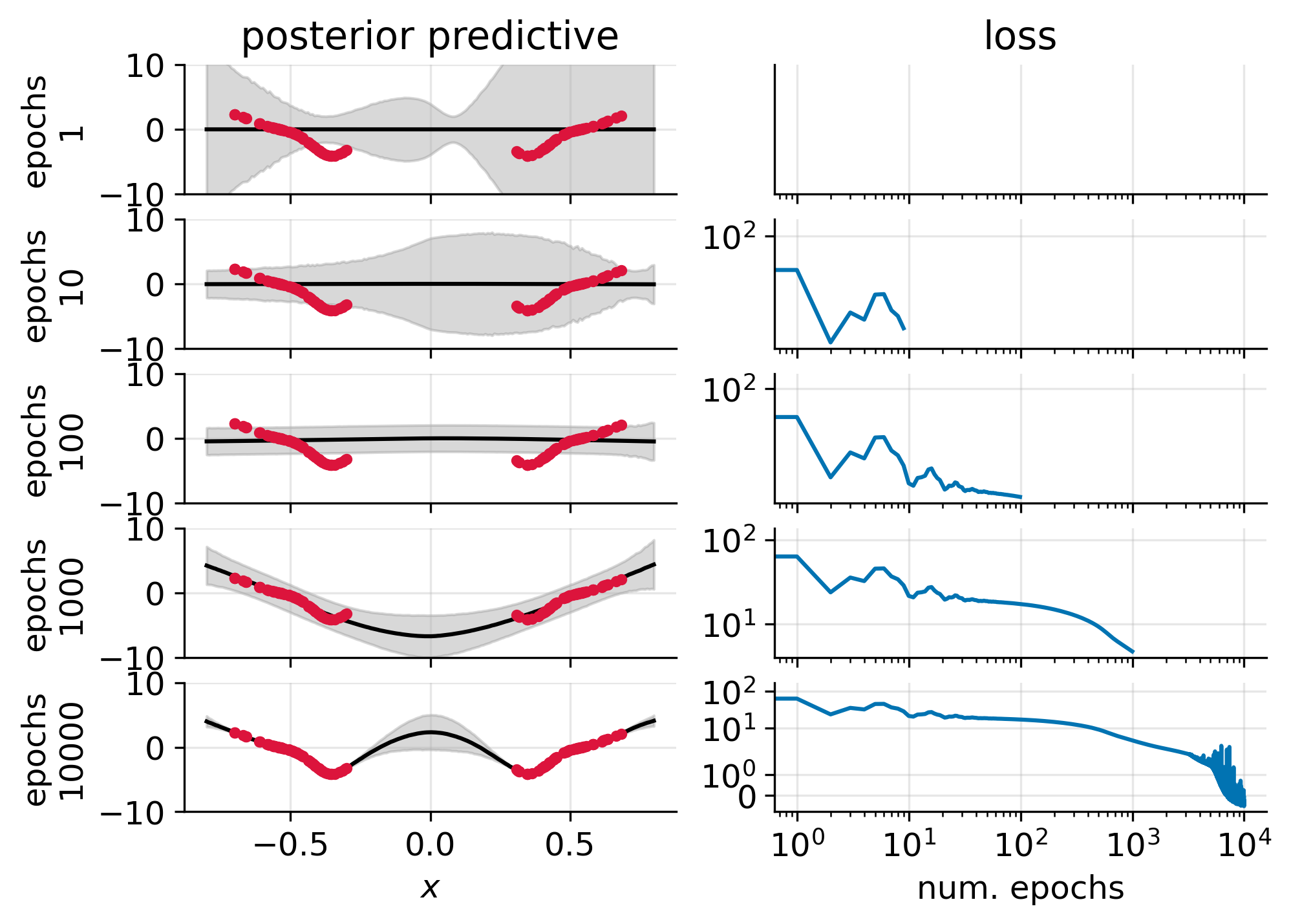}
    \caption{
        In-between uncertainty of the posterior predictive induced by \methodvbll
        as a function of number of epochs.
    }
    \label{fig:in-between-uncertainty-vbll}
\end{figure}

\clearpage
\section*{NeurIPS Paper Checklist}

\begin{enumerate}

\item {\bf Claims}
    \item[] Question: Do the main claims made in the abstract and introduction accurately reflect the paper's contributions and scope?
    \item[] Answer: \answerYes{} 
    \item[] Justification: The frequentist and Bayesian approaches to filtering are introduced in Section \ref{sec:EKF-summary} and detailed in Appendix \ref{sec:EKF}. Our method and its assumptions are detailed in Section \ref{sec:method}.
    Mathematical proofs of our claism are collected in Appendix \ref{sec:ll-lrkf-proofs}.
    Experiments and comparisons to other methods validating the experimental claims are provided in Section \ref{sec:experiments}.
    The motivation and contribution for our method is provided in Section \ref{sec:introduction}.

\item {\bf Limitations}
    \item[] Question: Does the paper discuss the limitations of the work performed by the authors?
    \item[] Answer: \answerYes{} 
    \item[] Justification: Limitations of the method are discussed in Section \ref{sec:conclusion}.

\item {\bf Theory assumptions and proofs}
    \item[] Question: For each theoretical result, does the paper provide the full set of assumptions and a complete (and correct) proof?
    \item[] Answer: \answerYes{} 
    \item[] Justification: The main assumptions of the methods we introduced are detailed in Section \ref{sec:EKF-summary} and Appendix \ref{sec:EKF}.
    The proof of Proposition \ref{prop:new-predict-step} is in Appendix \ref{proof:new-predict-step},
    the proof of Proposition \ref{prop:new-innovation-variance} is in Appendix \ref{proof:new-innovation-variance},
    the proof of Proposition \ref{prop:new-gain-new-update} is in Appendix \ref{proof:new-gain-new-update},
    and the proof of Proposition \ref{prop:hilofi-cov-per-step-error} is in Appendix \ref{proof:hilofi-cov-per-step-error}.
    Further propositions and proofs for the \methodlrkf are stated and proved in Appendix \ref{sec:LRKF}.

    \item {\bf Experimental result reproducibility}
    \item[] Question: Does the paper fully disclose all the information needed to reproduce the main experimental results of the paper to the extent that it affects the main claims and/or conclusions of the paper (regardless of whether the code and data are provided or not)?
    \item[] Answer: \answerYes{} 
    \item[] Justification: To facilitate reproducibility of the methods we present,
    Algorithm \ref{algo:low-rank-filter} details \methodlrkf,
    Algorithm \ref{algo:low-rank-full-rank-update} details \newmethodlin,
    and Algorithm \ref{algo:low-rank-low-rank-update} details \newmethodlinlow.
    The experiments are detailed in Section \ref{sec:experiments}, which mentions the dataset and the neural architecture used.
    Further information (such as hyperparameter selection) for all experiments is detailed in Appendix \ref{sec:experiments-further-results}.
    Furthermore, we provide code to reproduce the results in the anonymized link in the introduction.

\item {\bf Open access to data and code}
    \item[] Question: Does the paper provide open access to the data and code, with sufficient instructions to faithfully reproduce the main experimental results, as described in supplemental material?
    \item[] Answer: \answerYes{} 
    \item[] Justification: All our datasets are open source and available at the mentioned sources in Section \ref{sec:experiments}.
    Preprocessing of the data is detailed in Section \ref{sec:experiments} and Appendix \ref{sec:experiments-further-results}.
    Detailed pseudocode for the methods we introduce are detailed in
    Algorithm \ref{algo:low-rank-filter} for \methodlrkf,
    Algorithm \ref{algo:low-rank-full-rank-update} for \newmethodlin,
    and
    Algorithm \ref{algo:low-rank-low-rank-update} for \newmethodlinlow.
    Furthermore, we provide code to reproduce the results in the anonymized link in the introduction.

\item {\bf Experimental setting/details}
    \item[] Question: Does the paper specify all the training and test details (e.g., data splits, hyperparameters, how they were chosen, type of optimizer, etc.) necessary to understand the results?
    \item[] Answer: \answerYes{} 
    \item[] Justification:  The training and test details are in Section \ref{sec:experiments} and Appendix \ref{sec:experiments-further-results}. 

\item {\bf Experiment statistical significance}
    \item[] Question: Does the paper report error bars suitably and correctly defined or other appropriate information about the statistical significance of the experiments?
    \item[] Answer: \answerYes{} 
    \item[] Justification: Error bars are reported and explained for all experiments in Section \ref{sec:experiments} and Appendix \ref{sec:experiments-further-results}. For the experiment in Section \ref{experiment:mnist-bandits},
    suitable error bars are deferred to Figure \ref{fig:mnist-bandit-time-reward} in Appendix \ref{sec:further-results-mnist-bandit}
    as including them in the main figure would have made it overly cluttered and difficult to interpret.

\item {\bf Experiments compute resources}
    \item[] Question: For each experiment, does the paper provide sufficient information on the computer resources (type of compute workers, memory, time of execution) needed to reproduce the experiments?
    \item[] Answer: \answerYes{} 
    \item[] Justification:  Compute workers are detailed in Section \ref{sec:experiments}.
    An estimate of the running time to replicate the results with the compute used are provided in Section \ref{sec:experiments}.
    
\item {\bf Code of ethics}
    \item[] Question: Does the research conducted in the paper conform, in every respect, with the NeurIPS Code of Ethics \url{https://neurips.cc/public/EthicsGuidelines}?
    \item[] Answer: \answerYes{} 
    \item[] Justification: The authors have reviewed the NeurIPS Code of Ethics, and the research conducted in the paper conforms, in every respect, with these guidelines.

\item {\bf Broader impacts}
    \item[] Question: Does the paper discuss both potential positive societal impacts and negative societal impacts of the work performed?
    \item[] Answer: \answerYes{} 
    \item[] Justification: Potential societal impacts are acknowledged in Section \ref{sec:conclusion}.
    
\item {\bf Safeguards}
    \item[] Question: Does the paper describe safeguards that have been put in place for responsible release of data or models that have a high risk for misuse (e.g., pretrained language models, image generators, or scraped datasets)?
    \item[] Answer: \answerNA{} 
    \item[] Justification: The methods developed in this paper are general-purpose tools for online learning and sequential decision-making. We do not believe they pose a high risk of misuse, and therefore no specific safeguards were necessary.

\item {\bf Licenses for existing assets}
    \item[] Question: Are the creators or original owners of assets (e.g., code, data, models), used in the paper, properly credited and are the license and terms of use explicitly mentioned and properly respected?
    \item[] Answer: \answerYes{} 
    \item[] Justification: All assets used, whether code, data, or models have been properly cited in the main text.

\item {\bf New assets}
    \item[] Question: Are new assets introduced in the paper well documented and is the documentation provided alongside the assets?
    \item[] Answer: \answerYes{} 
    \item[] Justification: Assets are detailed in the introduction of the main text. We include all details to reproduce our methods and the licence is provided within the code. Consent was not required.

\item {\bf Crowdsourcing and research with human subjects}
    \item[] Question: For crowdsourcing experiments and research with human subjects, does the paper include the full text of instructions given to participants and screenshots, if applicable, as well as details about compensation (if any)? 
    \item[] Answer: \answerNA{} 
    \item[] Justification: This paper did not involve crowdsourcing nor research with human subjects.

\item {\bf Institutional review board (IRB) approvals or equivalent for research with human subjects}
    \item[] Question: Does the paper describe potential risks incurred by study participants, whether such risks were disclosed to the subjects, and whether Institutional Review Board (IRB) approvals (or an equivalent approval/review based on the requirements of your country or institution) were obtained?
    \item[] Answer: \answerNA{} 
    \item[] Justification: This paper did not involve crowdsourcing nor research with human subjects.

\item {\bf Declaration of LLM usage}
    \item[] Question: Does the paper describe the usage of LLMs if it is an important, original, or non-standard component of the core methods in this research? Note that if the LLM is used only for writing, editing, or formatting purposes and does not impact the core methodology, scientific rigorousness, or originality of the research, declaration is not required.
    \item[] Answer: \answerNA{} 
    \item[] Justification: The core method development in this research does not involve LLMs as any important, original, or non-standard components.

\end{enumerate}

\end{document}